\documentclass[11pt]{article}

\usepackage[preprint]{acl}

\usepackage{times}
\usepackage{latexsym}

\usepackage[T1]{fontenc}

\usepackage[utf8]{inputenc}

\usepackage{microtype}

\usepackage{inconsolata}

\usepackage{graphicx}

\usepackage{adjustbox}
\usepackage{subcaption}
\usepackage{booktabs}
\usepackage{multirow}

\usepackage{amsmath}
\usepackage{amssymb}
\usepackage{natbib}
\usepackage{amsthm}

\newtheorem{proposition}{Proposition}

\usepackage{hyperref}
\usepackage{xcolor}
\definecolor{mycitecolor}{RGB}{180,40,40}

\usepackage{tcolorbox}
\definecolor{lightgraybg}{RGB}{254,250,246}

\title{MiRD: Reliable Set-Valued Prediction for Open-Ended Question Answering via Miscoverage Risk Decomposition}

\author{
  Anqi Hu$^{1\dagger}$, Zhiyuan Wang$^{1\dagger *}$, Zijun Jia$^2$, Bo Fu$^{1*}$\\ \\
  $^1$University of Electronic Science and Technology of China\\
  $^2$Beihang University\\
  \small{
   \textbf{Correspondence:} \texttt{yhzywang@gmail.com,fubo@uestc.edu.cn}
 }
}

\begin{document}
\maketitle
\begin{abstract}
Reliable set-valued prediction provides a principled way to mitigate hallucinations in open-ended question answering (QA), yet existing conformal approaches typically rely on a fragile premise: finite sampling must already produce at least one admissible candidate, or calibration examples violating this condition are discarded. 
In this paper, we introduce \textbf{MiRD}, a two-stage framework that decomposes overall miscoverage into \emph{sampling failure} and \emph{conditional selection failure}. 
In Stage I, MiRD establishes an expectation-level marginal upper bound on the probability that finite sampling produces no admissible answer under a fixed budget. 
In Stage II, conditioned on sampling success, MiRD calibrates a conformal selection threshold using admission-correlated nonconformity scores defined over the full calibration set, thereby preserving calibration-set integrity. 
Across three open-ended QA datasets and eight models, MiRD controls sampling risk, conditional selection risk, and overall miscoverage, while yielding tighter first-stage bounds than PAC-style alternatives and more adaptive prediction sets than successful-only calibration.
\end{abstract}

\section{Introduction}
\label{sec: Introduction}
Large language models (LLMs) have become a central component of open-ended question answering (QA) systems~\citep{sui2025can,duan2025guidellm}, yet their generated answers remain unreliable~\citep{wang2025word,huang2024position}: even strong LLMs can produce hallucinated or semantically misaligned responses. 
A natural route toward reliability is split conformal prediction (SCP), which replaces a single answer with a prediction set that is ensured to contain at least one admissible candidate with high probability~\citep{campos2024conformal,zhou2025conformal}. 
However, such a guarantee implicitly presumes that the candidate pool to be filtered already contains an admissible answer. 
This premise is nontrivial in open-ended QA, where the output space is unbounded, correctness must be defined semantically, and a finite number of stochastic samples may fail to produce any admissible answer at all~\citep{wang2025sample,wang2026safer}.

This observation suggests that reliability in open-ended QA should not be viewed merely as a post-hoc filtering problem over a pre-existing candidate set. 
Even when an admissible answer exists in principle, finite sampling may fail to generate it~\citep{wang2025copu}; even when such an answer is sampled, subsequent conformalized selection may remove it. 
Existing conformal methods typically focus on only the latter source of error: they either assume that every sampled candidate set contains at least one admissible answer~\citep{wang-etal-2024-conu,wang2025sample}, or discard calibration data points that violate this condition~\citep{kaur2024addressing}, thereby restricting calibration to a success-conditioned subset of the data distribution. 
More recent sample-then-filter approaches relax this assumption using confidence-interval-based procedures~\citep{wang2026safer}, but establish PAC-style guarantees with an additional confidence parameter rather than a direct marginal conformal guarantee. 
These limitations motivate a framework that explicitly separates \emph{finite-sampling uncertainty} from \emph{conditional post-selection uncertainty}, while preserving the integrity of the full calibration set~\citep{wang2025sconu}. 

\begin{figure*}[!t]
    \centering
    \includegraphics[width=\linewidth]{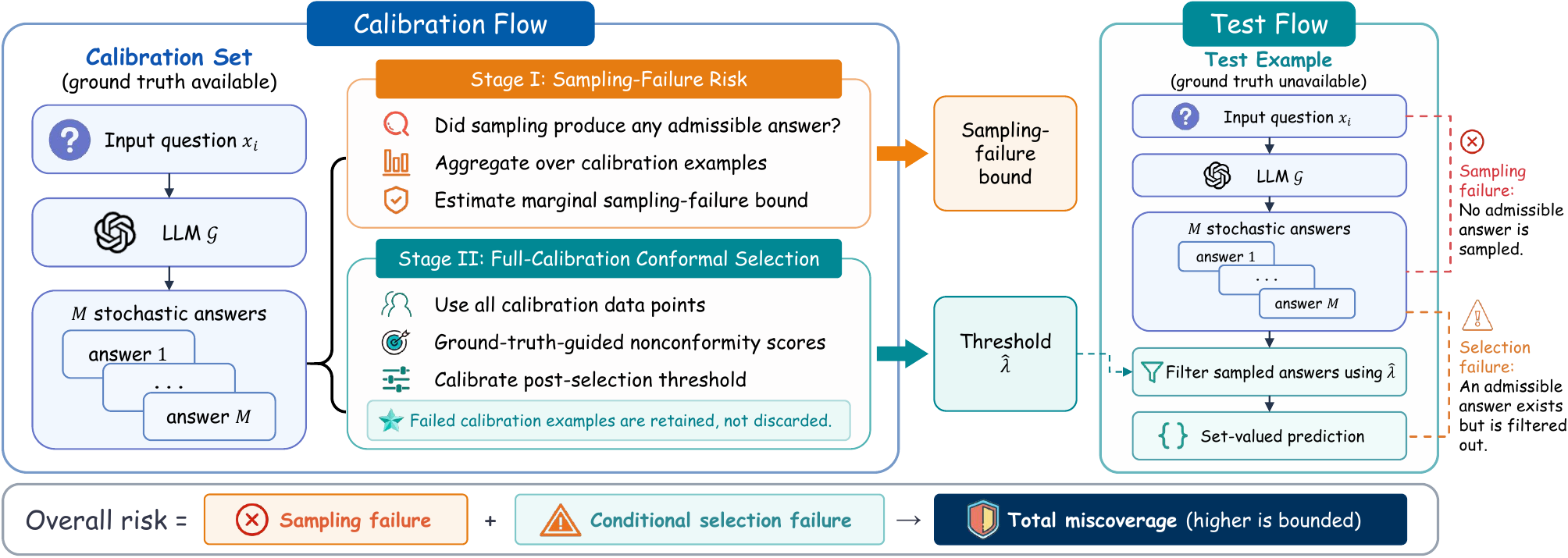}
    \vspace{-6mm}
    \caption{Overview of two-stage miscoverage risk decomposition.}
    \vspace{-4mm}
    \label{fig:overview}
\end{figure*}

In this work, we propose \textbf{MiRD}, a \textbf{Mi}scoverage \textbf{R}isk \textbf{D}ecomposition framework for reliable set-valued prediction in open-ended QA. 
The key idea is to decompose overall miscoverage into two statistically distinct sources: 
(i) a \emph{marginal sampling-failure risk}, where finite stochastic sampling fails to produce any admissible answer, and 
(ii) a \emph{conditional selection-failure risk}, where conformalized filtering removes all admissible answers despite their presence in the sampled candidate set. 
This decomposition naturally leads to a two-stage procedure. 
In Stage I, we derive an expectation-level marginal upper bound on the finite-sampling failure risk under a fixed sampling budget. 
In Stage II, conditional on sampling success, we calibrate a conformal selection threshold using admission-correlated nonconformity scores defined over the full calibration set. 
Crucially, MiRD does not discard calibration data whose sampled candidates contain no admissible answer. 
Instead, it leverages the availability of ground-truth answers during calibration to construct score information for all calibration examples, thereby preserving calibration-set integrity while enabling valid threshold calibration. 
An overview of MiRD is shown in Figure~\ref{fig:overview}.

This decomposition offers both conceptual and practical benefits. 
Conceptually, it separates open-ended miscoverage into finite-sampling failure and post-selection failure. 
Practically, it controls sampling failure explicitly and avoids the distributional distortion introduced by successful-only calibration. 
Although the selection stage relies on a score-level conditional compatibility assumption, we further show that, under our uncertainty construction, full-calibration thresholding can be formulated as a monotone enlargement of successful-only calibration: retaining sampling-failed calibration data yields a weakly more conservative threshold and thus a superset predictor on sampling-successful test data. 
This explains why MiRD preserves conformal validity while using the full calibration set.

We evaluate MiRD on three QA datasets with eight LLMs from four widely utilized model families. 
MiRD consistently controls finite-sampling risk, conditional selection risk, and overall miscoverage risk. 
Compared with PAC-style confidence bounds~\citep{bates2021distribution}, MiRD provides substantially tighter first-stage bounds while maintaining empirical validity. 
Compared with successful-only calibration~\citep{wang-etal-2024-conu,wang2025sample,kaur2024addressing}, MiRD yields prediction sets that better adapt to example difficulty. 
Overall, these results show that miscoverage risk decomposition benefits both statistical validity and the practical utility of set-valued prediction in open-ended generation.

In summary, our contributions are threefold:
\vspace{-0.2cm}
\begin{itemize}
    \item We reframe reliable open-ended QA as a miscoverage risk decomposition problem, separating overall failure into marginal sampling failure and conditional selection failure.
    \vspace{-0.2cm}
    \item We introduce MiRD, a two-stage calibration framework that combines an expectation-level marginal upper bound for finite-sampling risk with full-calibration conformalized selection for conditional post-selection risk control.
    \vspace{-0.2cm}
    \item We validate MiRD across multiple LLM families and QA benchmarks, achieving tighter sampling-risk bounds, more adaptive prediction sets, and competitive end-to-end miscoverage compared with prior alternatives.
\end{itemize}

\section{Related Work}
\label{sec: Related Work}

\noindent \textbf{\textit{Split Conformal Prediction.}} 
SCP converts arbitrary nonconformity scores into prediction sets with finite-sample marginal coverage guarantees under exchangeability~\citep{papadopoulos2008inductive,angelopoulos2023conformal,tan2025conformal}. 
Given a held-out calibration set, SCP calibrates a threshold from empirical nonconformity scores and applies it to future test examples according to a user-specified risk level. 
Recent work has adapted SCP to LLMs, but largely focuses on closed-ended scenarios~\citep{kumar2023conformal,cresswell2024conformal,ye2024benchmarking} or output modification~\citep{cherian2024large}. 
Reliable set-valued prediction for open-ended QA remains less settled, because admissible answers must be obtained from a finite sampled candidate pool rather than selected from a fixed label space.

\noindent \textbf{\textit{SCP in Open-Ended QA.}} 
For open-ended generation, \citet{kaur2024addressing} defines nonconformity as the uncertainty of admissible semantics within sampled answers~\citep{kuhn2023semantic}. 
ConU targets API-only settings and constructs self-consistency-based nonconformity, but assumes that every sampled set contains admissible answers~\citep{wang-etal-2024-conu}. 
TRON further develops a two-stage risk control framework that combines sampling-size calibration with conformal selection, while relying on the same sampling-success premise for both calibration and test data~\citep{wang2025sample}. 
\citet{quach2024conformal} and SAFER~\citep{wang2026safer} relax this premise through confidence-interval-based sample-then-filter procedures, but offer PAC-style guarantees with an additional confidence parameter $\delta$. 
In contrast, MiRD explicitly decomposes overall miscoverage into marginal finite-sampling failure and conditional post-selection failure, and performs full-calibration conformalized filtering without discarding sampling-failed calibration points. 

See additional related work in Appendix~\ref{sec: Additional Related Work}.






\section{Methodology}
\label{sec: Methodology}
\subsection{Notations and Problem Formulation}
\label{sec:notation_problem}

Let $\mathcal{G}: \mathcal{X}\to\mathcal{Y}$ be a pretrained LLM that maps an input prompt $x\in\mathcal{X}$ to a textual answer $\hat{y} \in \mathcal{Y}$. 
Following standard split conformal protocols~\citep{angelopoulos2024theoretical}, we hold out a calibration set $\mathcal{D}_{\mathrm{cal}}=\{(x_i,y_i^*)\}_{i=1}^{N}$, where $x_i\in\mathcal{X}$ is the $i$-th prompt and $y_i^*\in\mathcal{Y}$ is the corresponding ground truth. 
We consider one test example $(x_{N+1}, y_{N+1}^*)$, where $y_{N+1}^*$ is unavailable to the user. 
We assume that the calibration examples together with the test point are exchangeable. 
For each input $x$, we draw $M$ stochastic answers from $\mathcal{G}$ to form a finite candidate set $\mathcal{G}_M(x)=\{\hat y_m\}_{m=1}^{M}$. 
For calibration and test examples, we write $\mathcal{G}_M(x_i)=\{\hat y_m^{(i)}\}_{m=1}^{M}$ for all $i\in\{1, \cdots,N+1 \}$. 
To evaluate answer correctness, let $A:\mathcal{Y}\times\mathcal{Y}\to\{0,1\}$ be a binary admission function, where $A(y^*,\hat y)=1$ indicates that $\hat y$ is admissible with respect to the ground-truth answer $y^*$, and $A(y^*,\hat y)=0$ otherwise. 
Note that the admission rule can be instantiated by any task-specific semantic equivalence criterion (e.g., bi-entailment), and our framework is agnostic to its specific form.

A central challenge in open-ended generation is that finite sampling does not ensure the appearance of an admissible answer. 
To capture this irreducible source of error, we define the {sampling failure} indicator for the test example as
\vspace{-0.2cm}
\begin{equation}
Z_{N+1}
\!=\!
\mathbf{1}\!\left(
\forall \hat y \!\in\! \mathcal{G}_M(x_{N+1}), A(y_{N+1}^*, \hat y)\!=\! 0
\right),
\end{equation}
which equals 1 if no admissible answer is observed among the $M$ candidates for the future test input $x_{N+1}$, and 0 otherwise. We denote the corresponding sampling-failure probability by
\vspace{-0.2cm}
\begin{equation}
p_{\mathrm{fail}}:=\Pr(Z_{N+1}=1).
\end{equation}
\textbf{Our first objective} is to characterize this quantity through a marginal upper bound in expectation. 

Conditioned on the event $\{Z_{N+1}=0\}$, i.e., the sampling set containing at least one admissible answer, we then perform conformalized selection. 
Let $U_{\mathcal{G}}(\hat y\mid x)\in\mathbb{R}$ be an uncertainty score assigned by model $\mathcal{G}$ to the candidate $\hat y$ under input $x$. 
Given a threshold $\lambda$, we define the filtered prediction set as
\vspace{-0.2cm}
\begin{equation}
\mathcal{C}_{\lambda}(x_{N+1})
\!=\!
\left\{
\hat y \!\in\! \mathcal{G}_M(x_{N+1})\!:\!
U_{\mathcal{G}}(\hat y \!\mid\!  x_{N+1}) \!\le\! \lambda
\right\}.
\label{eq:pred_set_main}
\end{equation}
To formalize failure after filtering, we then define the miscoverage indicator as 
\vspace{-0.2cm}
\begin{equation}
R_{N+1}(\lambda)
\!=\!
\mathbf{1}\!\left(
\forall \hat y \!\in\! \mathcal{C}_{\lambda}(x_{N+1}),
A(y_{N+1}^*,\hat y)\!=\!0
\right).
\label{eq:sj_def_main}
\end{equation}
\textbf{Our second objective} is to calibrate a threshold $\hat\lambda$ such that
\vspace{-0.2cm}
\begin{equation}
\Pr\!\left(
R_{N+1}(\hat\lambda)=1 \mid Z_{N+1}=0
\right)
\le \alpha,
\label{eq:conditional_goal_main}
\end{equation}
guided by a user-specified risk level $\alpha\in(0,1)$.

Given that event $\{Z_{N+1}=0\}$ is unobservable at test time, constraining the conditional risk in Eq.~\eqref{eq:conditional_goal_main} is insufficient to guarantee the reliability of the final prediction sets. 
\textbf{Our overall goal} is therefore to construct coverage-valid set-valued predictions for future test inputs by explicitly decomposing the overall miscoverage risk into two sources: 
(i) the marginal risk that finite sampling produces no admissible answer at all, and (ii) the conditional risk that conformalized selection removes all admissible answers when such answers do exist.

To this end, we decompose the overall risk:
\vspace{-0.2cm}
\begin{equation}
\begin{split}
    \Pr&\!\left(R_{N+1}(\lambda)\!=\!1\right)
\!=\!
\Pr(Z_{N+1}=1)+\\
&\Pr\!\left(R_{N+1}(\lambda)\!=\!1\!\mid \!Z_{N+1}\!=\!0\right)\! \Pr(Z_{N+1}\!=\!0).
\end{split}
\label{eq:risk_decomposition_main}
\end{equation}
Eq.~\eqref{eq:risk_decomposition_main} motivates our two-stage framework: we first upper-bound the marginal sampling risk induced by finite generation, and then calibrate $\hat\lambda$ to control the conditional selection risk given $Z_{N+1}=0$.

\subsection{Bounding Sampling-Failure Probability}
\label{sec:bound_sampling_failure}

The conditional guarantee in Eq.~\eqref{eq:conditional_goal_main} only applies to test instances for which finite sampling has already produced at least one admissible answer. 
We therefore first quantify the marginal probability that a fixed sampling budget $M$ fails to generate any admissible answer at all. 
Since our goal is to control the average failure rate over future test data, rather than to identify which particular instances will fail, we target a marginal guarantee in expectation. 

To make the dependence on the sampling budget explicit, for each data point $(x_i,y_i^*)$, we define
\vspace{-0.2cm}
\begin{equation}
Z_i(M)
=
\mathbf{1}\!\left(
\forall \hat y\in\mathcal{G}_M(x_i),\;
A(y_i^*,\hat y)=0
\right).
\label{eq:ziM_main}
\end{equation}
Thus, $\{Z_i(M)=1\}$ indicates that no admissible answer is observed within $M$ stochastic sampling. 
We assume that the candidate sets are generated sequentially, so that $\mathcal{G}_{M_1}(x)\subseteq\mathcal{G}_{M_2}(x)$ whenever $M_1\le M_2$. 
Under the construction, $Z_i(M)$ is non-increasing in $M$ and bounded in $[0,1]$.

Given a fixed sampling budget $M$, we define the empirical sampling-failure rate on average over the calibration set as
\vspace{-0.2cm}
\begin{equation}
\widehat{R}_N(M)
=
\frac{1}{N}\sum_{i=1}^{N} Z_i(M).
\label{eq:emp_sampling_failure_main}
\end{equation}
Since $Z_i(M)\in\{0,1\}$, the sampling-failure probability for future test points satisfies
\vspace{-0.2cm}
\begin{equation}
p_{\mathrm{fail}}(M)\!
=\!
\Pr\!\big(Z_{N+1}(M)\!=\!1\big)
\!=\!
\mathbb{E}[Z_{N+1}(M)],
\label{eq:pfail_mean_equiv_main}
\end{equation}
which means, for this binary loss, the failure probability and the expected loss coincide~\citep{angelopoulos2023conformal,angelopoulos2024theoretical}. 

The following proposition connects the test-time failure probability to the empirical sampling-failure rate measured on the calibration set.

\begin{proposition}[Marginal upper bound for a fixed sampling budget]
\label{prop:fixedM_marginal_bound_main}
Assume that the $N$ calibration data points and one future test data are exchangeable. 
Then, for any fixed sampling budget $M$,
\vspace{-0.2cm}
\begin{equation}
\begin{split}
    p_{\mathrm{fail}}(M)
&=
\Pr\!\big(Z_{N+1}(M)=1\big)\\
&\le
\mathbb{E}\!\left[
\frac{N}{N+1}\widehat{R}_N(M)+\frac{1}{N+1}
\right].
\end{split}
\label{eq:fixedM_marginal_bound_main}
\end{equation}
\end{proposition}
See the proof of Proposition~\ref{prop:fixedM_marginal_bound_main} in Appendix~\ref{sec: proofs}.

Proposition~\ref{prop:fixedM_marginal_bound_main} establishes a marginal upper bound for the first-stage sampling risk under the fixed budget $M$. 
Importantly, the guarantee is in expectation over both the calibration instances and the future test data point. 
Therefore, unlike prior PAC-style approaches~\citep{wang2026safer}, the quantity
\vspace{-0.2cm}
\begin{equation}
\widetilde{\beta}(M)
:=
\frac{N}{N+1}\widehat{R}_N(M)+\frac{1}{N+1}
\label{eq:beta_tilde_main}
\end{equation}
can be interpreted as a calibration-dependent proxy whose expectation upper-bounds test-time marginal failure probability, rather than as a high-probability (i.e., $\geq 1-\delta$) upper confidence bound on a fixed but unknown parameter~\citep{clopper1934use}. 

At test time, we cannot determine which particular data satisfy $Z_{N+1}(M)=0$. 
However, Proposition~\ref{prop:fixedM_marginal_bound_main} guarantees that, \emph{on average over future test points}, the fraction of instances for which finite-sampling fails to yield an admissible answer is controlled by the exchangeability-corrected empirical risk in Eq.~\eqref{eq:beta_tilde_main}. 
In the subsequent section, conditioned on the complementary event $Z_{N+1}=0$, we calibrate the conformalized selection threshold to control the second-stage conditional risk.

\subsection{Conformalized Selection}
\label{sec:conformal_selection_main}

Even when finite generation fails to produce an admissible answer, the ground-truth answer remains available for every calibration data point. 
Given a calibration input, we therefore estimate the model uncertainty assigned to the ground truth to define the nonconformity score (NS). 
For each calibration example $(x_i,y_i^*)$, we formulate
\vspace{-0.2cm}
\begin{equation}
u_i := U_{\mathcal{G}}(y_i^* \mid x_i), \qquad i=1,\dots,N.
\label{eq:golden_ns_main}
\end{equation}
In this way, we establish the admission-correlated nonconformity for \emph{all} calibration data.
This avoids discarding calibration instances whose candidate sets fail to cover an admissible answer~\citep{wang-etal-2024-conu,wang2025sample,kaur2024addressing}, thereby preserving the integrity (distribution) of calibration data. 

Let $\{u_{(1)} \le \cdots \le u_{(N)}\}$ represent the ordered calibration NSs, and define
\vspace{-0.2cm}
\begin{equation}
k_\alpha := \left\lceil (N+1)(1-\alpha)\right\rceil.
\label{eq:kalpha_main}
\end{equation}
We calibrate the selection threshold as the conformal quantile
\vspace{-0.2cm}
\begin{equation}
\hat{\lambda} := u_{(k_\alpha)},
\label{eq:lambda_hat_main}
\end{equation}
with the standard convention that $u_{(N+1)}:=+\infty$.

Now consider a test data point $x_{N+1}$ satisfying $Z_{N+1}=0$, i.e., its candidate set contains at least one admissible answer. 
Among all admissible candidates, we define the reference admissible answer:
\vspace{-0.2cm}
\begin{equation}
\hat y_{N+1}^{(\mathrm{ref})}
\in
\operatorname*{\arg\min}_{\hat y\in\mathcal{G}_M(x_{N+1}):\,A(y_{N+1}^*,\hat y)=1}
U_{\mathcal{G}}(\hat y\mid x_{N+1}),
\label{eq:ref_answer_main}
\end{equation}
and let its associated admission-correlated NS be
\vspace{-0.2cm}
\begin{equation}
u_{N+1} := U_{\mathcal{G}}(\hat y_{N+1}^{(\mathrm{ref})}\mid x_{N+1}).
\label{eq:test_ns_main}
\end{equation}
By construction, on event $\{Z_{N+1}=0\}$, we have
\vspace{-0.2cm}
\begin{equation}
\left\{\!
\exists \hat y \!\in\! \mathcal{C}_{\hat\lambda}(x_{N+1}), A(y_{N+1}^*,\hat y)\!=\! 1\! 
\right\}
\!=\!
\{ u_{N+1} \!\le\! \hat\lambda \}.
\label{eq:event_equiv_main}
\end{equation}

We further impose the exchangeability requirement for this stage: conditional on $Z_{N+1}=0$, the augmented score sequence $\{u_1,\dots,u_N,u_{N+1}\}$ is exchangeable. 
That is, the admission-correlated NS of the test admissible answer is distributed symmetrically with the calibration-side ground-truth NSs. 
We discuss this assumption in Appendix~\ref{app:conditional_exchangeability}. 
Under this formulation, the standard split conformal quantile argument yields
\vspace{-0.2cm}
\begin{equation}
\begin{split}
    &\quad \, \Pr\left(
u_{N+1}\le \hat\lambda  \mid Z_{N+1}=0
\right)\\
&=
\Pr\left(
u_{N+1}\le u_{(k_\alpha)} \mid Z_{N+1}=0
\right)\\
&\ge
\frac{k_\alpha}{N+1}
\ge
1-\alpha.
\end{split}
\label{eq:quantile_guarantee_main}
\end{equation}
Combining Eqs.~\eqref{eq:event_equiv_main} and \eqref{eq:quantile_guarantee_main}, we obtain
\vspace{-0.2cm}
\begin{equation}
\begin{split}
\Pr&\!\left(
\exists \hat y \!\in\! \mathcal{C}_{\hat\lambda}(x_{N+1}), A(y_{N+1}^*,\hat y)\!=\!1
\middle|
Z_{N+1}\!=\!0
\right)\\
&\qquad \qquad \qquad \ge 1-\alpha,
\end{split}
\label{eq:conditional_coverage_main}
\end{equation}
and equivalently,
\vspace{-0.2cm}
\begin{equation}
\Pr\!\left(
R_{N+1}(\hat\lambda)=1 \mid Z_{N+1}=0
\right)
\le
\alpha.
\label{eq:conditional_miscoverage_main}
\end{equation}
Therefore, Eq.~\eqref{eq:conditional_goal_main} holds with the threshold $\hat\lambda$.

\subsection{Overall Coverage Guarantee}
\label{sec:overall_guarantee_main}
We combine the marginal sampling-failure bound in Proposition~\ref{prop:fixedM_marginal_bound_main} with the conditional guarantee in Eq.~\eqref{eq:conditional_miscoverage_main}. 
For a future test example, the decomposition in Eq.~\eqref{eq:risk_decomposition_main} gives
\vspace{-0.2cm}
\begin{equation}
\begin{split}
    \Pr\!\left(
R_{N+1}(\hat\lambda)=1
\right)
=
\Pr\!\left(
Z_{N+1}(M)=1
\right)&
+\\
\Pr\!\left(
R_{N+1}(\hat\lambda)=1 \mid Z_{N+1}(M)=0
\right)\cdot&\\
\Pr\!\left(
Z_{N+1}(M)=0
\right)&.
\end{split}
\label{eq:overall_decomp_main}
\end{equation}
Using Proposition~\ref{prop:fixedM_marginal_bound_main} and Eq.~\eqref{eq:conditional_miscoverage_main}, we obtain
\vspace{-0.2cm}
\begin{equation}
\begin{split}
    \Pr\!\left(
R_{N+1}(\hat\lambda)=1
\right)
&\le
p_{\mathrm{fail}}(M)+\alpha\\
&\le
\mathbb{E}\!\left[\widetilde{\beta}(M)\right]+\alpha
\end{split}.
\label{eq:overall_bound_main}
\end{equation}
Therefore, MiRD establishes an expectation-level marginal coverage guarantee for open-ended set-valued prediction. 
Eq.~\eqref{eq:overall_bound_main} is conservative; Appendix~\ref{app:tighter_bound} gives a tighter population-level refinement by exploiting the complementarity between sampling failure and conditional selection failure.

\begin{figure*}[!t]
  \centering
  \begin{subfigure}[b]{0.245\textwidth}
    \centering
    \includegraphics[width=\textwidth]{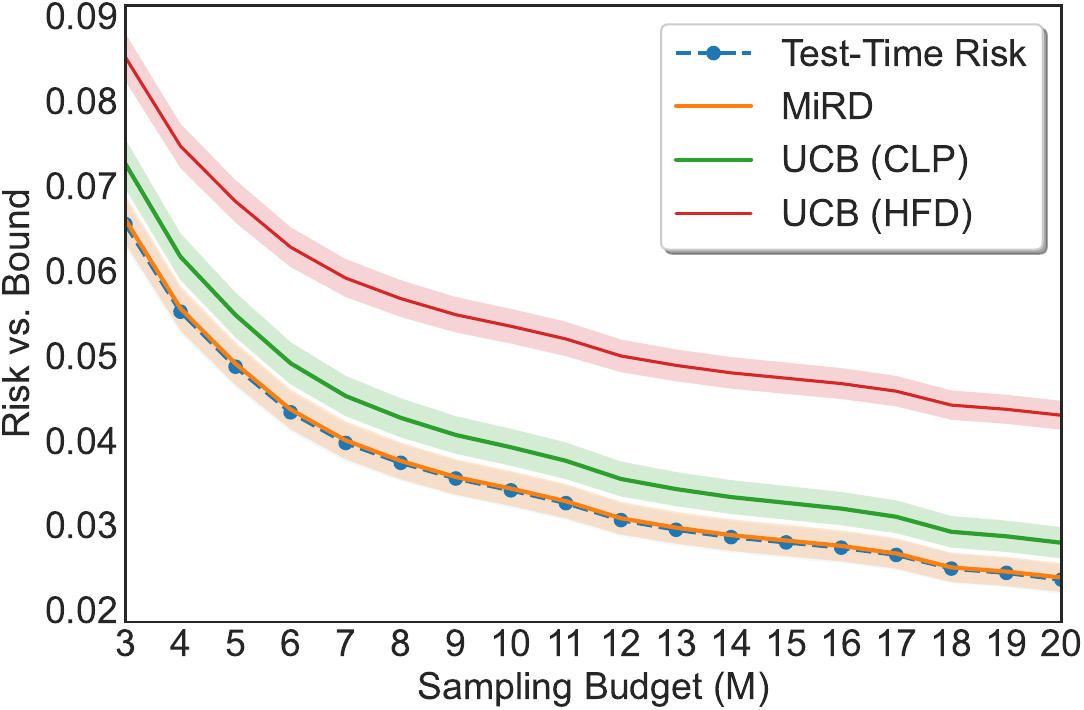}
    \caption{LLaMA-3.1-8B-Instruct.}
  \end{subfigure}
  \hfill
  \begin{subfigure}[b]{0.245\textwidth}
    \centering
    \includegraphics[width=\textwidth]{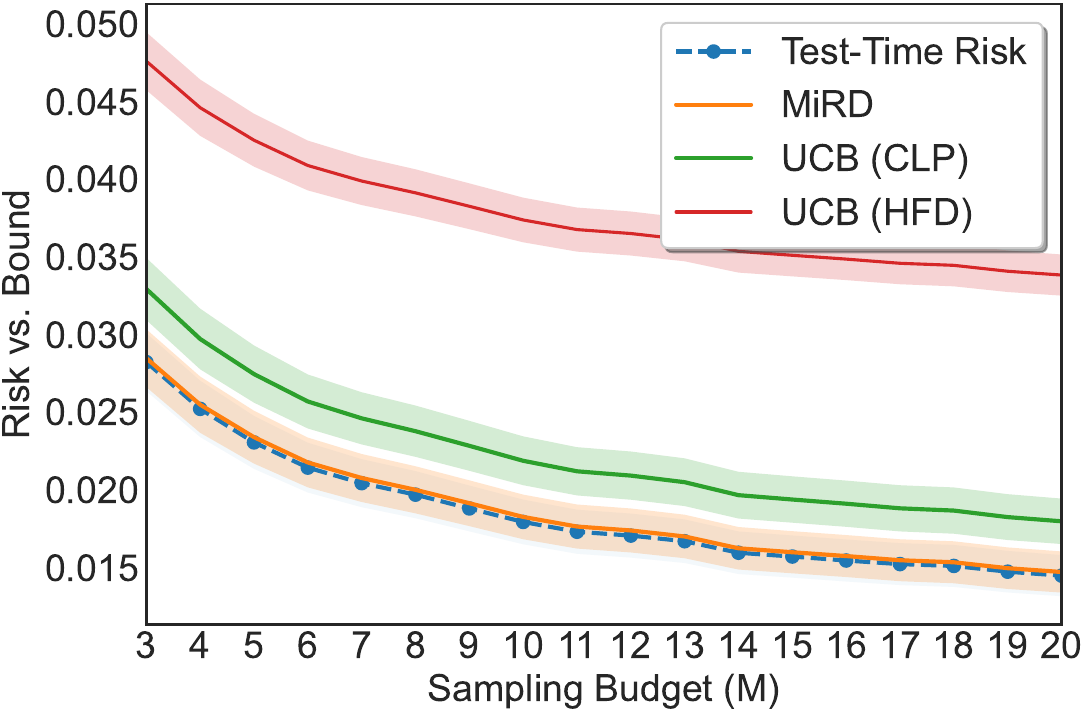}
    \caption{LLaMA-3.1-70B-Instruct.}
  \end{subfigure}
  \hfill
  \begin{subfigure}[b]{0.245\textwidth}
    \centering
    \includegraphics[width=\textwidth]{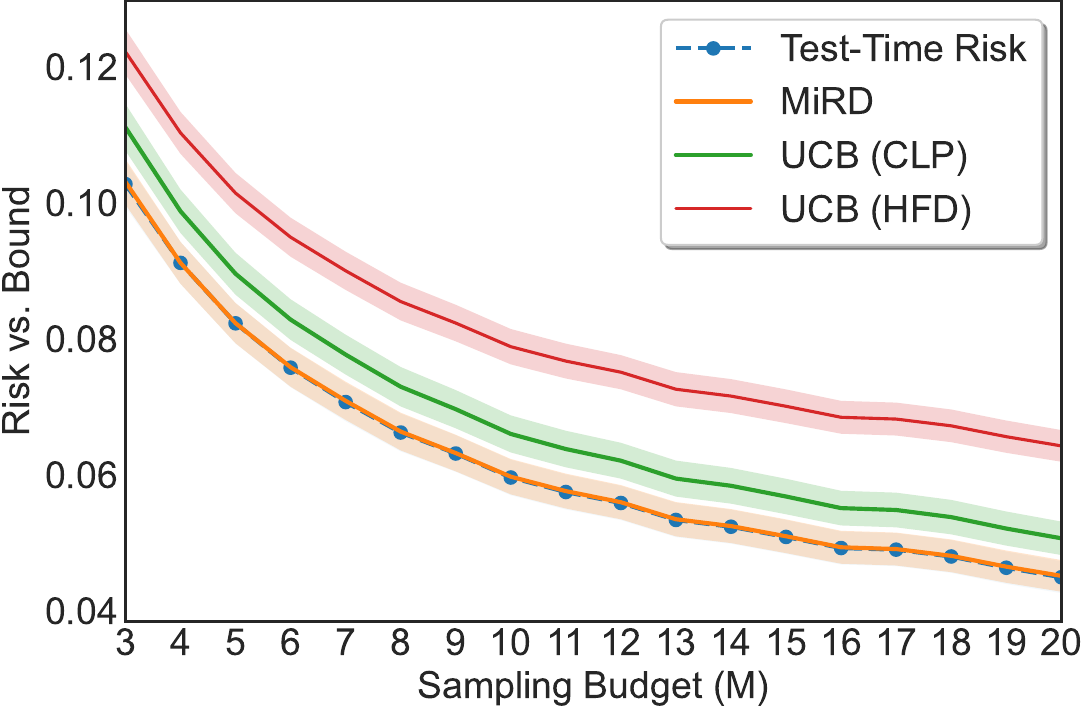}
    \caption{Vicuna-7B-Instruct.}
  \end{subfigure}
  \hfill
  \begin{subfigure}[b]{0.245\textwidth}
    \centering
    \includegraphics[width=\textwidth]{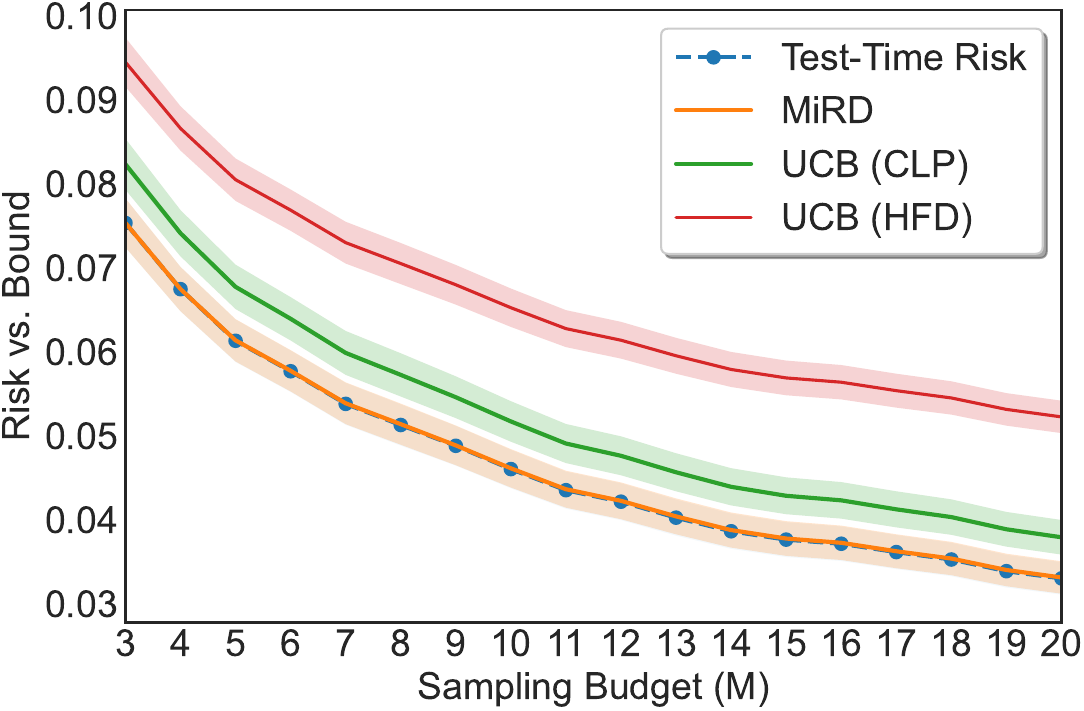}
    \caption{Vicuna-13B-Instruct.}
  \end{subfigure}

  \begin{subfigure}[b]{0.245\textwidth}
    \centering
    \includegraphics[width=\textwidth]{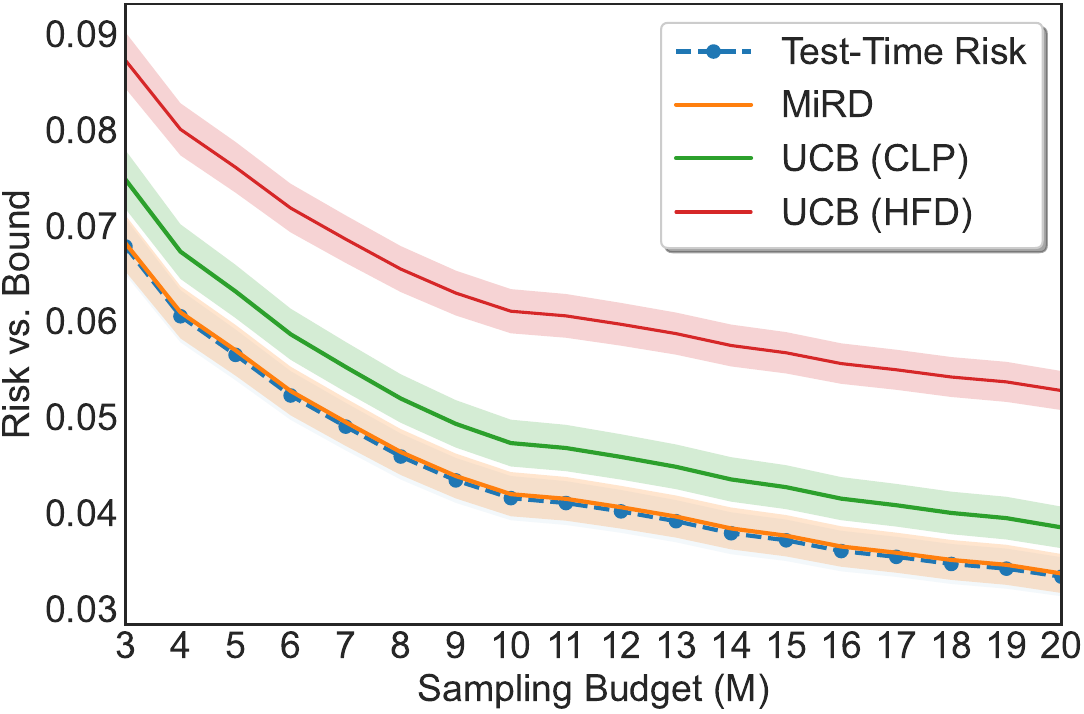}
    \caption{OpenChat-3.5 (7B).}
  \end{subfigure}
  \hfill
  \begin{subfigure}[b]{0.245\textwidth}
    \centering
    \includegraphics[width=\textwidth]{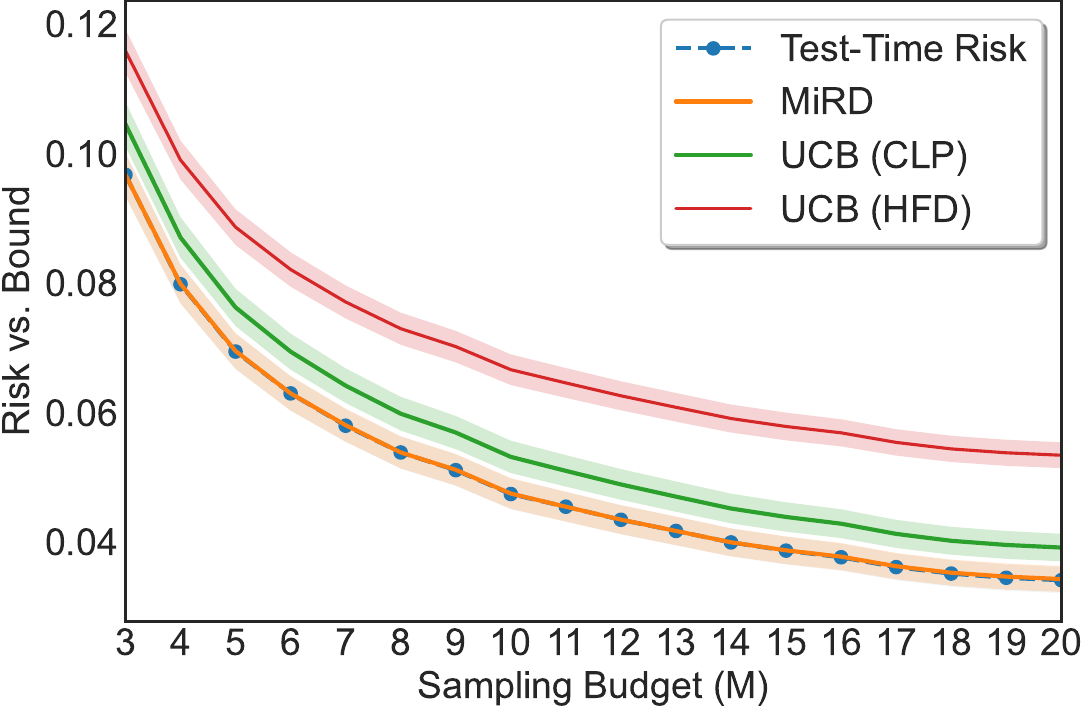}
    \caption{Qwen2.5-3B-Instruct.}
  \end{subfigure}
  \hfill
  \begin{subfigure}[b]{0.245\textwidth}
    \centering
    \includegraphics[width=\textwidth]{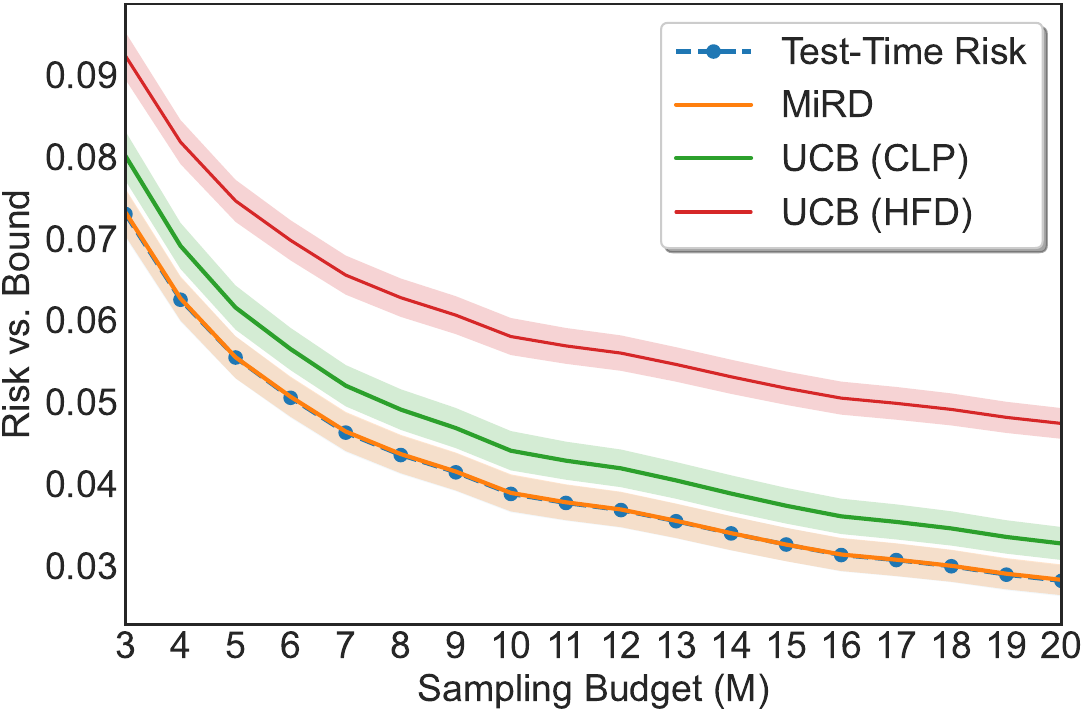}
    \caption{Qwen2.5-7B-Instruct.}
  \end{subfigure}
  \hfill
  \begin{subfigure}[b]{0.245\textwidth}
    \centering
    \includegraphics[width=\textwidth]{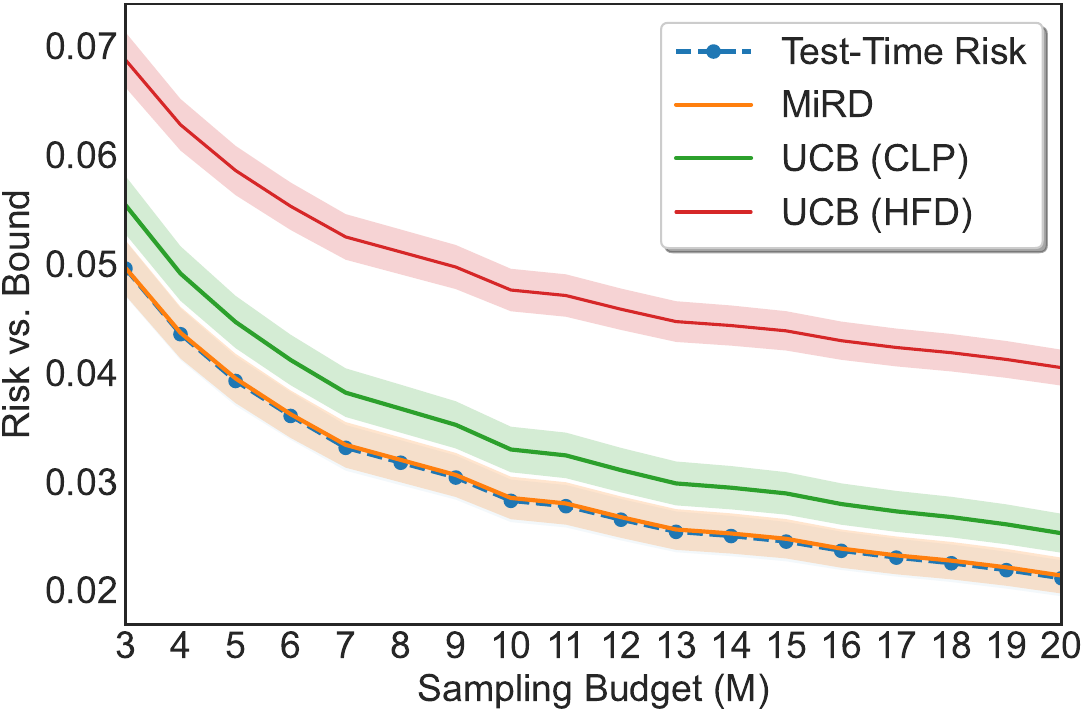}
    \caption{Qwen2.5-14B-Instruct.}
  \end{subfigure}
      \caption{Sampling risk vs. three upper bounds at various sampling budgets ($M$) on TriviaQA with eight LLMs.}
  \label{fig: sampling risk control triviaqa (similairty-0.6).}
  \vspace{-5mm}
\end{figure*}

\section{Experiments}
\label{sec: Experiments}
\subsection{Experimental Setup}
\label{sec: Experimental Setup}
\noindent \textbf{Datasets.} 
We use the closed-book TriviaQA~\citep{joshi2017triviaqa} and open-book CoQA~\citep{reddy2019coqa} datasets, covering both knowledge-intensive and context-supported generation settings to test the robustness of MiRD. 
We also consider more challenging Natural Questions (NQ)~\citep{kwiatkowski2019natural}. 
See details in Appendix~\ref{app:Details of Experimental Setup}. 

\noindent \textbf{Models.} 
We use eight open-source LLMs, including four popular families: LLaMA~\citep{touvron2023llama} (LLaMA-3.1-8B-Instruct and LLaMA-3.1-70B-Instruct), Qwen~\citep{bai2023qwen} (Qwen2.5-3B-Instruct, Qwen2.5-7B-Instruct, and Qwen2.5-14B-Instruct), Vicuna~\citep{zheng2023judging} (Vicuna-7B-Instruct and Vicuna-13B-Instruct), and OpenChat~\citep{wang2024openchat} (OpenChat-3.5).

\noindent \textbf{Baselines.} 
In the sampling stage, previous studies typically construct a $(1-\delta)$-level upper confidence bound (UCB) on the sampling-size-conditioned system failure rate~\citep{wang2026safer}. 
We compare our marginal upper bound in Proposition~\ref{prop:fixedM_marginal_bound_main} with both the Clopper--Pearson-style~\citep{clopper1934use} and Hoeffding-style~\citep{hoeffding1963probability} UCBs, denoted as UCB-CLP and UCB-HFD, respectively. 
In the conformal selection stage, we compare against ConU~\cite{wang-etal-2024-conu}, which discards calibration examples whose sampling sets fail to contain admissible answers and calibrates the threshold only on the remaining subset. 
This comparison directly tests the benefit of preserving full calibration-set integrity in MiRD.

\noindent \textbf{Correctness Criteria.} 
We use sentence similarity by default. 
To examine admission-sensitivity, we also use bi-entailment. 
See details in Appendix~\ref{app:Details of Experimental Setup}.

\noindent \textbf{Uncertainty Measure.} 
Following previous studies~\cite{wang-etal-2024-conu,wang2025sample,wang2026safer}, we perform semantic clustering over the candidate set~\citep{kuhn2023semantic} and define the uncertainty score of each candidate as one minus its cluster frequency. 
This yields a normalized uncertainty measure $U(\cdot)\in[0,1]$, where larger estimates represent weaker self-consistency and therefore lower reliability.

\begin{table}[!t]
\centering
\caption{Sampling risk vs. bound on TriviaQA with the LLaMA-3.1-8B-Instruct model, using sentence similarity for correctness evaluation across various thresholds.}
\vspace{-1mm}
\label{tab: sampling risk control admission ablation 1}
\adjustbox{max width=\linewidth}{
    \begin{tabular}{cccccc}
        \toprule
        \multicolumn{2}{c}{\textbf{Threshold /} $\boldsymbol{M}$}  & \textbf{5} & \textbf{10} & \textbf{15} & \textbf{20}\\
        \midrule    
        \multirow{2}{*}{0.5} & Risk & 0.0107 &  0.0067 &  0.0058 &  0.0044\\
        {} & Bound & \textbf{0.0110} &  \textbf{0.0069} &  \textbf{0.0061} &  \textbf{0.0047}\\
        
        \multirow{2}{*}{0.55} & Risk & 0.0249 &  0.0165 &  0.0138 &  0.0112\\
        {} & Bound & \textbf{0.0252} &  \textbf{0.0167} &  \textbf{0.0141} &  \textbf{0.0114}\\
        
        \multirow{2}{*}{0.6} & Risk & 0.0486 &  0.0340 &  0.0279 &  0.0235\\
        {} & Bound & \textbf{0.0491} &  \textbf{0.0344} &  \textbf{0.0282} &  \textbf{0.0239}\\
        
        \multirow{2}{*}{0.65} & Risk & 0.0876 &  0.0632 &  0.0533 &  0.0468\\
        {} & Bound & \textbf{0.0880} &  \textbf{0.0635} &  \textbf{0.0536} &  \textbf{0.0471}\\
        
        \multirow{2}{*}{0.7} & Risk & 0.1349 &  0.1006 &  0.0855 &  0.0757\\
        {} & Bound & \textbf{0.1355} &   \textbf{0.1010} &   \textbf{0.0858} &   \textbf{0.0760}\\
        \bottomrule
    \end{tabular}
}
\vspace{-2mm}
\end{table}

\noindent \textbf{Evaluation Criteria.} 
Following prior studies~\citep{angelopoulos2023conformal,quach2024conformal,wang2025sample,wang2026safer}, we randomly split the calibration and test set 500 times, and calculate the mean of test-time failure (miscoverage) proportion as the Monte Carlo approximation of failure (miscoverage) rate. 
As long as the means do not exceed the corresponding upper bounds, statistical validity holds. 
See \citet{angelopoulos2024theoretical} for details.

\noindent \textbf{Evaluation Protocol.} 
Our theory targets \emph{marginal} validity, so all empirical risks are evaluated by repeated random calibration-test splits~\citep{angelopoulos2023conformal,quach2024conformal,wang2025sample,wang2026safer}. 
We randomly split the data into calibration and test sets 500 times and report the mean test-time risk across splits as a Monte Carlo approximation of the corresponding population quantity. 
We evaluate three risks corresponding to the three layers of our framework. 
In the sampling stage, we report the \emph{sampling-failure risk}, i.e., the average proportion of test examples for which no admissible answer is obtained within a fixed sampling budget. 
In the conformalized selection stage, we report the \emph{conditional selection risk}, i.e., the average miscoverage rate on the subset of test data satisfying $Z_{N+1}(M)=0$. 
For the final prediction sets, we report the \emph{overall miscoverage risk}, i.e., the average rate at which the final set fails to contain any admissible answer over all test data points. 
We regard a statistical guarantee as empirically validated when the corresponding Monte Carlo risk estimate does not exceed its theoretical upper bound. 

\begin{table}[!t]
\centering
\caption{Sampling risk vs. bound on TriviaQA with the LLaMA-3.1-Instruct family of two model sizes, using bi-entailment for correctness evaluation.}
\vspace{-1mm}
\label{tab: sampling risk control admission ablation 2}
\adjustbox{max width=\linewidth}{
    \begin{tabular}{cccccc}
        \toprule
        \multicolumn{2}{c}{\textbf{Model Size /} $\boldsymbol{M}$}  & \textbf{5} & \textbf{10} & \textbf{15} & \textbf{20}\\
        \midrule    
        \multirow{2}{*}{8B} & Risk & 0.1779 & 0.1381 & 0.1196 & 0.1073\\
        {} & Bound & \textbf{0.1782} & \textbf{0.1383} & \textbf{0.1198} & \textbf{0.1074}\\

        \multirow{2}{*}{70B} & Risk & 0.0693 & 0.0557 & 0.0475 & 0.0450\\
        {} & Bound & \textbf{0.0696} & \textbf{0.0560} & \textbf{0.0479} & \textbf{0.0454}\\
        \bottomrule
    \end{tabular}
}
\vspace{-2mm}
\end{table}


\begin{figure*}[!t]
  \centering
  \begin{subfigure}[b]{0.245\textwidth}
    \centering
    \includegraphics[width=\textwidth]{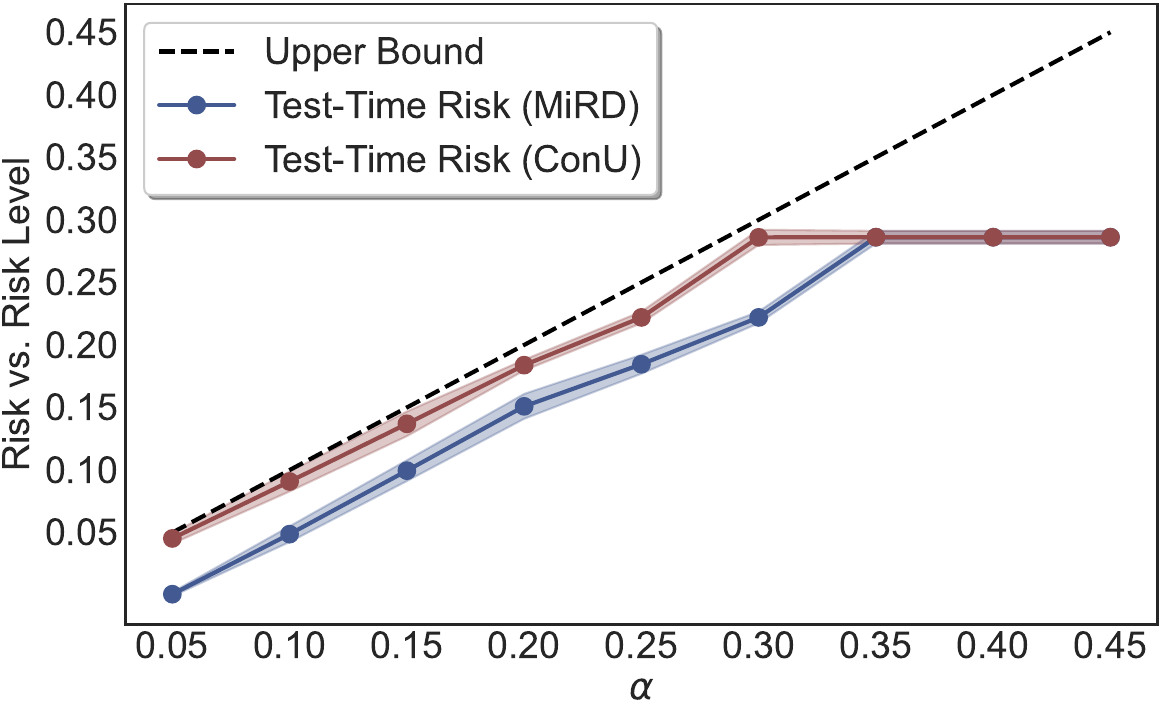}
    \caption{LLaMA-3.1-8B-Instruct.}
  \end{subfigure}
  \hfill
  \begin{subfigure}[b]{0.245\textwidth}
    \centering
    \includegraphics[width=\textwidth]{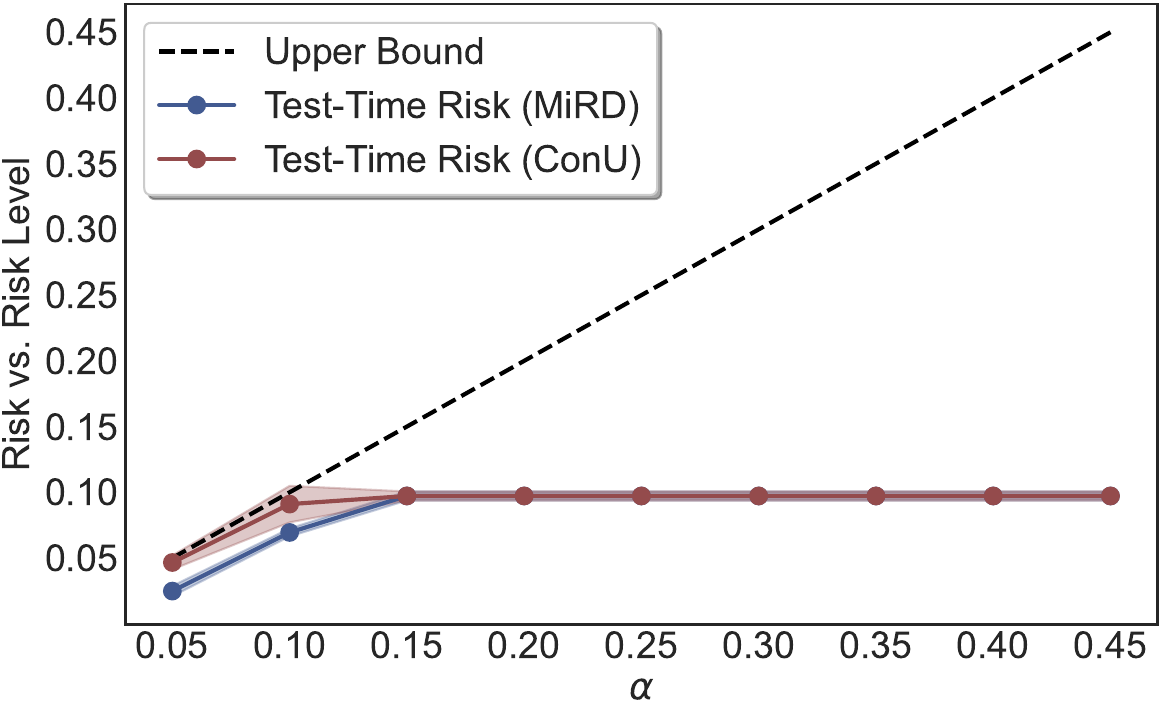}
    \caption{LLaMA-3.1-70B-Instruct.}
  \end{subfigure}
  \hfill
  \begin{subfigure}[b]{0.245\textwidth}
    \centering
    \includegraphics[width=\textwidth]{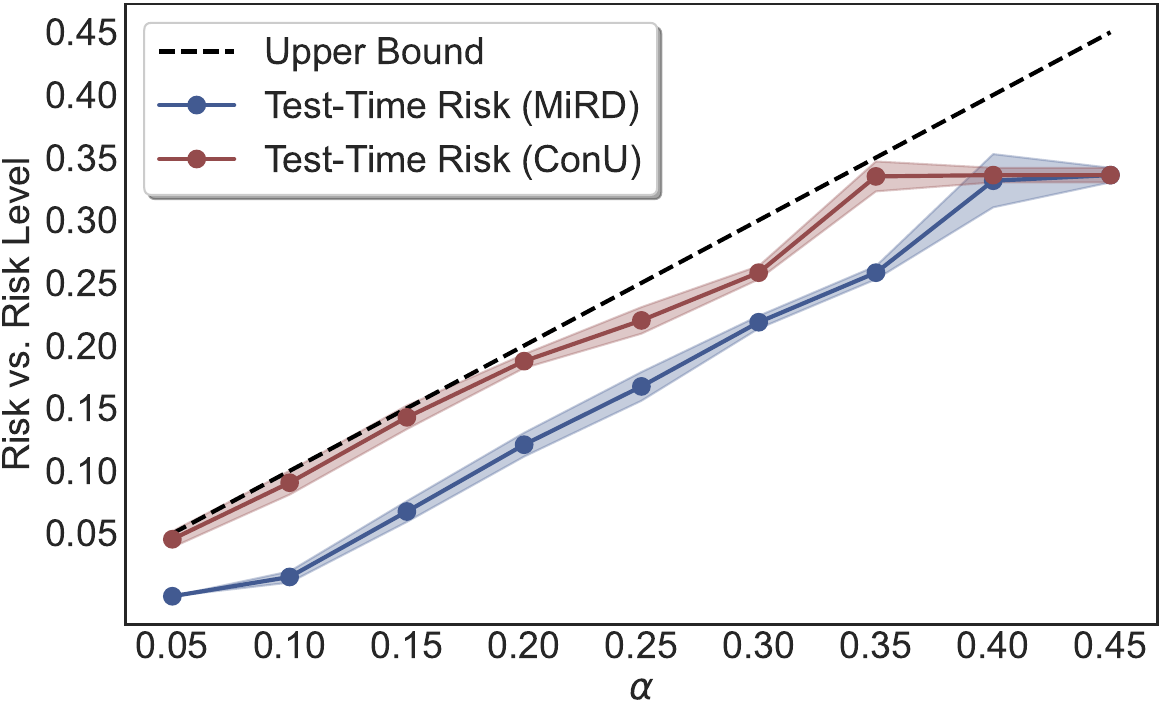}
    \caption{Vicuna-7B-Instruct.}
  \end{subfigure}
  \hfill
  \begin{subfigure}[b]{0.245\textwidth}
    \centering
    \includegraphics[width=\textwidth]{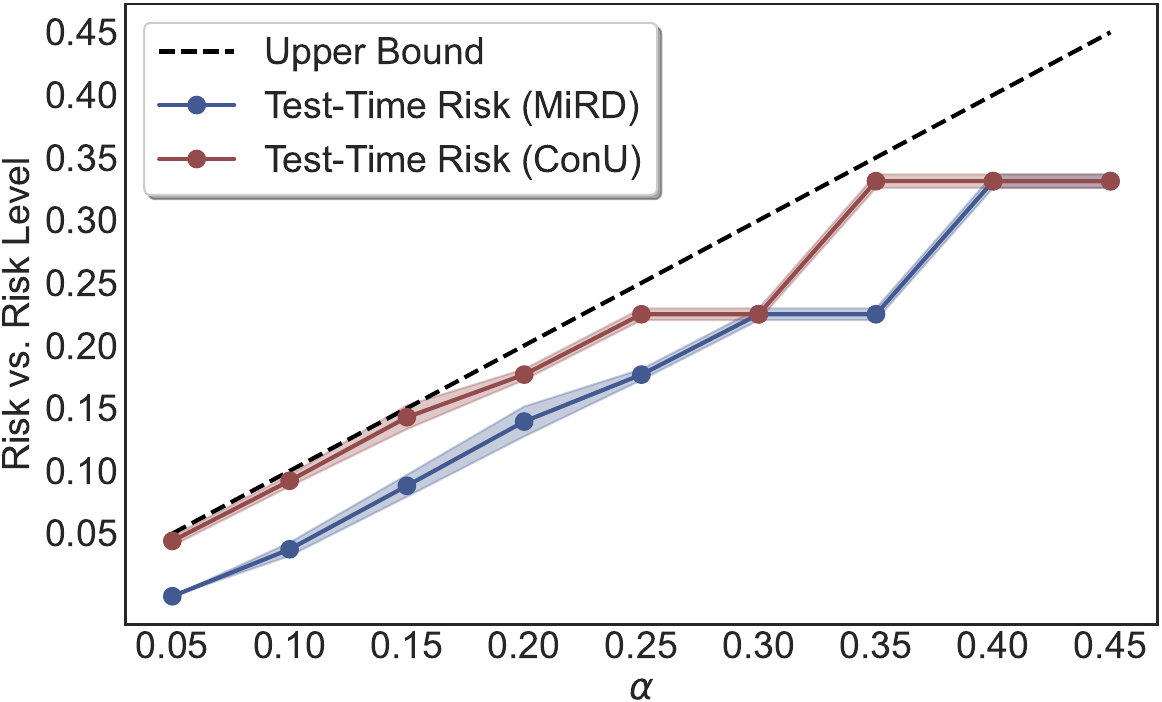}
    \caption{Vicuna-13B-Instruct.}
  \end{subfigure}

  \begin{subfigure}[b]{0.245\textwidth}
    \centering
    \includegraphics[width=\textwidth]{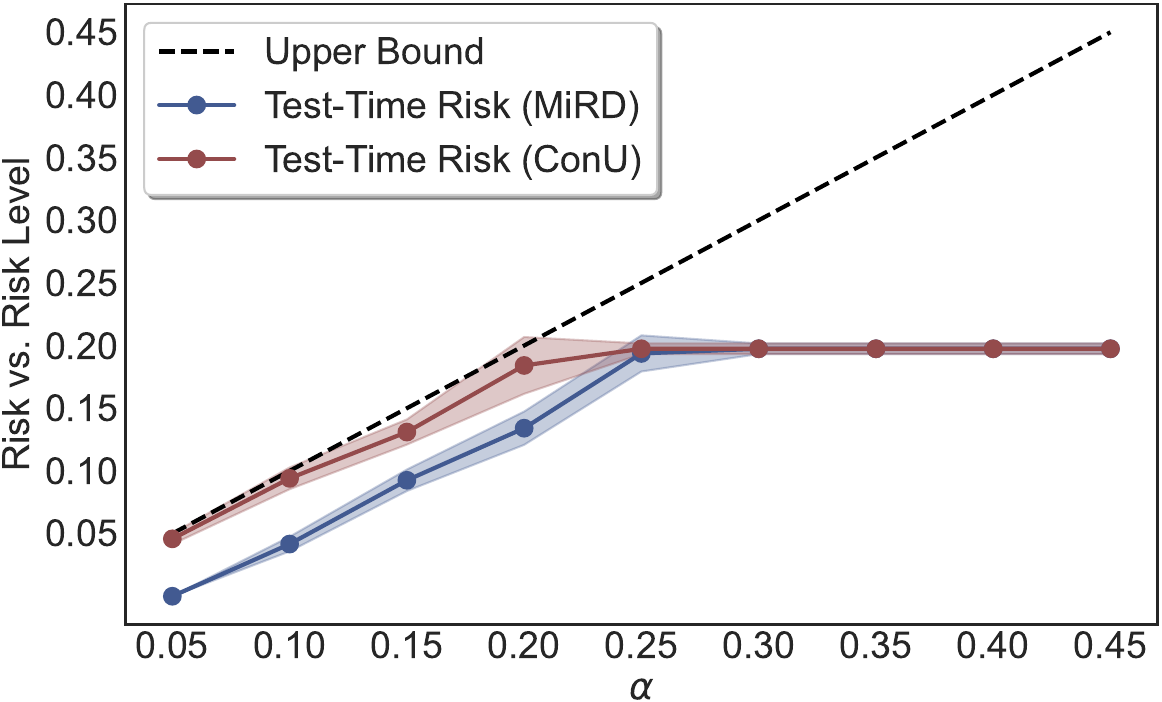}
    \caption{OpenChat-3.5 (7B).}
  \end{subfigure}
  \hfill
  \begin{subfigure}[b]{0.245\textwidth}
    \centering
    \includegraphics[width=\textwidth]{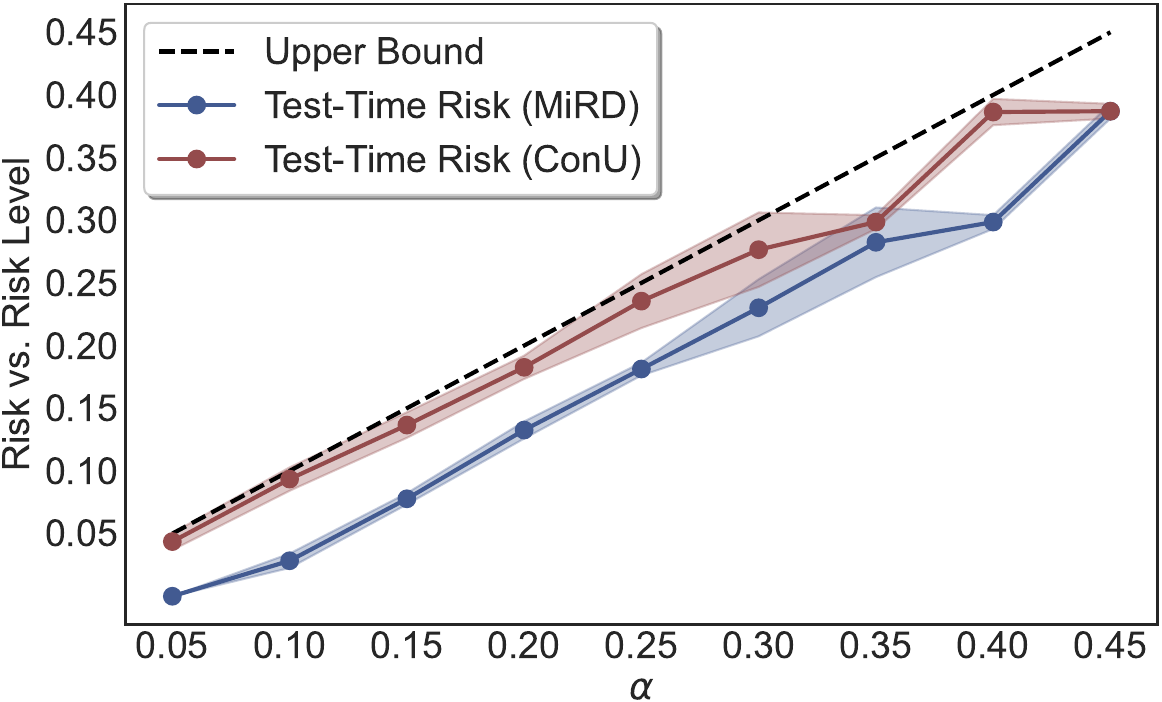}
    \caption{Qwen2.5-3B-Instruct.}
  \end{subfigure}
  \hfill
  \begin{subfigure}[b]{0.245\textwidth}
    \centering
    \includegraphics[width=\textwidth]{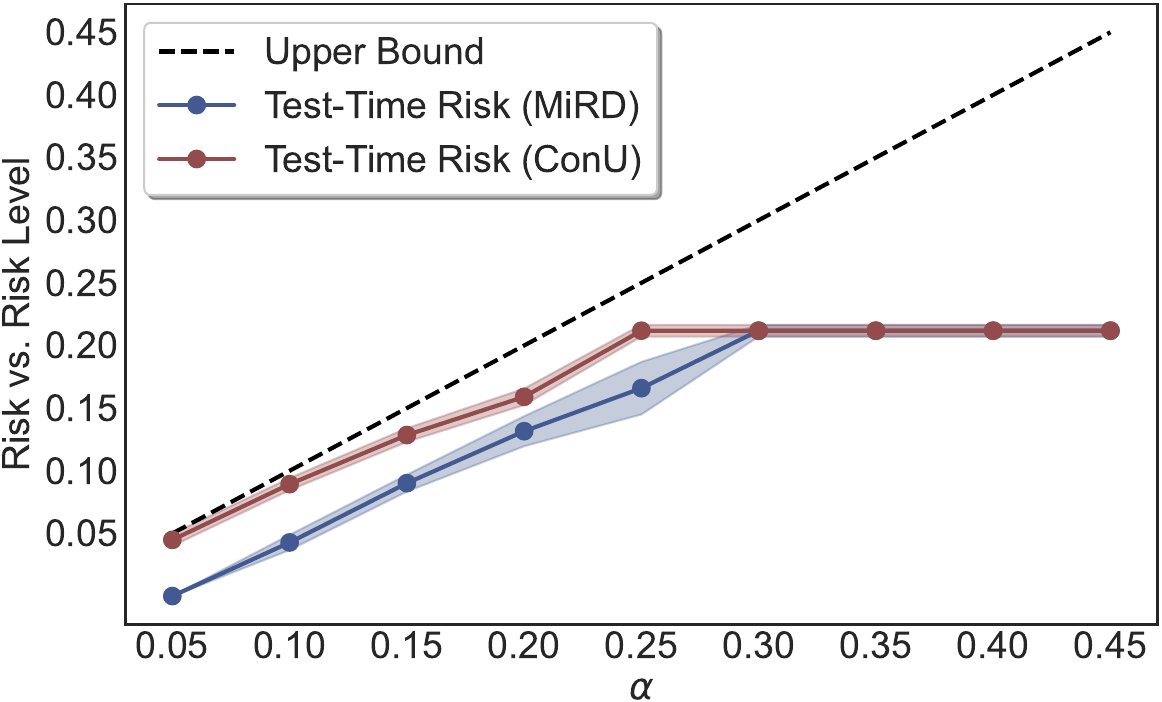}
    \caption{Qwen2.5-7B-Instruct.}
  \end{subfigure}
  \hfill
  \begin{subfigure}[b]{0.245\textwidth}
    \centering
    \includegraphics[width=\textwidth]{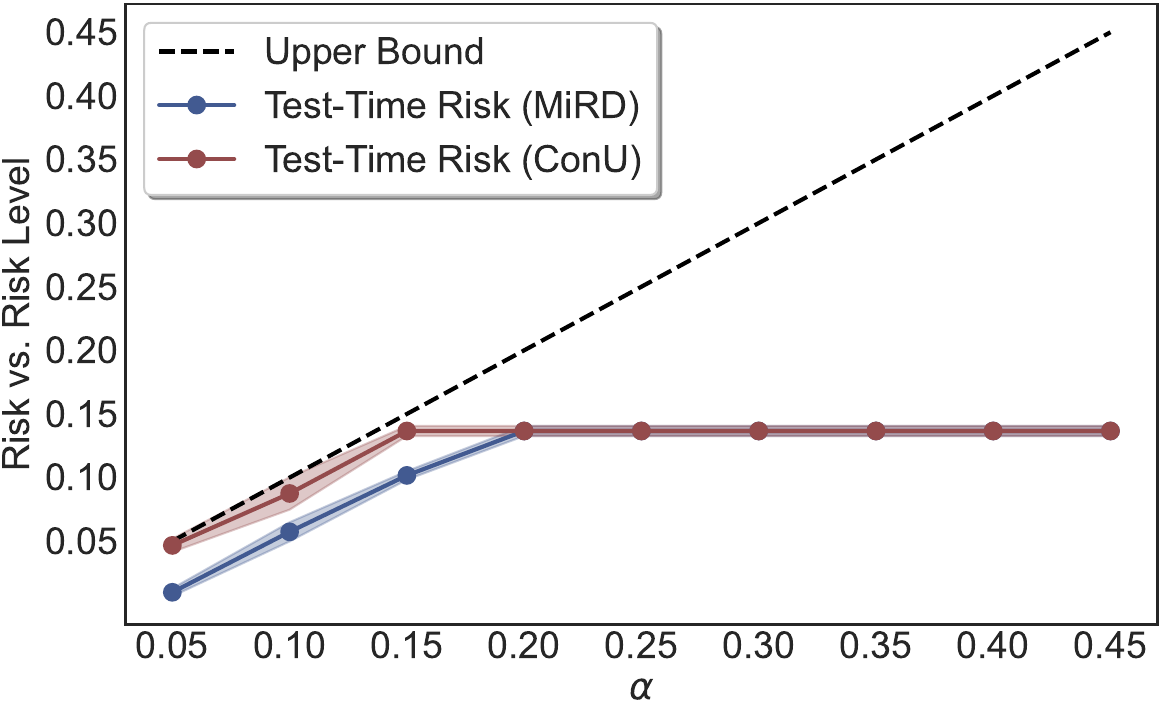}
    \caption{Qwen2.5-14B-Instruct.}
  \end{subfigure}
      \caption{Conditional selection risk vs. upper bound at various risk levels on TriviaQA with eight LLMs ($M=5$).}
  \label{fig: conditional selection risk control triviaqa (similairty-0.6, M=5).}
  \vspace{-5mm}
\end{figure*}

\subsection{Statistical Validity of the Marginal Upper Bound for Finite-Sampling Risk Control } 

We first evaluate the statistical validity of Proposition~\ref{prop:fixedM_marginal_bound_main}. 
Since our first-stage guarantee is a \emph{marginal upper bound in expectation}, rather than a high-probability confidence bound on a fixed parameter, we follow the repeated-splitting protocol in Section~\ref{sec: Experimental Setup} and estimate both the test-time sampling-failure risk and the bound by Monte Carlo averaging over 500 random calibration-test splits. Concretely, for each split and each fixed sampling budget $M$, we compute the test-time proportion of examples whose candidate sets fail to contain any admissible answer, together with the exchangeability-corrected empirical proxy $\widetilde{\beta}(M)$, and then report their averages across splits as approximations of $p_{\mathrm{fail}}(M)$ and $\mathbb{E}[\widetilde{\beta}(M)]$, respectively. 

\begin{figure}[!t]
    \centering
    \includegraphics[width=\linewidth]{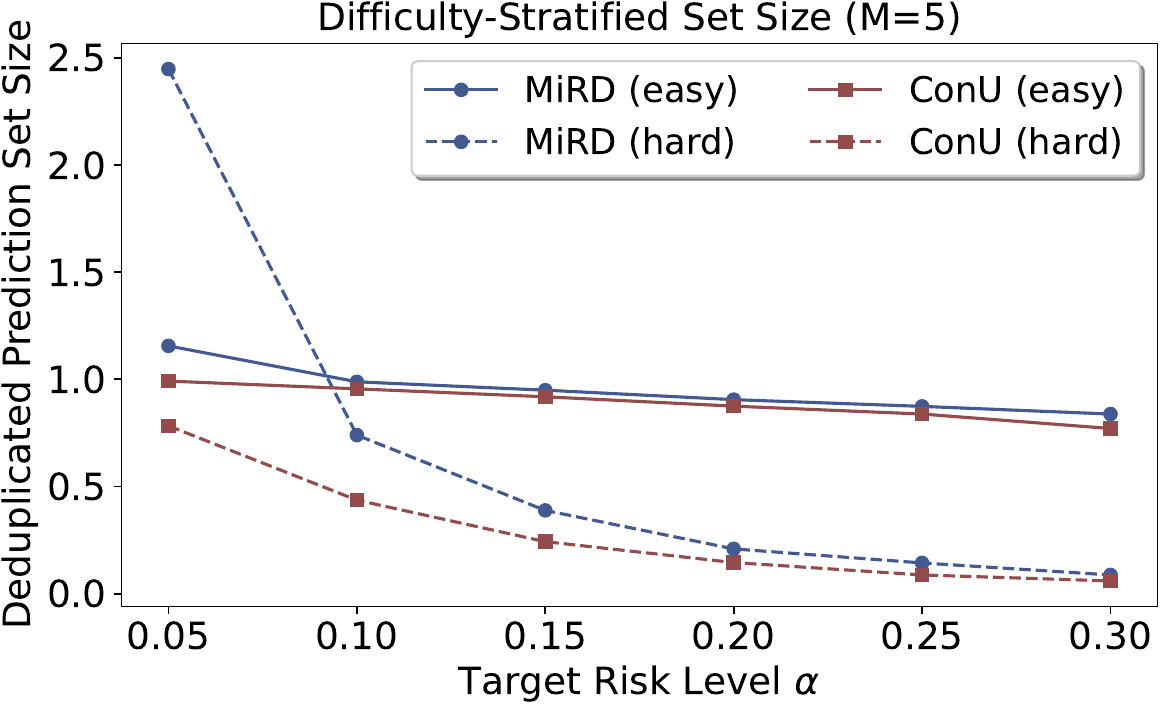}
    \caption{Difficulty-stratified deduplicated prediction set size on TriviaQA with Vicuna-7B-Instruct. Test samples are partitioned into easy and hard groups according to whether the top-1 candidate is admissible.}
    \label{fig:difficulty_stratified_set_size_5}
    \vspace{-5mm}
\end{figure}

As illustrated in Figure~\ref{fig: sampling risk control triviaqa (similairty-0.6).}, across all eight LLMs and all sampling budgets considered, the estimated MiRD bound consistently stays above the empirical sampling-failure risk on TriviaQA, confirming the statistical validity of our marginal upper bound in practice. At the same time, the MiRD bound is uniformly below both UCB-CLP and UCB-HFD, indicating that it is substantially tighter than PAC-style confidence bounds while still remaining valid. This is precisely the advantage of the expectation-level formulation in Proposition~\ref{prop:fixedM_marginal_bound_main}: it yields a non-vacuous and empirically sharp characterization of finite-sampling risk without introducing an additional (external) confidence parameter $\delta$.

Tables~\ref{tab: sampling risk control admission ablation 1} and~\ref{tab: sampling risk control admission ablation 2} further demonstrate that this validity is robust to the admissibility criterion. Under sentence similarity, the bound continues to closely track the empirical sampling-failure risk across a range of semantic thresholds; under bi-entailment, the same behavior is observed for both 8B and 70B backbone models. These results suggest that the first-stage validity of MiRD is stable not only across model families and sampling budgets, but also across different notions of answer correctness.

\begin{figure}[!t]
    \centering
    \includegraphics[width=\linewidth]{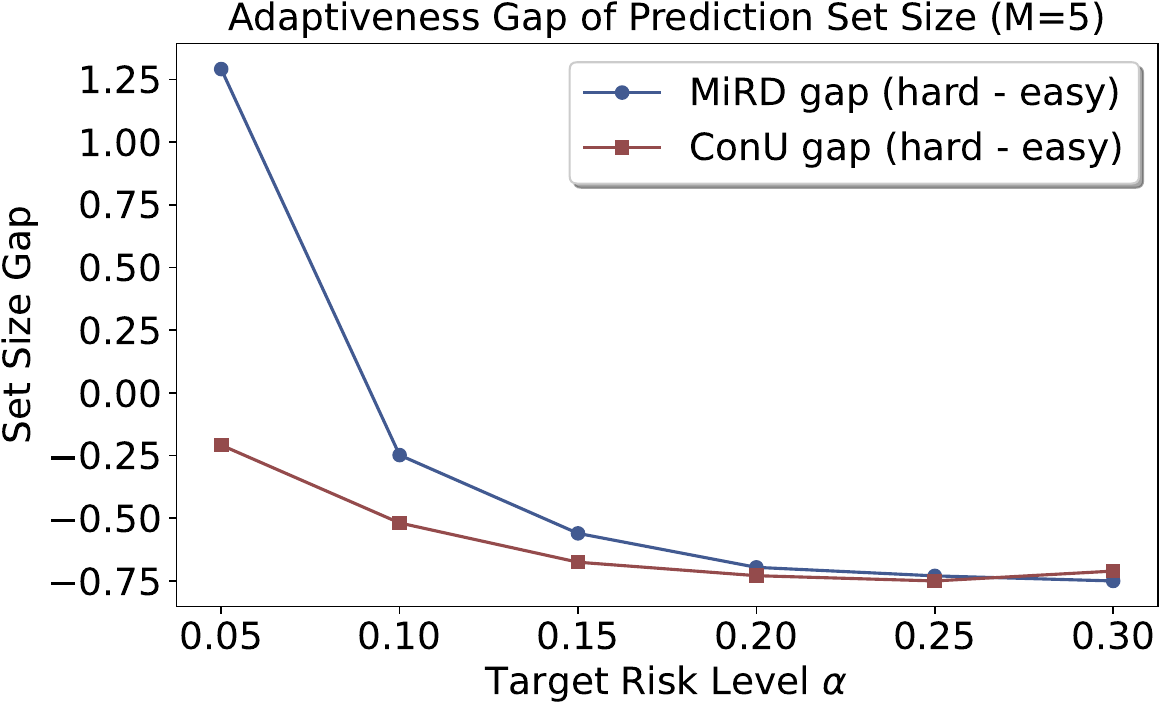}
    \caption{Adaptiveness gap of prediction set size on TriviaQA with Vicuna-7B-Instruct, defined as the difference between the average deduplicated prediction set sizes of hard and easy examples.}
    \label{fig:difficulty_stratified_gap_5}
    \vspace{-5mm}
\end{figure}

\begin{figure*}[!t]
  \centering
  \begin{subfigure}[b]{0.245\textwidth}
    \centering
    \includegraphics[width=\textwidth]{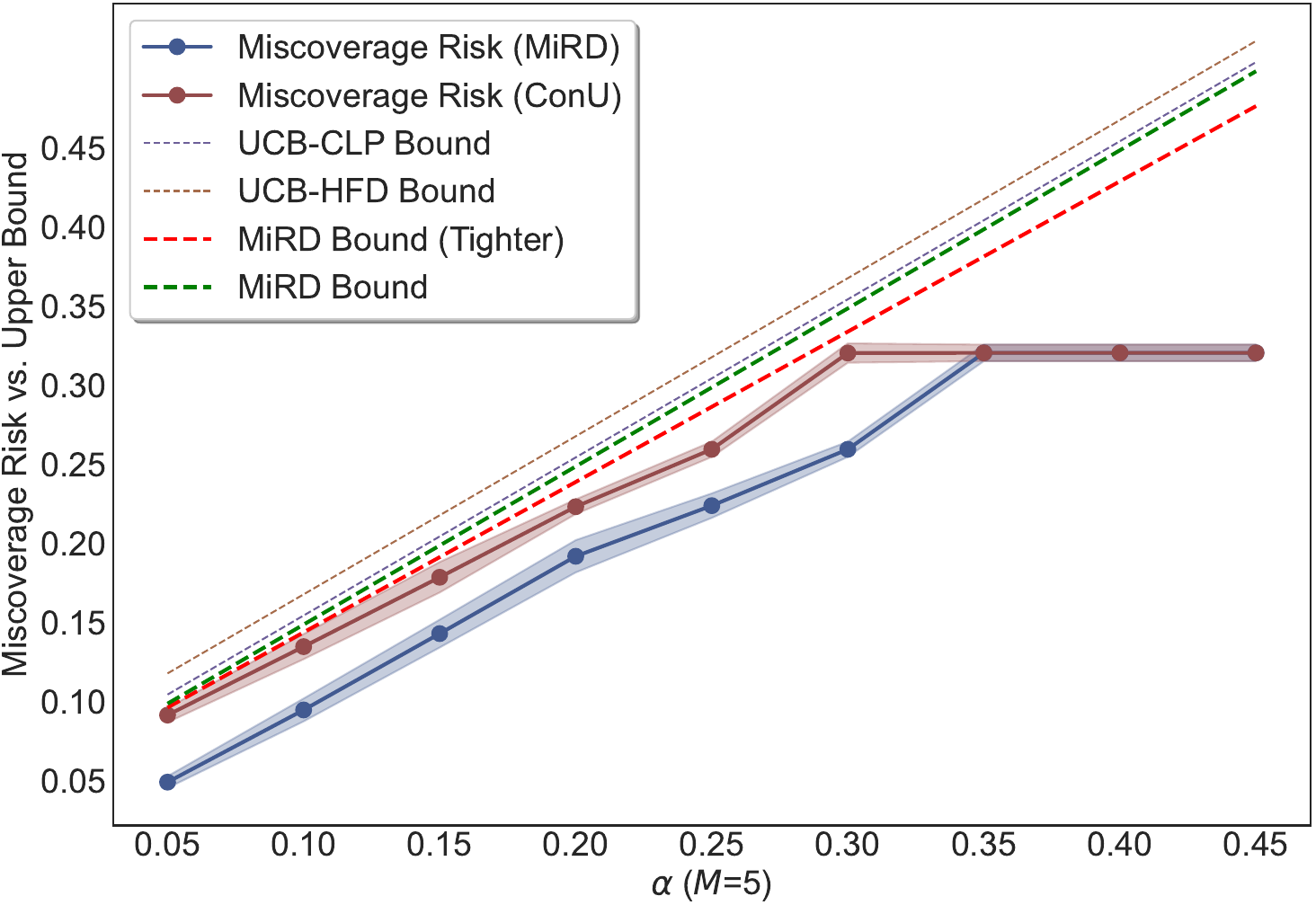}
    \caption{LLaMA-3.1-8B-Instruct.}
  \end{subfigure}
  \hfill
  \begin{subfigure}[b]{0.245\textwidth}
    \centering
    \includegraphics[width=\textwidth]{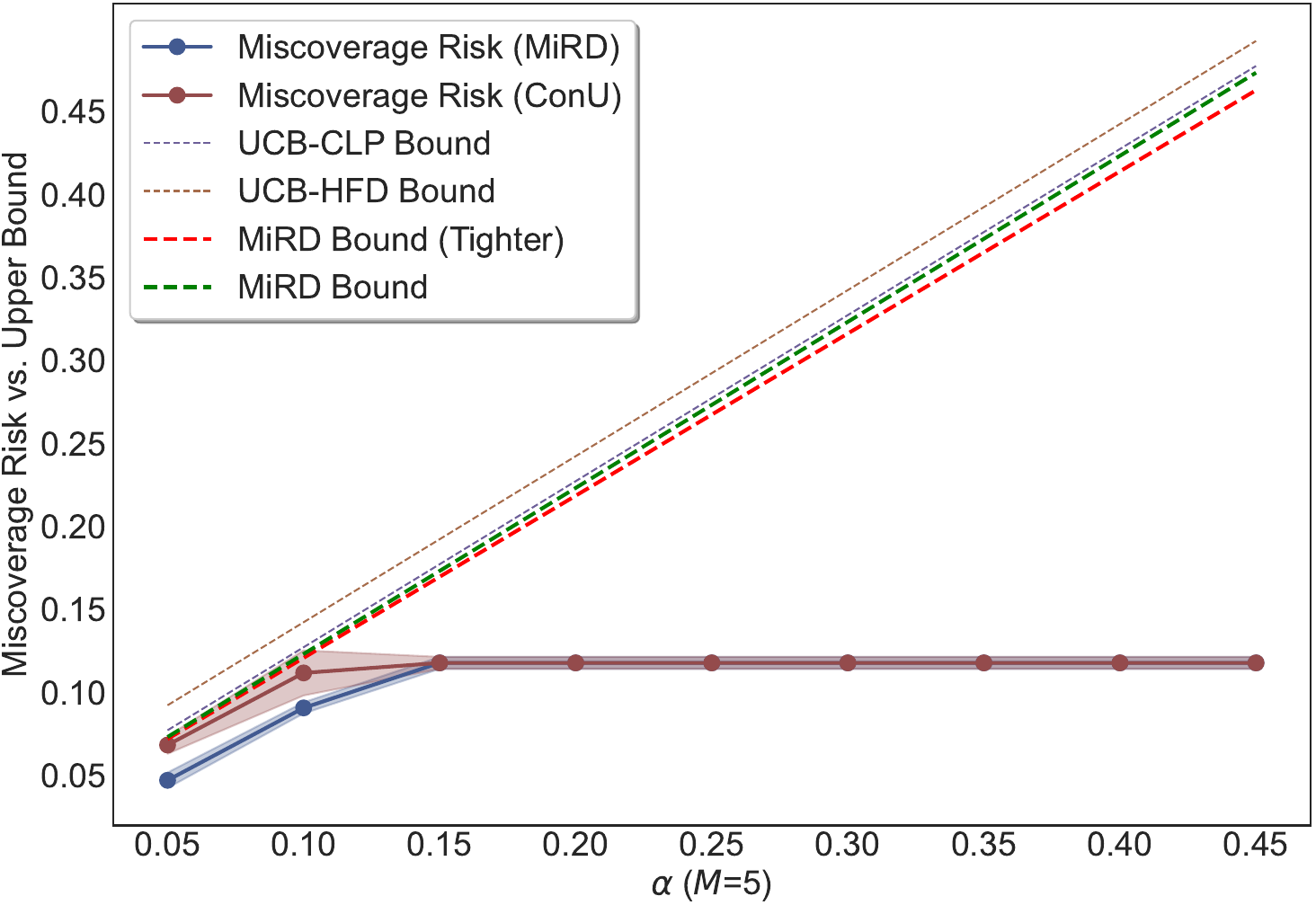}
    \caption{LLaMA-3.1-70B-Instruct.}
  \end{subfigure}
  \hfill
  \begin{subfigure}[b]{0.245\textwidth}
    \centering
    \includegraphics[width=\textwidth]{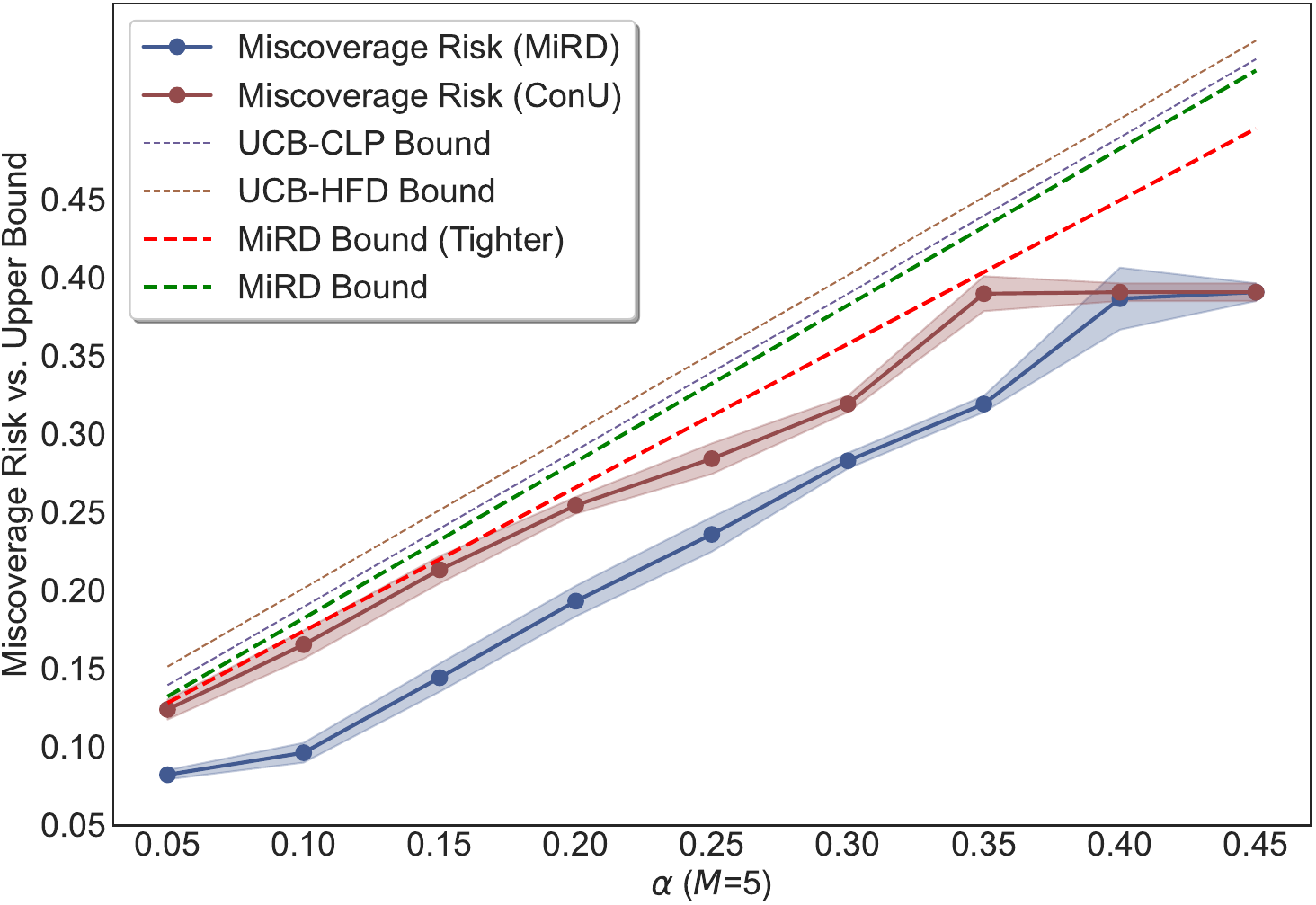}
    \caption{Vicuna-7B-Instruct.}
  \end{subfigure}
  \hfill
  \begin{subfigure}[b]{0.245\textwidth}
    \centering
    \includegraphics[width=\textwidth]{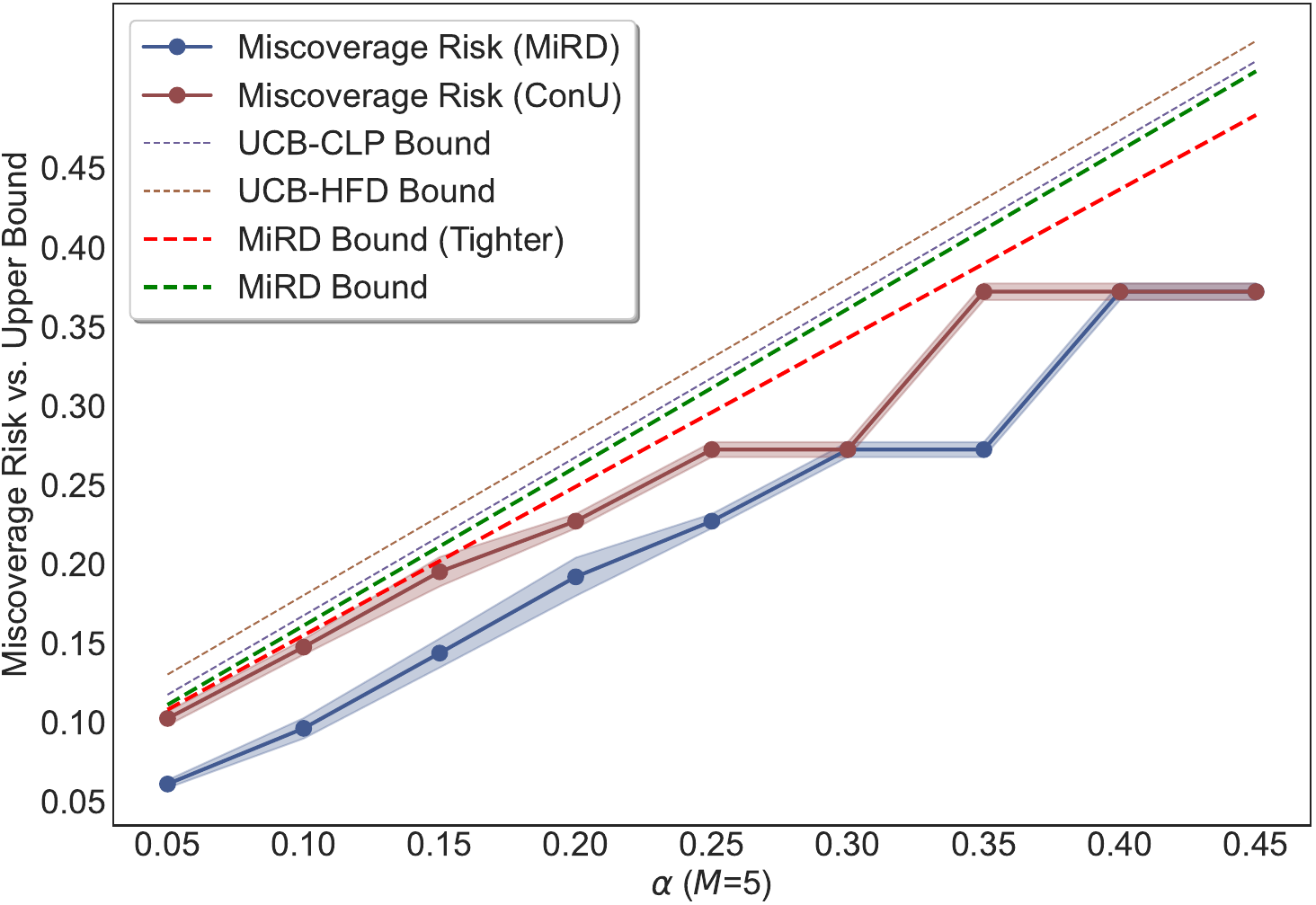}
    \caption{Vicuna-13B-Instruct.}
  \end{subfigure}

  \begin{subfigure}[b]{0.245\textwidth}
    \centering
    \includegraphics[width=\textwidth]{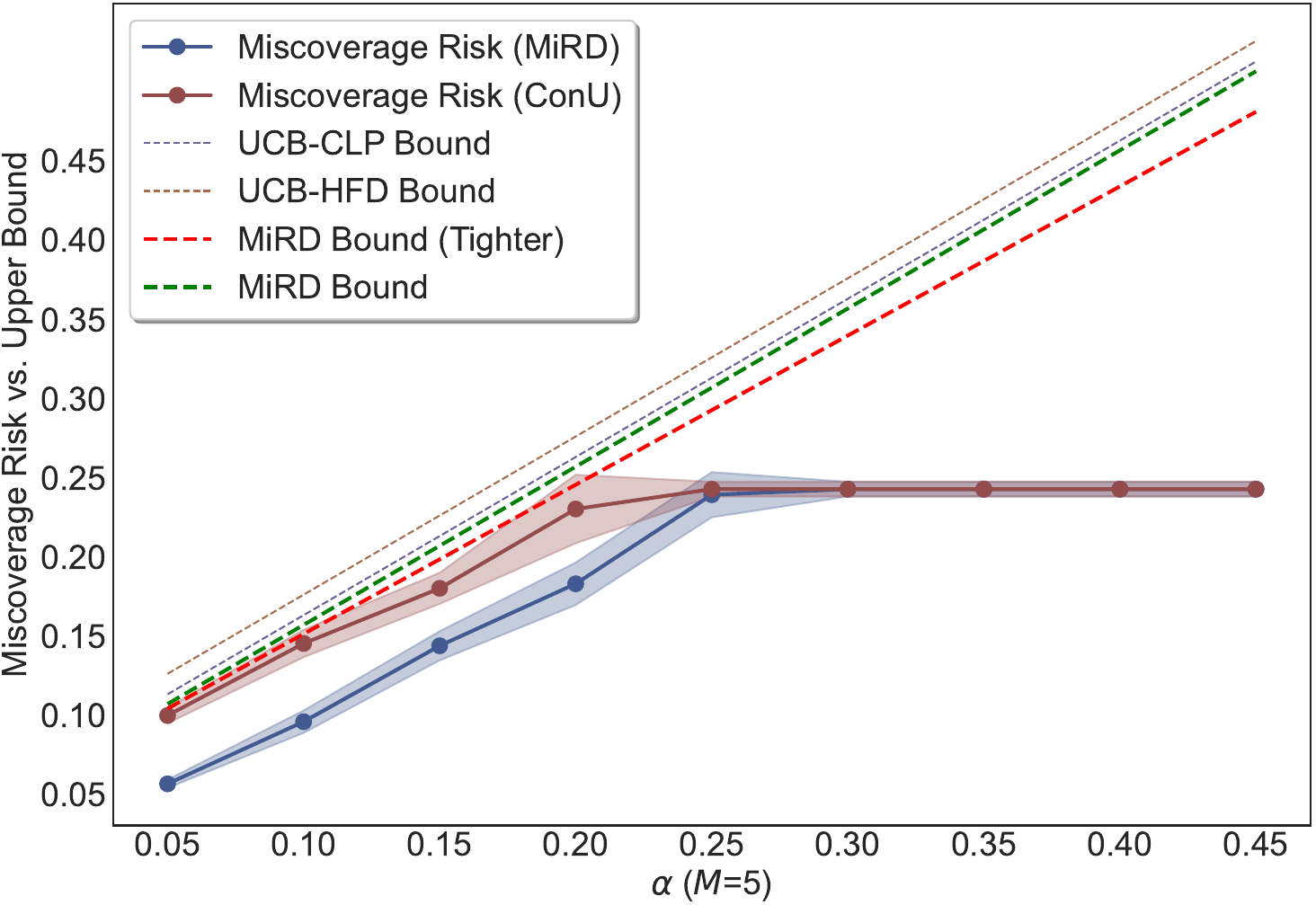}
    \caption{OpenChat-3.5 (7B).}
  \end{subfigure}
  \hfill
  \begin{subfigure}[b]{0.245\textwidth}
    \centering
    \includegraphics[width=\textwidth]{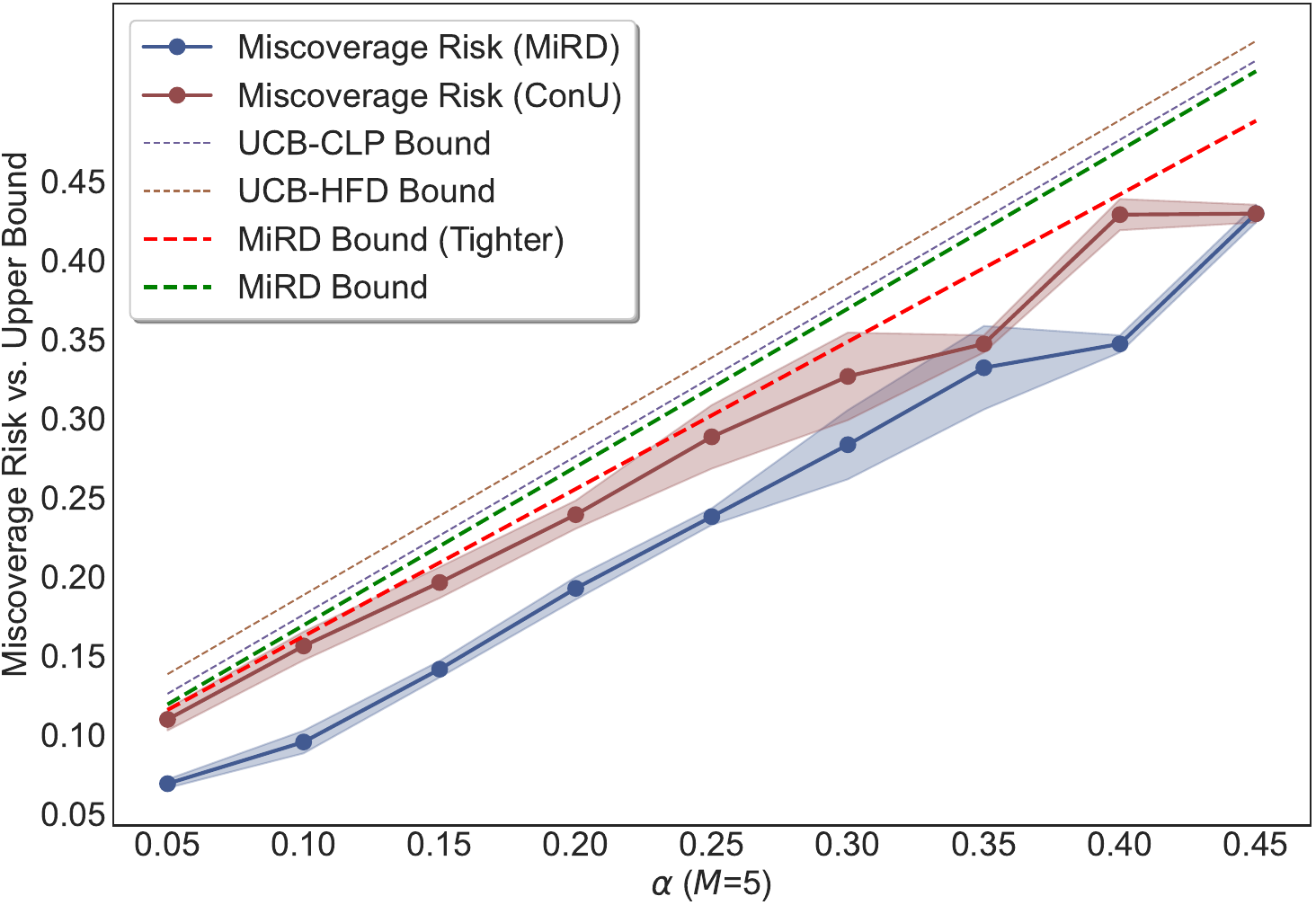}
    \caption{Qwen2.5-3B-Instruct.}
  \end{subfigure}
  \hfill
  \begin{subfigure}[b]{0.245\textwidth}
    \centering
    \includegraphics[width=\textwidth]{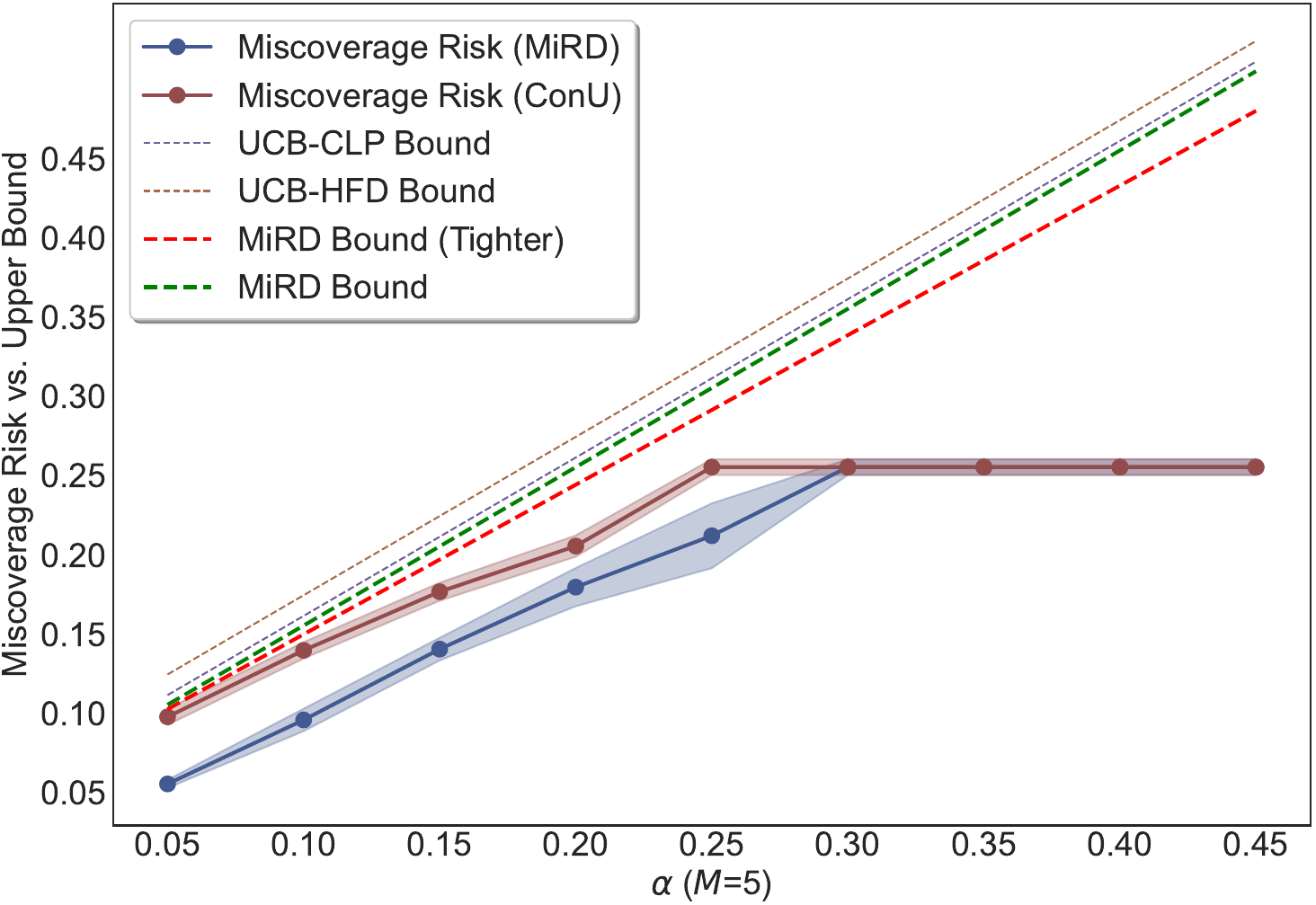}
    \caption{Qwen2.5-7B-Instruct.}
  \end{subfigure}
  \hfill
  \begin{subfigure}[b]{0.245\textwidth}
    \centering
    \includegraphics[width=\textwidth]{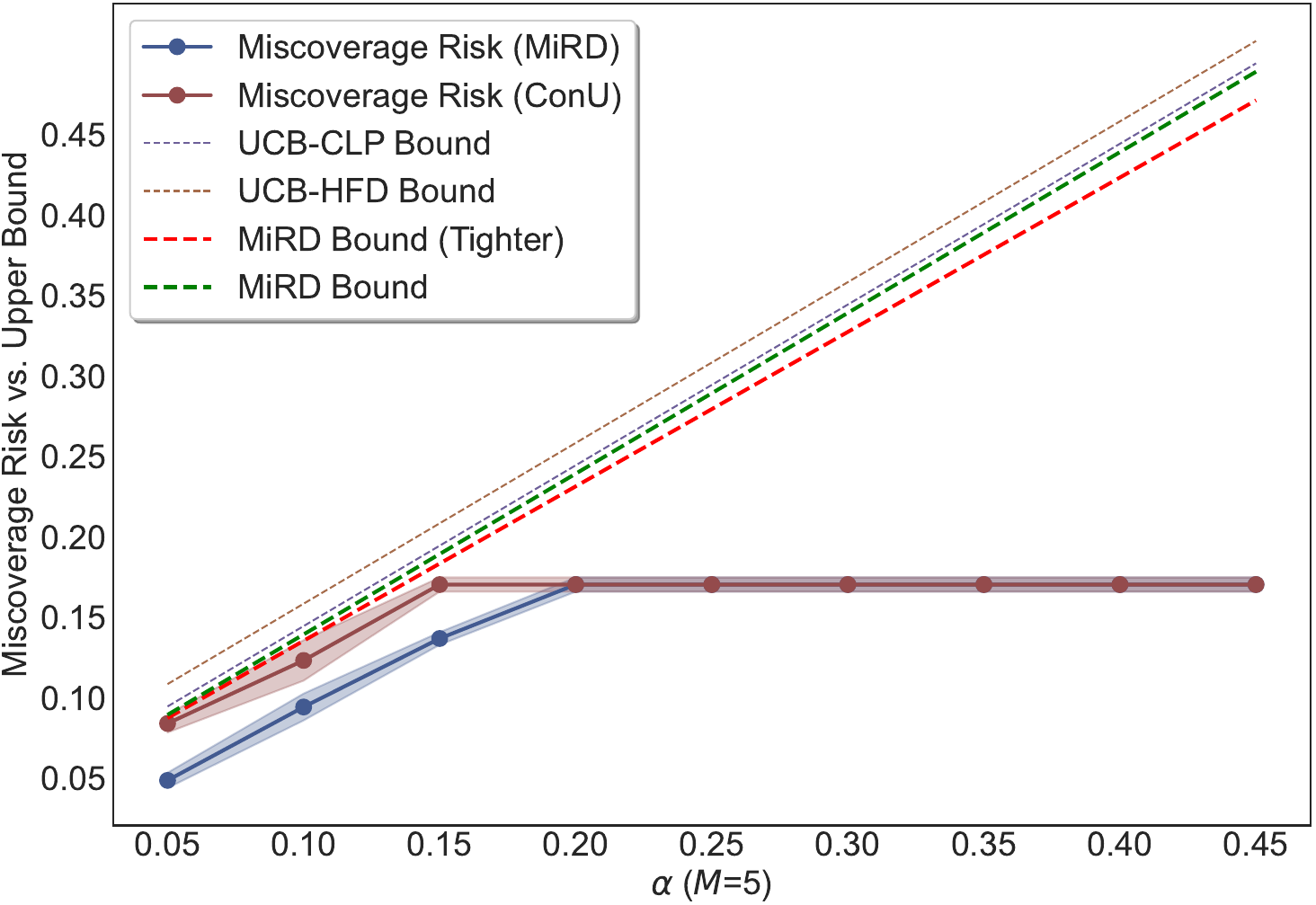}
    \caption{Qwen2.5-14B-Instruct.}
  \end{subfigure}

      \caption{Overall miscoverage risk vs. upper bound at various risk levels on TriviaQA with eight LLMs ($M=5$).}
  \label{fig: miscoverage risk control triviaqa (similairty-0.6, M=5).}
  \vspace{-5mm}
\end{figure*}

\subsection{Conditional Selection Risk Control under Full-Calibration Thresholding}

We next evaluate the statistical validity of Stage II, fixing $M=5$ on TriviaQA. 
As shown in Figure~\ref{fig: conditional selection risk control triviaqa (similairty-0.6, M=5).}, MiRD consistently controls the conditional selection risk below the target $\alpha$ across all eight LLMs, validating the Stage-II guarantee in practice. 
Moreover, MiRD achieves uniformly lower test-time conditional miscoverage than ConU under the same risk level. 
This behavior is expected: by retaining the full calibration set, MiRD calibrates its threshold against a larger admission-correlated score distribution, which results in a weakly larger threshold than successful-only calibration and therefore a weakly larger prediction set on sampling-successful test examples. 
As discussed in Appendix~\ref{app:conditional_exchangeability}, this monotone enlargement preserves coverage and explains why MiRD is empirically more conservative than ConU in the selection stage.

However, lower conditional miscoverage alone does not fully characterize the utility of a set-valued predictor, because a larger threshold typically leads to a larger prediction set. This should not be viewed as a weakness by default. In set-valued prediction, the role of prediction-set size is not to be minimized uniformly, but to adapt to sample difficulty and reflect model uncertainty. A useful predictor should allocate larger sets to more ambiguous examples, rather than collapsing to uniformly small sets.

To examine this behavior, we analyze deduplicated prediction set size using Vicuna-7B-Instruct on TriviaQA, stratifying sampling-successful test examples into easy and hard groups according to whether the top-1 candidate answer is admissible. As demonstrated in Figure~\ref{fig:difficulty_stratified_set_size_5}, MiRD yields consistently larger and more stable prediction sets than ConU for both groups, avoiding the rapid shrinkage observed under successful-only calibration. More importantly, Figure~\ref{fig:difficulty_stratified_gap_5} shows that MiRD preserves a larger adaptiveness gap between hard and easy examples across risk levels. This indicates that the additional set size introduced by MiRD is not merely uniform inflation; rather, it better retains difficulty-dependent variation and makes prediction-set size a more informative indicator of sample uncertainty.

Taken together, these results suggest that preserving full calibration-set integrity improves not only conditional risk control, but also the adaptiveness of the resulting set-valued predictions.

\subsection{Overall Miscoverage Risk Control}

We finally evaluate the end-to-end reliability. 
Since Stage I controls the marginal risk of finite-sampling failure, and Stage II controls conditional selection risk given sampling success, the resulting overall prediction set is expected to satisfy the population-level guarantees in Section~\ref{sec:overall_guarantee_main} and Appendix~\ref{app:tighter_bound}. 
Figure~\ref{fig: miscoverage risk control triviaqa (similairty-0.6, M=5).} confirms this empirically on TriviaQA: across all eight LLMs and all target risk levels considered, the test-time overall miscoverage of MiRD remains below both the conservative main-text bound and the tighter refinement in Appendix~\ref{app:tighter_bound}.

Moreover, the tighter MiRD bound consistently tracks the empirical overall miscoverage at test time more closely than the conservative bound, confirming that the overlap correction in Appendix~~\ref{app:tighter_bound} is meaningful in practice rather than merely algebraic. In contrast, UCB-CLP and UCB-HFD remain substantially looser, reflecting the conservativeness of PAC-style confidence bounds when propagated through the two-stage pipeline. 
Meanwhile, compared with ConU, MiRD achieves competitive or lower overall miscoverage while additionally accounting for finite-sampling failure, showing that preserving full calibration-set integrity improves end-to-end reliability rather than merely enlarging intermediate prediction sets. 

Additional results on the CoQA and NQ datasets, and broader settings are provided in Appendix~\ref{app:Additional Experimental Results}.

\section{Conclusion}

We present MiRD, a two-stage framework for reliable set-valued prediction in open-ended QA. 
The central idea is to explicitly decompose overall miscoverage into two sources: marginal sampling failure and conditional selection failure. Based on this decomposition, MiRD first establishes a marginal upper bound in expectation for the finite-sampling risk, and further calibrates a conformalized selection threshold using the full calibration set through ground-truth-guided nonconformity estimates. 
Empirically, MiRD achieves valid risk control in both stages and for the overall prediction set across diverse LLM families, while producing tighter first-stage bounds than PAC-style confidence intervals and more adaptive prediction sets than successful-only calibration. 
These results demonstrate that preserving calibration-set integrity is beneficial not only for statistical validity, but also for the utility of set-valued prediction in open-ended generation. 

\section*{Limitations}

MiRD inherits several standard assumptions from conformal prediction and open-ended generation. 
First, its guarantees rely on the exchangeability between calibration and test examples, as in standard split conformal prediction. 
This assumption is appropriate for the marginal validity studied in this work, but substantial distribution shifts may require additional adaptation or recalibration mechanisms. 
Second, the selection-stage guarantee uses a score-level conditional compatibility assumption. 
We make this assumption explicit and provide a discussion in the appendix; developing weaker or empirically testable variants of this condition is an interesting direction for future work. 
Finally, MiRD studies risk control under a fixed finite sampling budget. 
This setting allows a clean decomposition of sampling failure and post-selection failure, while adaptive budget allocation could further improve the cost--risk trade-off in future deployments.

\bibliography{custom}

\newpage
\appendix

\section{Proofs and Discussions}
\label{sec: proofs}
\subsection{Proof of Eq.~\eqref{eq:risk_decomposition_main}}
By the law of total probability,
\[
\begin{split}
&\!\Pr(R_{N+1}(\lambda)=1)=\\
&\!\Pr(R_{N+1}(\lambda)=1\mid Z_{N+1}=1)\!\Pr(Z_{N+1}=1)+\\
&\!\Pr(R_{N+1}(\lambda)=1\mid Z_{N+1}=0)\!\Pr(Z_{N+1}=0).
\end{split}
\]
When $Z_{N+1}=1$, no admissible answer appears in the candidate set $\mathcal{G}_M(x_{N+1})$. Since $\mathcal{C}_\lambda(x_{N+1})\subseteq \mathcal{G}_M(x_{N+1})$, it follows that $R_{N+1}(\lambda)=1$ almost surely, so
\[
\Pr(R_{N+1}(\lambda)=1\mid Z_{N+1}=1)=1.
\]
This yields Eq.~\eqref{eq:risk_decomposition_main}.

\subsection{Proof of Proposition~\ref{prop:fixedM_marginal_bound_main}}

\begin{proof}
By exchangeability, the random variables
\vspace{-0.2cm}
\[
Z_1(M),\ldots,Z_N(M),Z_{N+1}(M)
\]
have the same marginal distribution. Hence,
\vspace{-0.2cm}
\begin{equation}
\begin{split}
    \mathbb{E}[Z_{N+1}(M)]
&=
\frac{1}{N+1}\sum_{i=1}^{N+1}\mathbb{E}[Z_i(M)]\\
&=
\mathbb{E}\!\left[
\frac{1}{N+1}\sum_{i=1}^{N+1} Z_i(M)
\right].
\end{split}
\label{eq:exchangeability_avg_main}
\end{equation}
Using $Z_{N+1}(M)\le 1$, we obtain
\vspace{-0.2cm}
\[
\begin{split}
    &\quad \frac{1}{N+1}\sum_{i=1}^{N+1} Z_i(M)\\
&=\!
\frac{1}{N+1}\sum_{i=1}^{N} Z_i(M)
\!+\!
\frac{1}{N+1}Z_{N+1}(M)\\
&\le\!
\frac{1}{N+1}\sum_{i=1}^{N} Z_i(M)
\!+\!
\frac{1}{N+1}.
\end{split}
\]
Taking expectations on both sides and using Eq.~\eqref{eq:emp_sampling_failure_main} gives
\vspace{-0.2cm}
\[
\mathbb{E}[Z_{N+1}(M)]
\le
\mathbb{E}\!\left[
\frac{N}{N+1}\widehat{R}_N(M)+\frac{1}{N+1}
\right].
\]
Combining this with Eq.~\eqref{eq:pfail_mean_equiv_main} proves the statistical validity of Proposition~\ref{prop:fixedM_marginal_bound_main}.
\end{proof}

\subsection{Conditional Exchangeability Assumption}
\label{app:conditional_exchangeability}

The second-stage guarantee in Section~\ref{sec:conformal_selection_main} requires that, conditional on $\{Z_{N+1}=0\}$, the augmented admission-correlated score sequence
\[
\{u_1,\dots,u_N,u_{N+1}\}
\]
is exchangeable, a score-level assumption tailored to our construction and stronger than the exchangeability of the underlying data points alone.

The requirement for this assumption arises from an intrinsic asymmetry in our two-stage framework. In Stage I, the sampling failure is defined with respect to the finite candidate set generated by the model. 
In Stage II, however, conformalized selection is only meaningful on test data points for which sampling has already succeeded (i.e., $Z_{N+1}=0$). By contrast, at calibration time, the ground-truth answer is always available, which allows us to define an admission-correlated score for \emph{every} calibration data point, regardless of whether its sampled candidate set contains an admissible answer. Consequently, the calibration-side score distribution is constructed on the full calibration set, while the test-side score is only employed on the sampling-successful subset of the test distribution.

This asymmetry is inherent to the proposed risk decomposition, but rather a direct consequence of the risk decomposition in Eq.~\eqref{eq:risk_decomposition_main}: the first-stage term accounts for whether an admissible answer appears at all, while the second-stage term only concerns whether such an answer is retained after conformalized filtering. 
Since Stage II is conditioned on $\{Z_{N+1}=0\}$, its theoretical validity does not require the calibration and test candidate sets themselves to be symmetric in a pointwise sense. Instead, what is required is that the calibration-time ground truth scores provide a valid surrogate for the test-time admission-correlated reference (admissible) score under the same conditional regime.

More concretely, the calibration scores are defined as
\[
u_i = U_G(y_i^* \mid x_i),
\]
whereas the test score is defined using the lowest-uncertainty admissible sampled answer,
\[
u_{N+1} = U_G(\hat y^{(\mathrm{ref})}_{N+1}\mid x_{N+1}),
\]
where
\[
\hat y^{(\mathrm{ref})}_{N+1}
\in
\operatorname*{\arg\min}_{\hat y \in G_M(x_{N+1}):\,A(y_{N+1}^*,\hat y)=1}
U_G(\hat y \mid x_{N+1}).
\]
Therefore, exchangeability required in Section~\ref{sec:conformal_selection_main} is not implied automatically by the exchangeability of the raw data points $(x_i,y_i^*)$. Rather, it is a compatibility assumption asserting that, conditional on sampling success, ground truth score distribution on calibration data is symmetric with reference admissible-score distribution on the future test data. 
Thus, the assumption formalizes the precise role of  ground-truth answers: they are not injected into the candidate sets, nor do they alter the first-stage sampling event. Instead, they are used only to construct admission-correlated NSs on the calibration side, so as to retain all calibration examples while avoiding the exchangeability mismatch induced by augmentation-based designs. The validity of Stage II therefore hinges on score-level conditional compatibility, rather than candidate-set-level identity between calibration and test examples.

Then, we note that this assumption is empirically testable in a limited but informative sense. On calibration examples with $Z_i(M)=0$, one may compare the ground truth score $U_G(y_i^*\mid x_i)$ against the corresponding admissible sampled-answer reference score
\[
\min_{\hat y \in G_M(x_i):\,A(y_i^*,\hat y)=1} U_G(\hat y\mid x_i),
\]
to evaluate whether the two scores are well aligned in distribution or rank. Such analyses do not prove the assumption, but they offer evidence for the plausibility of using ground truth scores as surrogates for test-time admission-correlated sampled-answer scores.

\noindent \textbf{A simpler interpretation based on nested conformal sets.} 
For the specific nonconformity construction considered here, the Stage-II validity of MiRD also admits a simpler interpretation through threshold monotonicity. 
Previous methods such as ConU~\citep{wang-etal-2024-conu} assume that both calibration and test sampling sets contain at least one admissible candidate. Under this assumption, the admission-correlated NSs on calibration and test examples are exchangeable, and the resulting conformal threshold is therefore statistically valid.

MiRD differs from this successful-only calibration view only in that it retains the full calibration set. Concretely, for calibration data points whose sampling sets already contain admissible answers, MiRD uses the same admission-correlated nonconformity scores as ConU; for calibration examples whose sampling sets fail, MiRD effectively adds additional maximal scores equal to $1$ (let $U(\cdot)\leq1$). Therefore, under the same target risk level $\alpha$, the empirical score distribution used by MiRD is obtained by augmenting the successful-only score distribution with extra values at the upper endpoint. As a result, the calibrated quantile cannot decrease:
\[
\hat\lambda_{\mathrm{MiRD}} \ge \hat\lambda_{\mathrm{ConU}}.
\]

Since the conformal prediction sets are nested in the threshold,
\[
\lambda_1 \le \lambda_2
\quad\Longrightarrow\quad
C_{\lambda_1}(x)\subseteq C_{\lambda_2}(x),
\]
we further have
\[
C_{\hat\lambda_{\mathrm{ConU}}}(x)\subseteq C_{\hat\lambda_{\mathrm{MiRD}}}(x)
\qquad \text{for every } x.
\]
Hence, whenever the successful-only conformal predictor achieves coverage, the MiRD prediction set, being a superset under the same $\alpha$, also preserves that coverage. In this sense, for the present nonconformity design, MiRD can be viewed as a monotone enlargement of the successful-only conformal predictor, which offers an intuitive explanation for its Stage-II statistical validity.

This interpretation is specific to the current nonconformity construction, where failed calibration examples contribute maximal scores, and should be understood as a complementary justification rather than a replacement for the more general score-level conditional exchangeability formulation above.

\subsection{A Tighter Population-Level Miscoverage Bound}
\label{app:tighter_bound}

In the main text, we present a conservative marginal guarantee
\begin{equation}
\Pr\!\left(R_{N+1}(\hat{\lambda})=1\right)
\le
\alpha+\mathbb{E}\!\left[\widetilde{\beta}(M)\right],
\label{eq:app_loose_bound}
\end{equation}
where
\[
\widetilde{\beta}(M)
=
\frac{N}{N+1}\widehat{R}_N(M)+\frac{1}{N+1}
\]
is the exchangeability-corrected empirical proxy introduced in Section~\ref{sec:bound_sampling_failure}. 
This bound is simple and sufficient for our main result, but it is conservative because it upper-bounds the second-stage contribution by $\alpha$ ($\Pr\!\left(
Z_{N+1}(M)=0\right)\le 1$) without exploiting the fact that selection failure can only occur when finite sampling has already succeeded. 

To sharpen this characterization, let
\begin{equation}
\begin{split}
    &p_{\mathrm{fail}}(M):=\Pr\!\left(Z_{N+1}(M)=1\right),\\
&p_{\mathrm{sel}}:=\Pr\!\left(R_{N+1}(\hat{\lambda})=1 \mid Z_{N+1}(M)=0\right).
\end{split}
\label{eq:app_pfail_psel}
\end{equation}
Then the decomposition in Section~\ref{sec:overall_guarantee_main} can be written as
\begin{equation}
\begin{split}
\Pr\!\left(R_{N+1}(\hat{\lambda})=1\right)
&=
p_{\mathrm{fail}}(M)\\&+p_{\mathrm{sel}}\bigl(1-p_{\mathrm{fail}}(M)\bigr).
\end{split}
\label{eq:app_exact_decomp}
\end{equation}
This formulation makes explicit that the second-stage risk is incurred only on the subset of test data for which the first-stage finite-sampling step has already produced at least one admissible answer.

\begin{proposition}[A tighter population-level miscoverage bound]
\label{prop:app_tighter_bound}
Under the assumptions of Sections~\ref{sec:bound_sampling_failure}--\ref{sec:overall_guarantee_main},
\begin{equation}
\begin{split}
    \Pr\!\left(R_{N+1}(\hat{\lambda})=1\right)
&\le
\alpha+(1-\alpha)p_{\mathrm{fail}}(M)\\
&\le
\alpha+(1-\alpha)\mathbb{E}\!\left[\widetilde{\beta}(M)\right].
\end{split}
\label{eq:app_tight_bound}
\end{equation}
\end{proposition}

\begin{proof}
Starting from Eq.~\eqref{eq:app_exact_decomp}, we have
\[
\Pr\!\left(R_{N+1}(\hat{\lambda})\!=\!1\right)
\!=\!
p_{\mathrm{fail}}(M)\!+\!p_{\mathrm{sel}}\bigl(1\!-\!p_{\mathrm{fail}}(M)\bigr).
\]
By Eq.~(23) in the main text, $p_{\mathrm{sel}}\le \alpha$. Therefore,
\[
\begin{split}
    \Pr\!\left(R_{N+1}(\hat{\lambda})\!=\!1\right)
&\!\le\!
p_{\mathrm{fail}}(M)\!+\!\alpha\bigl(1-p_{\mathrm{fail}}(M)\bigr)\\
&\!=\!
\alpha\!+\!(1\!-\!\alpha)p_{\mathrm{fail}}(M).
\end{split}
\]
Finally, Proposition~1 in the main text gives
\[
p_{\mathrm{fail}}(M)\le \mathbb{E}\!\left[\widetilde{\beta}(M)\right],
\]
which yields the second inequality in Eq.~\eqref{eq:app_tight_bound}.
\end{proof}

Compared with the conservative upper bound in Eq.~\eqref{eq:app_loose_bound}, the improvement is
\[
\begin{split}
  &  \bigl(\alpha+\mathbb{E}[\widetilde{\beta}(M)]\bigr)
-
\bigl(\alpha+(1-\alpha)\mathbb{E}[\widetilde{\beta}(M)]\bigr)
\\&\qquad \qquad \qquad =
\alpha\,\mathbb{E}[\widetilde{\beta}(M)].
\end{split}
\]
Hence, the gap is exactly the overlap that is double-counted by the loose upper bound. 
Importantly, this tightening is obtained by coupling the two error terms through the exact decomposition in Eq.~\eqref{eq:app_exact_decomp}; it does \emph{not} follow from independently replacing $\Pr(Z_{N+1}(M)=0)$ by $1-\mathbb{E}[\widetilde{\beta}(M)]$.

As an immediate corollary, we obtain a tighter population-level coverage guarantee
\begin{equation}
\begin{split}
    &\quad \Pr\!\left(
\exists \hat y\in\mathcal{C}_{\hat{\lambda}}(x_{N+1}), 
A(y_{N+1}^*,\hat y)=1
\right)\\
&\ge
1-\alpha-(1-\alpha)\mathbb{E}\!\left[\widetilde{\beta}(M)\right].
\end{split}
\label{eq:app_tight_coverage}
\end{equation}
Eq.~\eqref{eq:app_tight_bound} is viewed as a tighter population-level refinement of the main-text guarantee, while Eq.~\eqref{eq:app_loose_bound} remains the simpler and more transparent statement in the main development.

\section{Additional Related Work}
\label{sec: Additional Related Work}

\noindent \textbf{\textit{Necessity of Calibration Set Integrity.}} 
Previous SCP methods for open-ended QA typically discard calibration data points whose sampling sets fail to cover admissible candidates~\citep{wang-etal-2024-conu,wang2025sample,kaur2024addressing}. 
Under the i.i.d. assumption, however, this pruning contracts the empirical uncertainty distribution and restricts conformal validity to a smaller subset of test distribution~\citep{wang2025sconu}. 
Moreover, both TRON~\citep{wang2025sample} and SConU~\citep{wang2025sconu} suggest, through the calibration of sampling budget, that if every calibration instance contains admissible answers within the sampling set, then test points are guaranteed to produce admissible candidates via finite sampling. 
Such a premise conflicts with the model-agnostic nature of SCP~\citep{angelopoulos2023conformal}: for any deployed generative model, it is generally impossible to ensure that every input query is answerable or that an admissible answer lies within the model's finite-sample output space.

\noindent \textbf{\textit{Uncertainty Estimation for LLM Generation.}} 
Uncertainty estimation for free-form generation has been widely investigated as a way to detect hallucinations and unreliable model responses. 
Early approaches often exploit sampling-based consistency: if multiple stochastic generations are semantically consistent, the model is considered more reliable; otherwise, disagreement among generations can indicate uncertainty~\citep{manakul2023selfcheckgpt}. 
Because natural language admits many surface forms with the same meaning, later work proposes semantic-level uncertainty measures, such as semantic entropy~\citep{kuhn2023semantic}, relevance-aware token or sentence weighting~\citep{duan2024shifting}, and word-sequence-level uncertainty~\citep{wang2025word}, to better capture uncertainty in free-form answers. 
These methods provide useful uncertainty scores for open-ended QA, but they are typically heuristic and do not by themselves provide finite-sample risk guarantees. 
MiRD is complementary to this line of work: it can use such uncertainty measures as candidate scores, while calibrating them into risk-controlled set-valued predictions through the proposed miscoverage decomposition.

\noindent \textbf{\textit{Risk-Controlling Prediction and Conformal Generation.}} 
Beyond SCP, risk-controlling prediction frameworks calibrate prediction sets to satisfy user-specified loss constraints under distribution-free or finite-sample guarantees~\citep{bates2021distribution,jin2023selection,angelopoulos2025learn}. 
Recent work has also adapted conformal ideas to language generation, including conformal language modeling, conformal factuality guarantees, and surveys of SCP for NLP~\citep{quach2024conformal,mohri2024language,campos2024conformal,li2026set}. 
These studies indicate the promise of conformal methods for reliable generation, but many of them either focus on controlling the final generated output, calibrating a sampling or rejection rule, or constructing guarantees without explicitly separating the risk caused by finite sampling from the risk caused by post-hoc filtering. 
MiRD differs by decomposing overall miscoverage into marginal sampling failure and conditional selection failure, which makes the role of finite sampling explicit and enables full-calibration conformalized selection.

\paragraph{Selective Prediction and Abstention.}
Selective prediction studies the problem of abstaining from uncertain predictions to trade coverage for reliability~\citep{geifman2017selective,geifman2019selectivenet}. 
This idea has also influenced LLM reliability, where systems may abstain, defer, or filter outputs when uncertainty is high~\citep{jia2026balancerag,wang2025coin,wang2025lec}. 
However, selective prediction usually focuses on deciding whether to output a prediction, while MiRD focuses on producing a set-valued answer for open-ended QA. 
Moreover, MiRD distinguishes two different reasons why a set-valued predictor may fail: the admissible answer may never appear in the sampled candidate pool, or it may appear but be removed by the selection rule. 
This distinction is essential in open-ended generation, where the candidate space is not fixed and must be approximated by finite stochastic sampling.

\section{Details of Experimental Setup}
\label{app:Details of Experimental Setup}

\noindent \textbf{Details of utilized Datasets.} 
TriviaQA~\citep{joshi2017triviaqa} is a reading comprehension benchmark containing over 650K samples, which are authored by trivia enthusiasts and independently gathered evidence documents. 
CoQA~\citep{reddy2019coqa} is a large-scale dataset for building Conversational QA systems, with 127k questions with free-form answers, and each question is equipped with contextual information. 
Following \citet{lin2024generating}, we use the validation split of the \texttt{rc.nocontext} subset of TriviaQA, the development split of CoQA. 
We employ 8,000 question-answer pairs of TriviaQA (total 9,960 de-duplicated questions), and the full 7,983 questions of CoQA. 
For each sample, we adopt a one-shot prompt, as illustrated as follows:

\begin{tcolorbox}[
    colback=lightgraybg
]
\small{\texttt{$\#\#\#$ System:}\\
\texttt{This is a bot that correctly answers questions.}
\\

\texttt{$\#\#\#$ User:}\\
\texttt{In 1968, who did radical feminist Valerie Solanas shoot and wound as he entered his New York studio?}\\
\texttt{$\#\#\#$ Assistant:}\\
\texttt{Andy Warhol}\\

\texttt{$\#\#\#$ User:}\\
\texttt{Who was the man behind The Chipmunks?}\\
\texttt{$\#\#\#$ Assistant:\\}}
\end{tcolorbox}

\begin{tcolorbox}[
    colback=lightgraybg
]
\small{\texttt{$\#\#\#$ System:}\\
\texttt{This is a bot that correctly answers questions.}
\\
\texttt{Once upon a time, in a barn near a farm house, there lived a little white kitten named Cotton. Cotton lived high up in a nice warm place above the barn ...}

\texttt{$\#\#\#$ User:}\\
\texttt{What color was Cotton?}\\
\texttt{$\#\#\#$ Assistant:}\\
\texttt{white}\\

\texttt{$\#\#\#$ User:}\\
\texttt{Where did she live?}\\
\texttt{$\#\#\#$ Assistant:}\\}
\end{tcolorbox}

NQ~\citep{kwiatkowski2019natural} is a more realistic and challenging closed-book dataset, containing questions from real users, and it requires QA systems to read and comprehend an entire Wikipedia article that may or may not contain the answer to the question. 
We utilize 2,000 QA samples from the validation split with a total of 3,610 questions. 
The prompt adopted in NQ is illustrated as follows:

\begin{tcolorbox}[colback=lightgraybg]
\small{
\texttt{$\#\#\#$ System:}\\
\texttt{This is a bot that correctly answers questions.}
\\
\texttt{$\#\#\#$ User:}\\
\texttt{who breaks a tie in the us senate?}\\
\texttt{$\#\#\#$ Assistant:}\\
\texttt{The Vice President of the United States}\\

\texttt{$\#\#\#$ User:}\\
\texttt{which domain of life are humans members of?}\\
\texttt{$\#\#\#$ Assistant:}\\
}
\end{tcolorbox}

\noindent \textbf{Details of Correctness Criteria.} 
Following previous research~\citep{duan2024shifting,wang-etal-2024-conu,wang2025word,wang2026safer},  similarity is measured by the stsb-roberta-large~\citep{reimers2019sentence} model. 
We set the threshold as 0.6 by default, i.e., candidate answers having above 0.6 semantic similarities with the ground truth are deemed correct. 
For the bi-entailment criterion, we leverage a natural language inference (NLI) model with DeBERTa-v3-large-mnli-fever-anli-ling-wanli~\citep{he2023debertav} as the backbone. 
Here, given a QA pair, the NLI classifier predicts scores (logits) for three distinct classes: entailment, neutral, and contradiction. 
See Section 4.1 in \citet{lin2024generating} for more details. 
If and only if neither direction is labeled as contradiction and one direction is entailment, we consider the two QA instances to be equivalent. 

\noindent \textbf{Hyperparameters.} 
We fix the calibration-test split ratio at 0.5 by default. 
Following prior work~\citep{duan2024shifting,wang-etal-2024-conu,wang2025lec}, we use multinominal sampling for candidates with the generation \texttt{temperature} of 1.0 and \texttt{top-p} of 0.9.

\section{Additional Experimental Results}
\label{app:Additional Experimental Results}

\paragraph{Robustness of conditional selection risk across sampling budgets.}
Figure~\ref{fig:conditional_selection_budget} further evaluates the conditional selection risk of MiRD under different sampling budgets on TriviaQA, using Vicuna-7B-Instruct as the backbone. 
Across all sampling budgets from $M=4$ to $M=20$ and all target risk levels, MiRD consistently keeps the empirical conditional selection risk below the target level $\alpha$. 
This confirms that the Stage-II calibration remains valid not only at the default sampling budget used in the main experiments, but also across a broad range of finite-sampling regimes. 
As $\alpha$ increases, the empirical risk also increases smoothly, reflecting the expected coverage--set-size trade-off induced by conformal thresholding.

Compared with ConU, MiRD generally yields lower or comparable conditional selection risk across sampling budgets. 
This is consistent with our theoretical interpretation: by retaining sampling-failed calibration examples rather than discarding them, MiRD calibrates a weakly more conservative threshold over the full calibration distribution. 
The results show that this conservativeness is stable with respect to the sampling budget, suggesting that the benefit of full-calibration thresholding is not tied to a particular choice of $M$.

\begin{figure*}[!t]
    \centering
    \includegraphics[width=\linewidth]{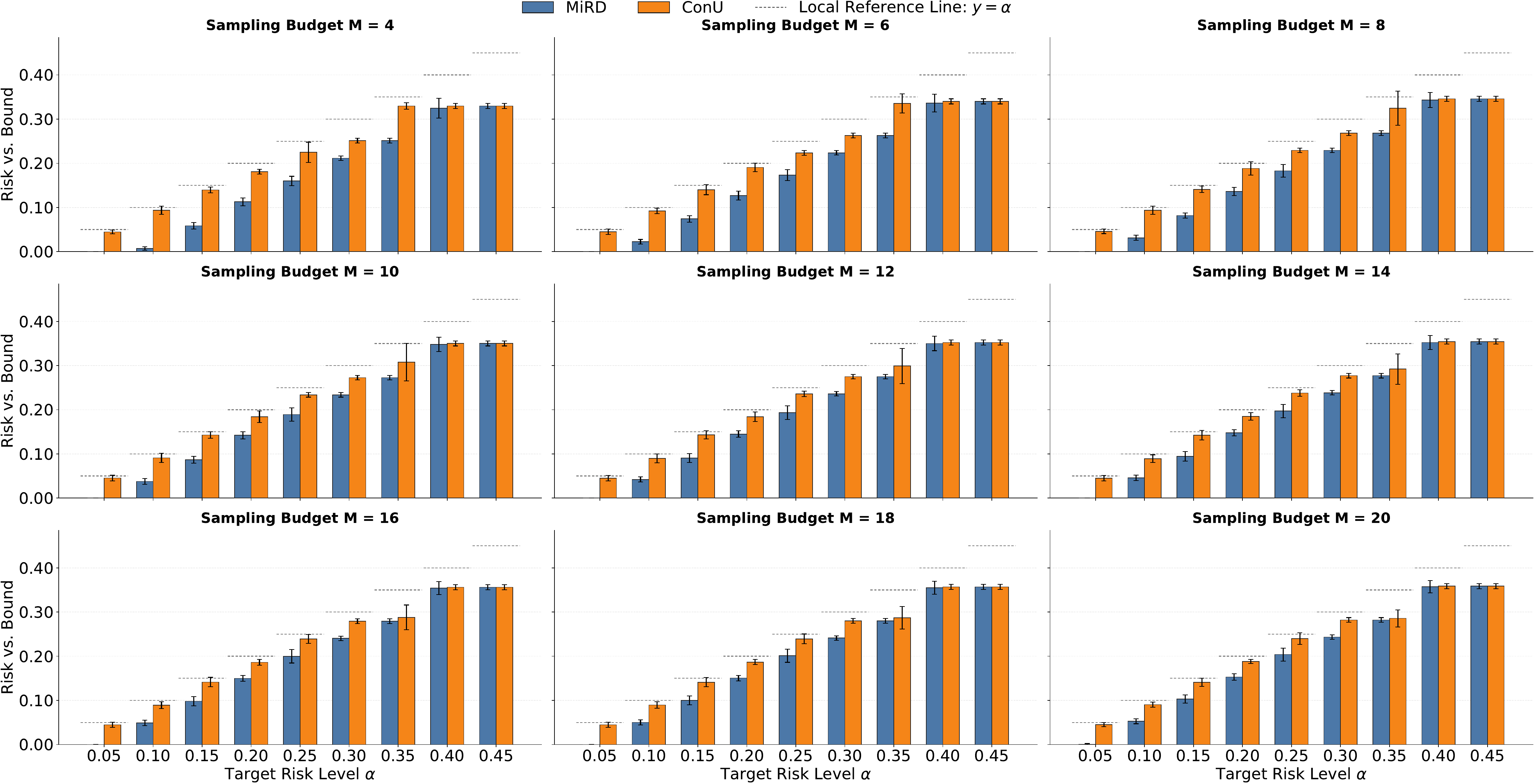}
    \caption{Conditional Selection Risk of MiRD vs. ConU across various sampling budgets and risk levels, using Vicuna-7B-Instruct on TriviaQA. Each subplot corresponds to one sampling budget. Within each $\alpha$ group, the dashed segment marks the target level $\alpha$; bars lower than the dashed segment indicate the mean risk is below $\alpha$.}
     \vspace{-2mm}\label{fig:conditional_selection_budget}
\end{figure*}

\begin{figure*}[!t]
  \centering
  \begin{subfigure}[b]{0.245\textwidth}
    \centering
    \includegraphics[width=\textwidth]{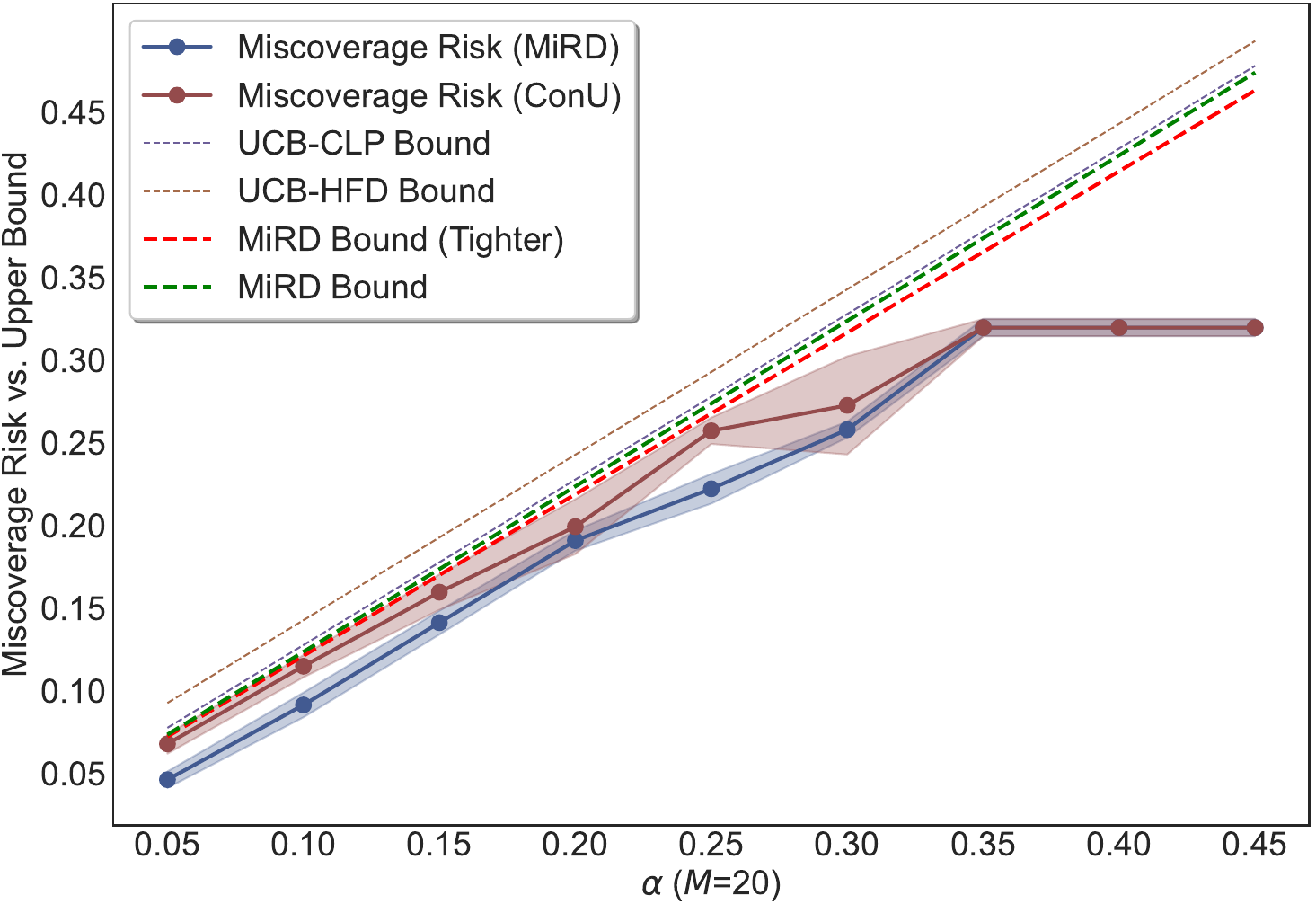}
    \caption{LLaMA-3.1-8B-Instruct.}
  \end{subfigure}
  \hfill
  \begin{subfigure}[b]{0.245\textwidth}
    \centering
    \includegraphics[width=\textwidth]{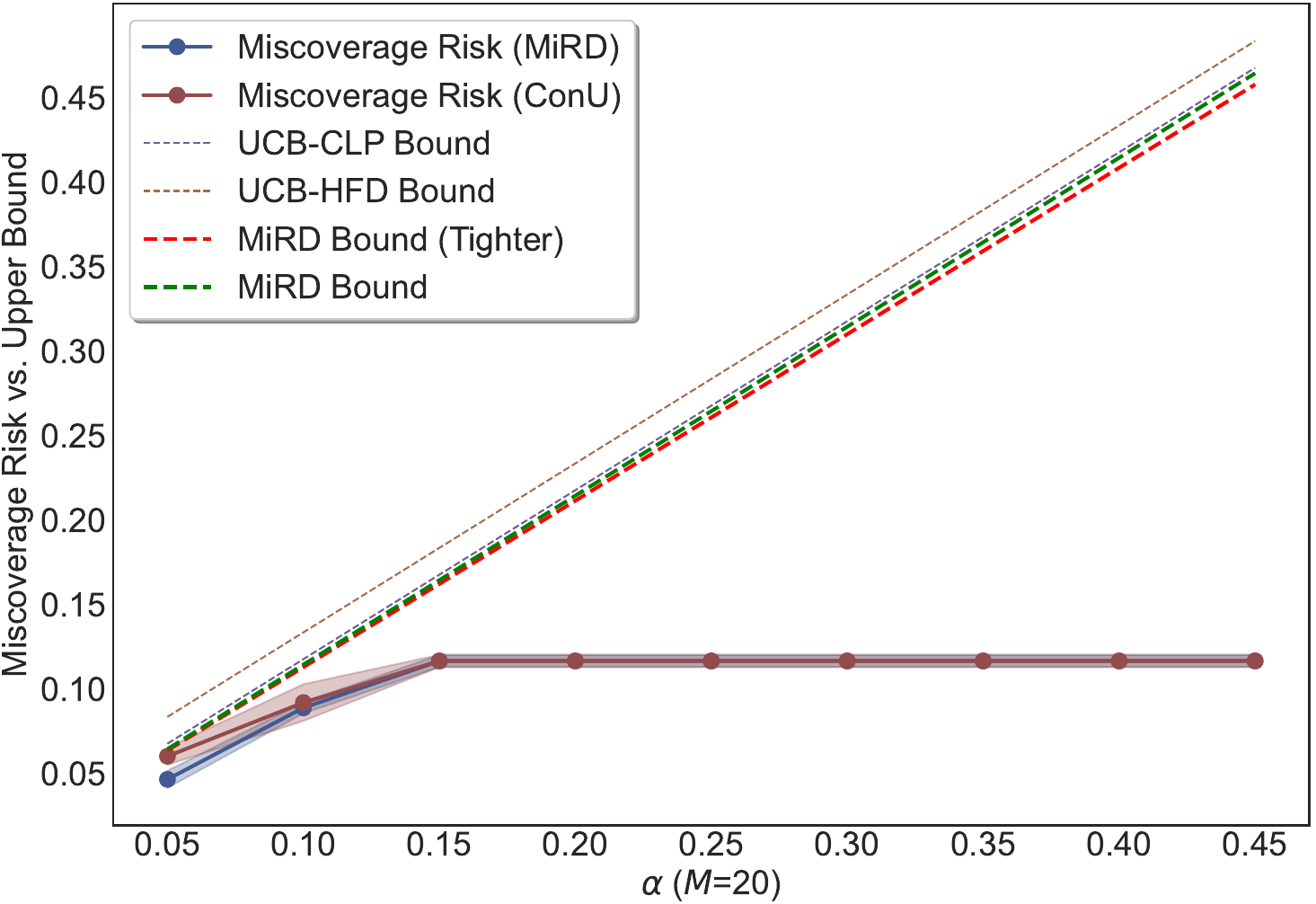}
    \caption{LLaMA-3.1-70B-Instruct.}
  \end{subfigure}
  \hfill
  \begin{subfigure}[b]{0.245\textwidth}
    \centering
    \includegraphics[width=\textwidth]{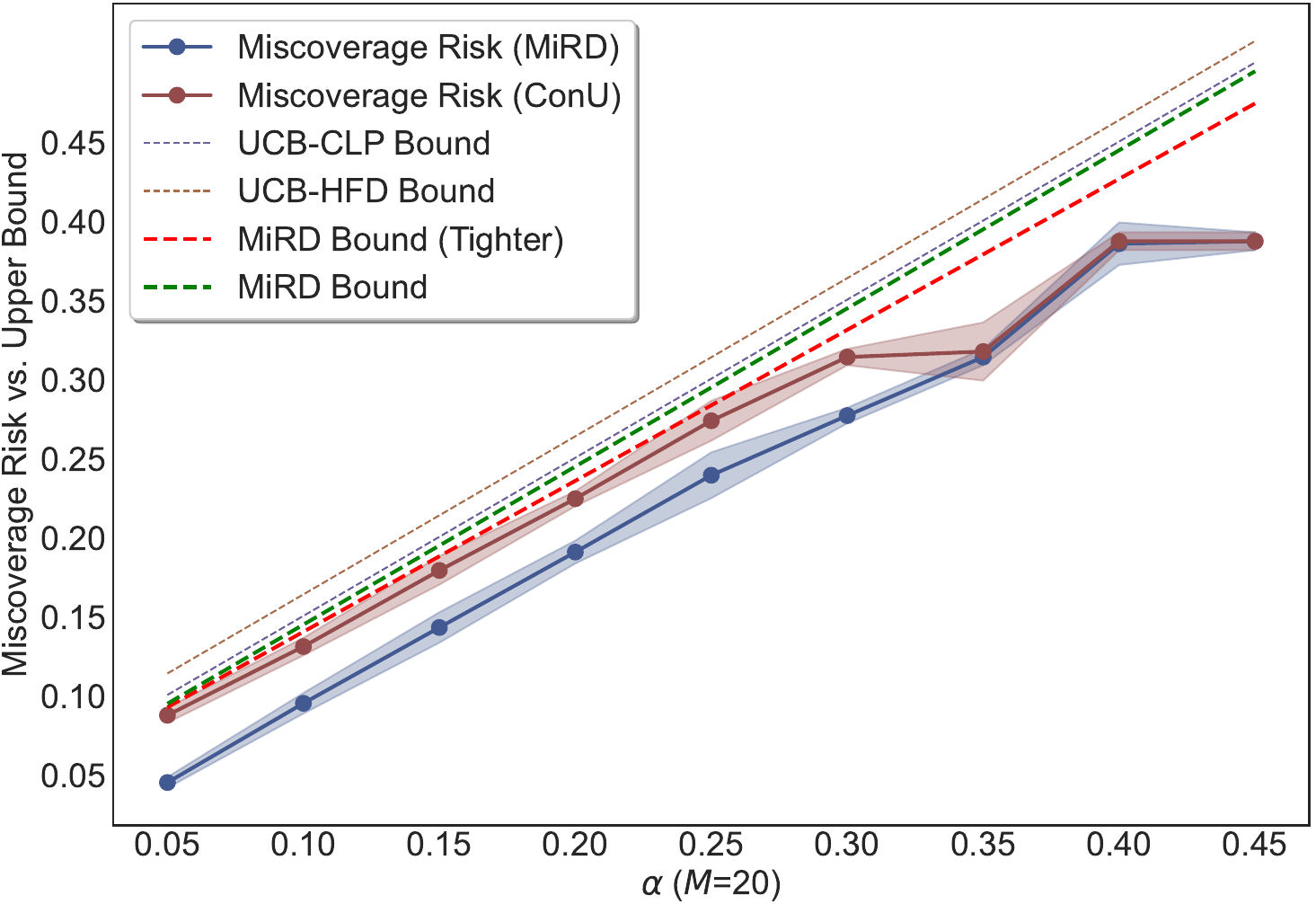}
    \caption{Vicuna-7B-Instruct.}
  \end{subfigure}
  \hfill
  \begin{subfigure}[b]{0.245\textwidth}
    \centering
    \includegraphics[width=\textwidth]{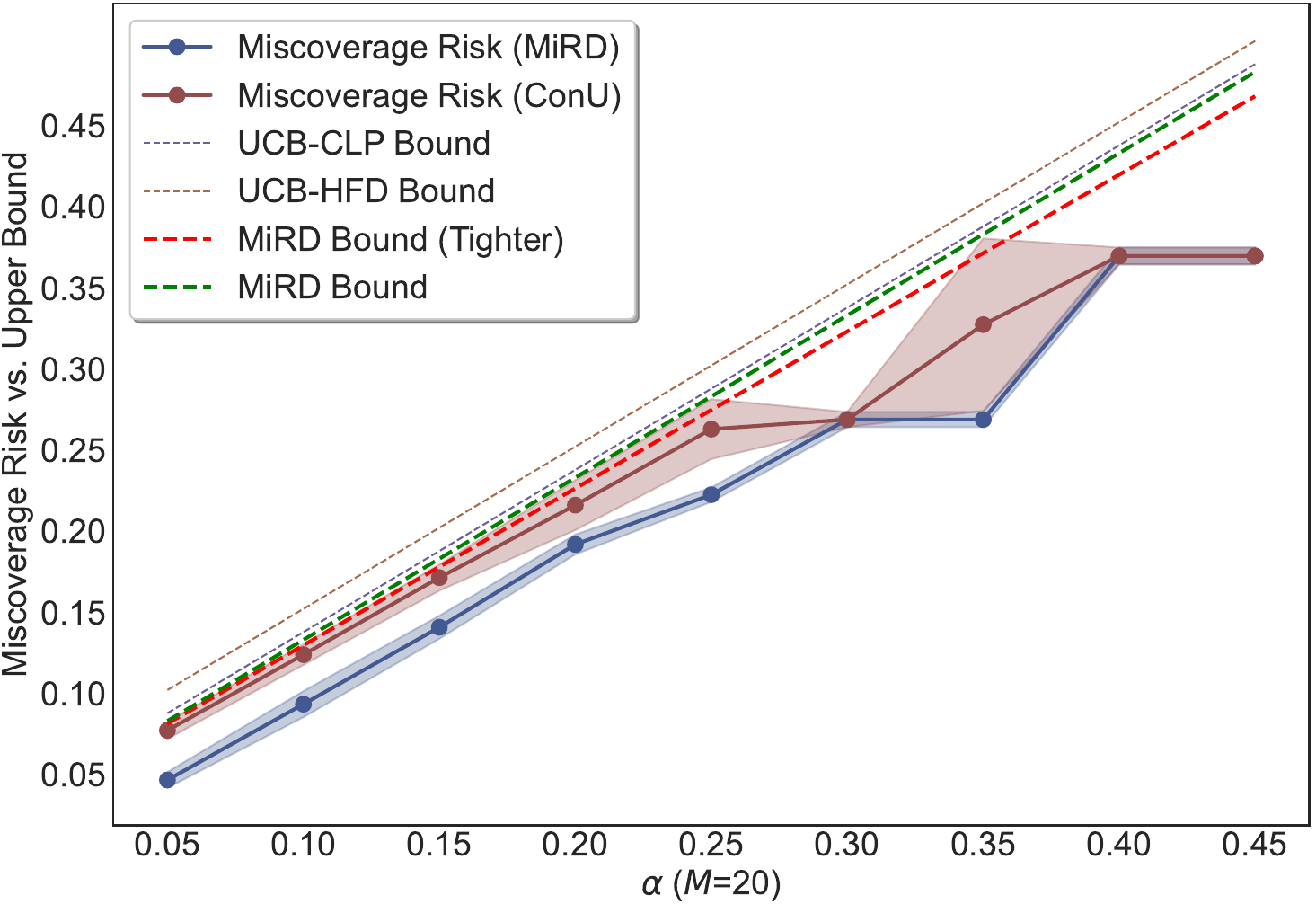}
    \caption{Vicuna-13B-Instruct.}
  \end{subfigure}

  \begin{subfigure}[b]{0.245\textwidth}
    \centering
            \includegraphics[width=\textwidth]{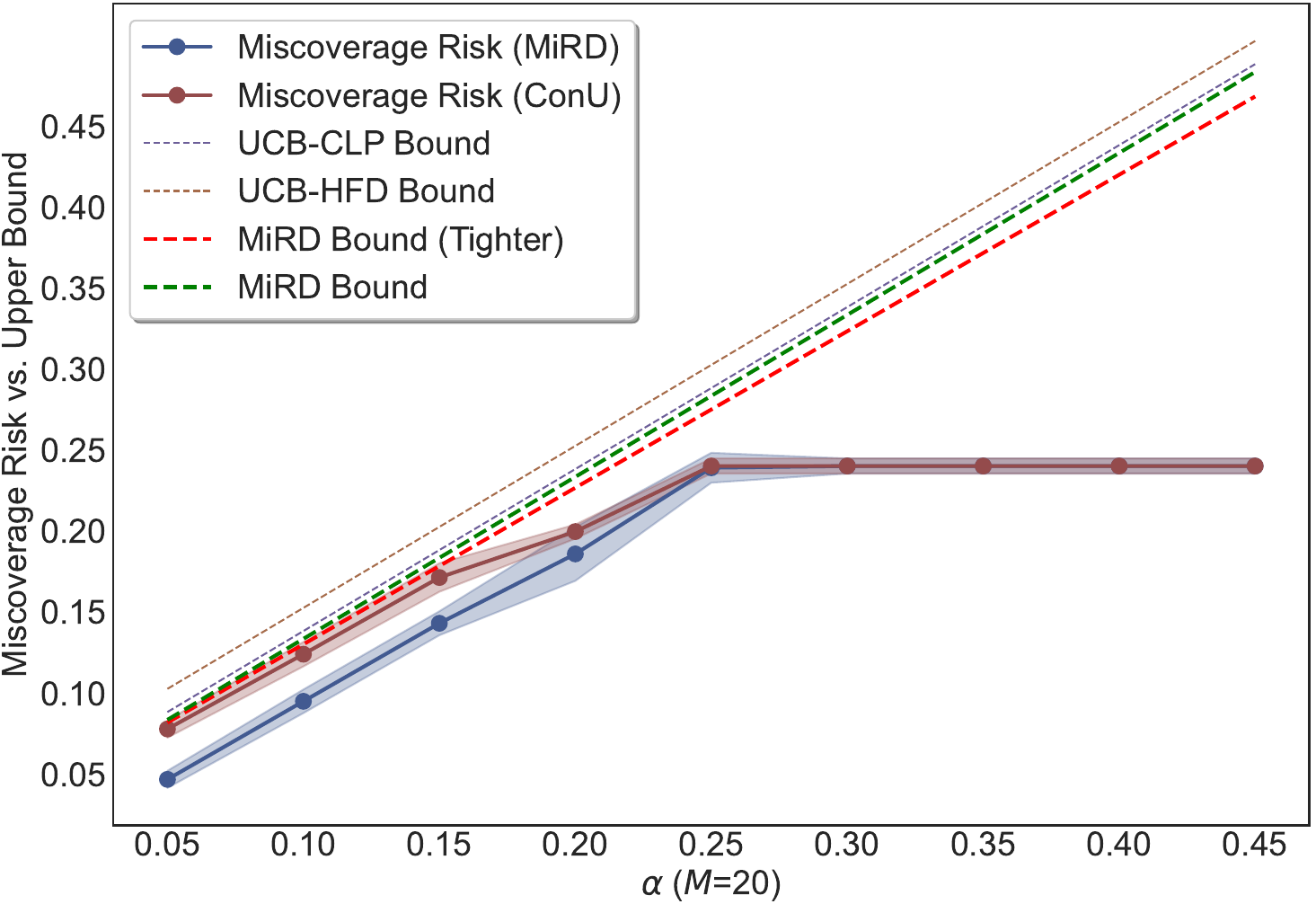}
    \caption{OpenChat-3.5 (7B).}
  \end{subfigure}
  \hfill
  \begin{subfigure}[b]{0.245\textwidth}
    \centering
    \includegraphics[width=\textwidth]{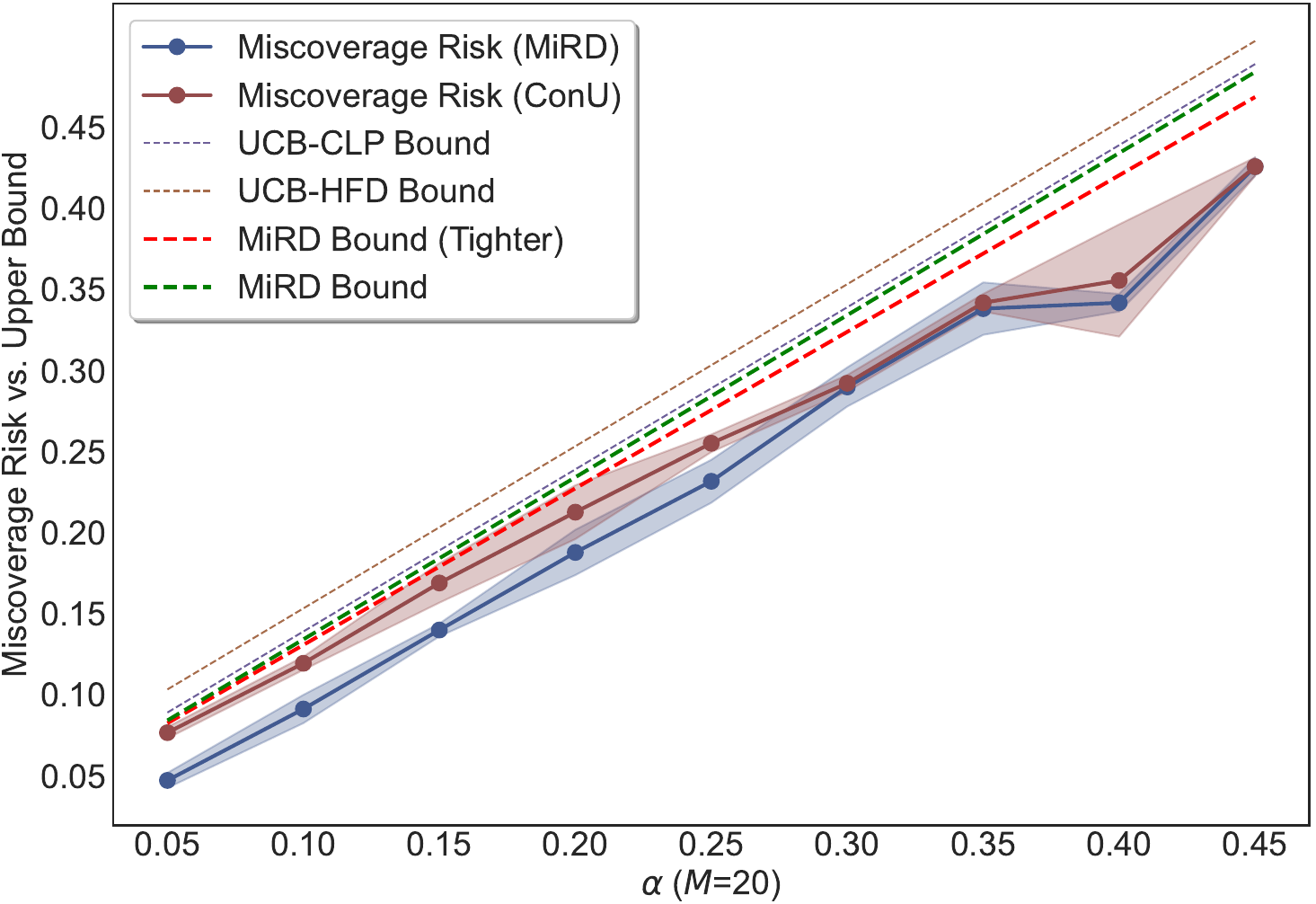}
    \caption{Qwen2.5-3B-Instruct.}
  \end{subfigure}
  \hfill
  \begin{subfigure}[b]{0.245\textwidth}
    \centering
    \includegraphics[width=\textwidth]{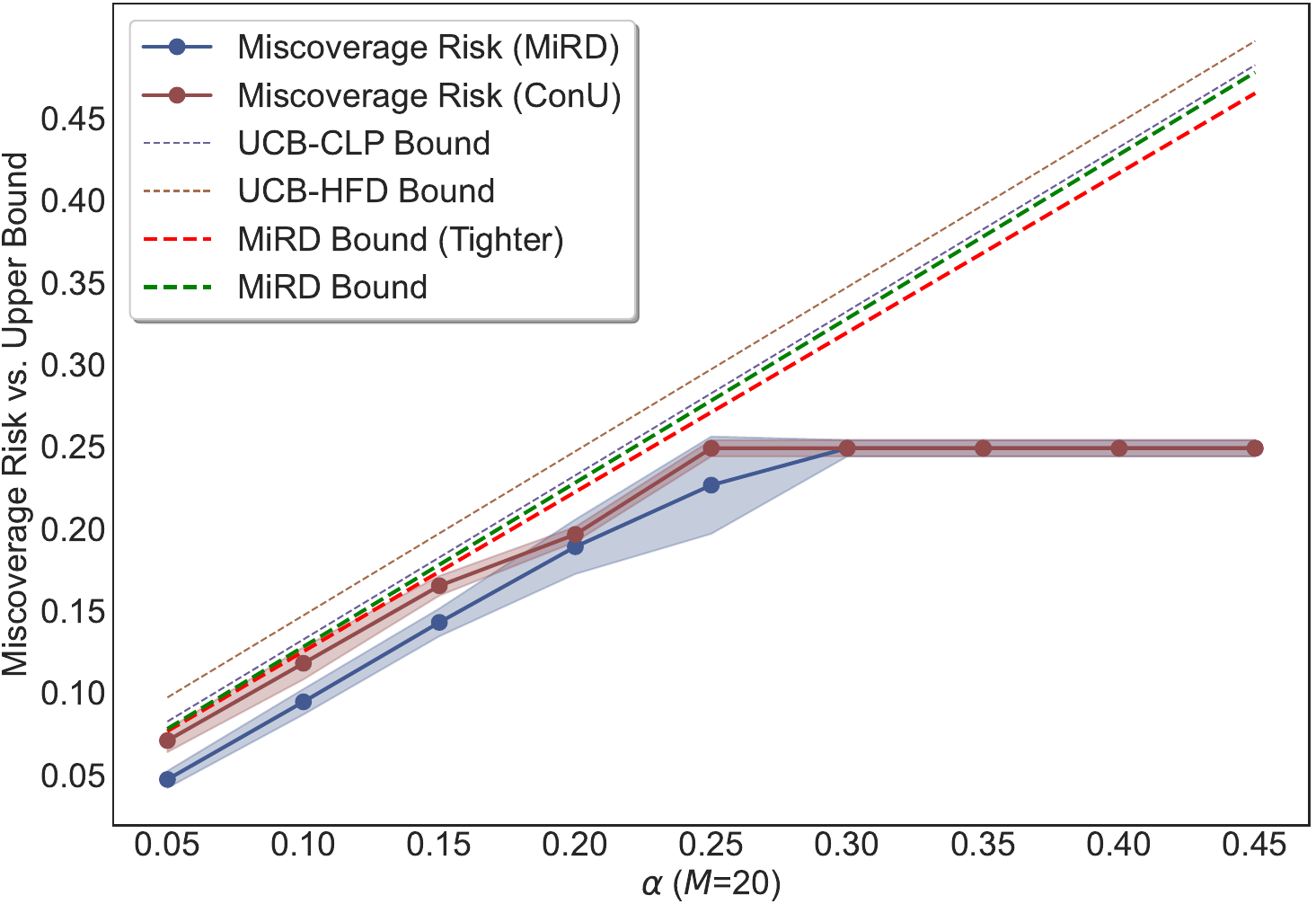}
    \caption{Qwen2.5-7B-Instruct.}
  \end{subfigure}
  \hfill
  \begin{subfigure}[b]{0.245\textwidth}
    \centering
    \includegraphics[width=\textwidth]{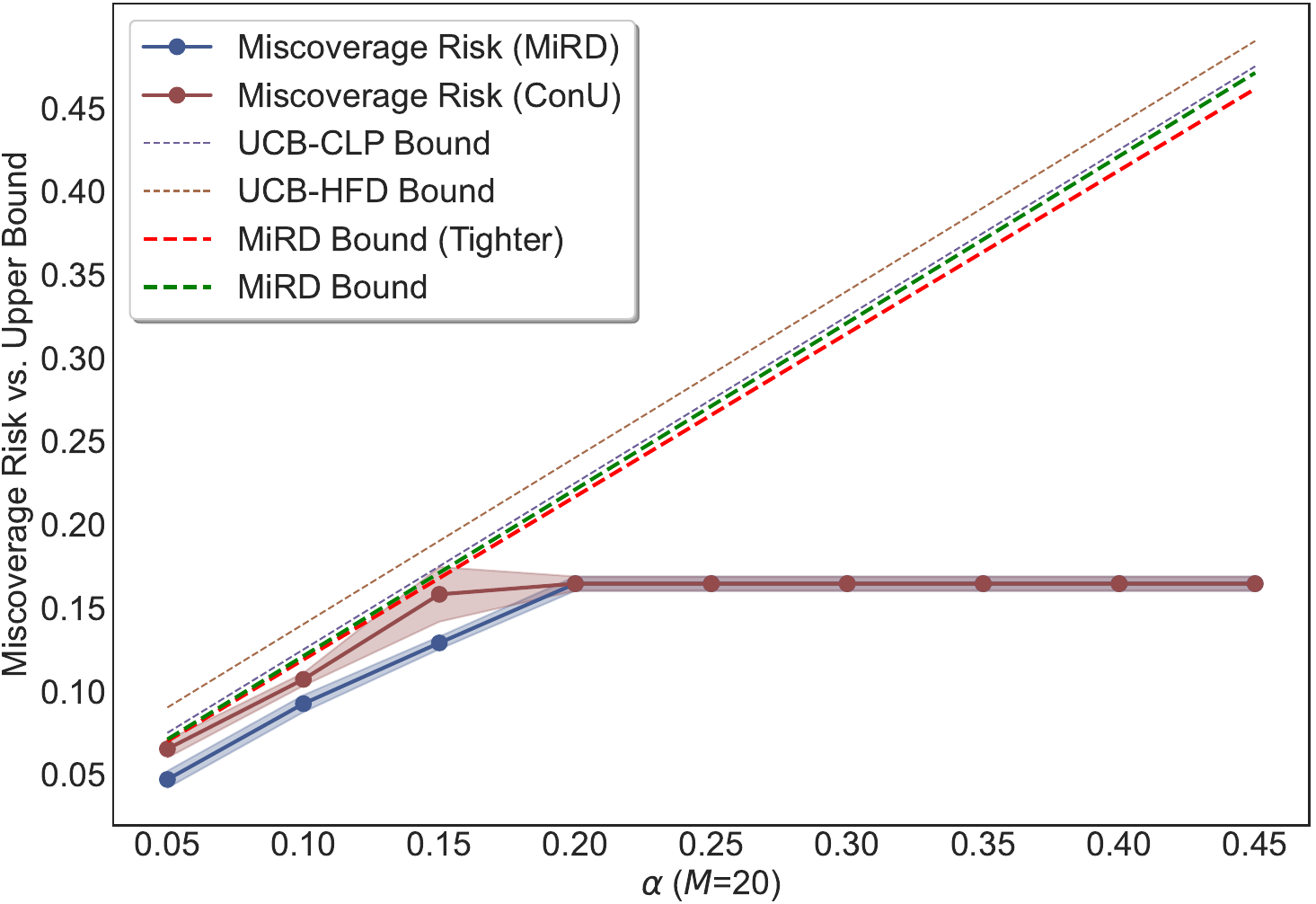}
    \caption{Qwen2.5-14B-Instruct.}
  \end{subfigure}

      \caption{Overall miscoverage risk vs. upper bound at various risk levels ($\alpha$) on TriviaQA with eight LLMs. Here, we set the sampling budget to 20.}
      \vspace{-4mm}
  \label{fig: miscoverage risk control triviaqa (similairty-0.6, M=20).}
\end{figure*}

\paragraph{Overall miscoverage control under a larger sampling budget.}
Figure~\ref{fig: miscoverage risk control triviaqa (similairty-0.6, M=20).} reports the end-to-end overall miscoverage risk on TriviaQA when the sampling budget is increased to $M=20$. 
Across all eight LLMs, the empirical overall miscoverage of MiRD remains below the theoretical upper bounds over the full range of target risk levels. 
Moreover, the tighter MiRD bound tracks the empirical miscoverage more closely than the conservative main-text bound, while both Clopper--Pearson and Hoeffding-style UCB bounds remain substantially looser. 
This further supports the usefulness of the refined decomposition bound in Appendix~\ref{app:tighter_bound}.

The comparison with ConU also shows that MiRD maintains competitive or lower end-to-end miscoverage while explicitly accounting for finite-sampling failure. 
This is important because increasing the sampling budget reduces, but does not conceptually remove, the possibility of sampling failure in open-ended QA. 
Therefore, even under a relatively large sampling budget, explicitly decomposing overall miscoverage into sampling failure and conditional selection failure remains beneficial for reliable set-valued prediction.

\begin{figure*}[!t]
  \centering
  \begin{subfigure}[b]{0.32\textwidth}
    \centering
    \includegraphics[width=\textwidth]{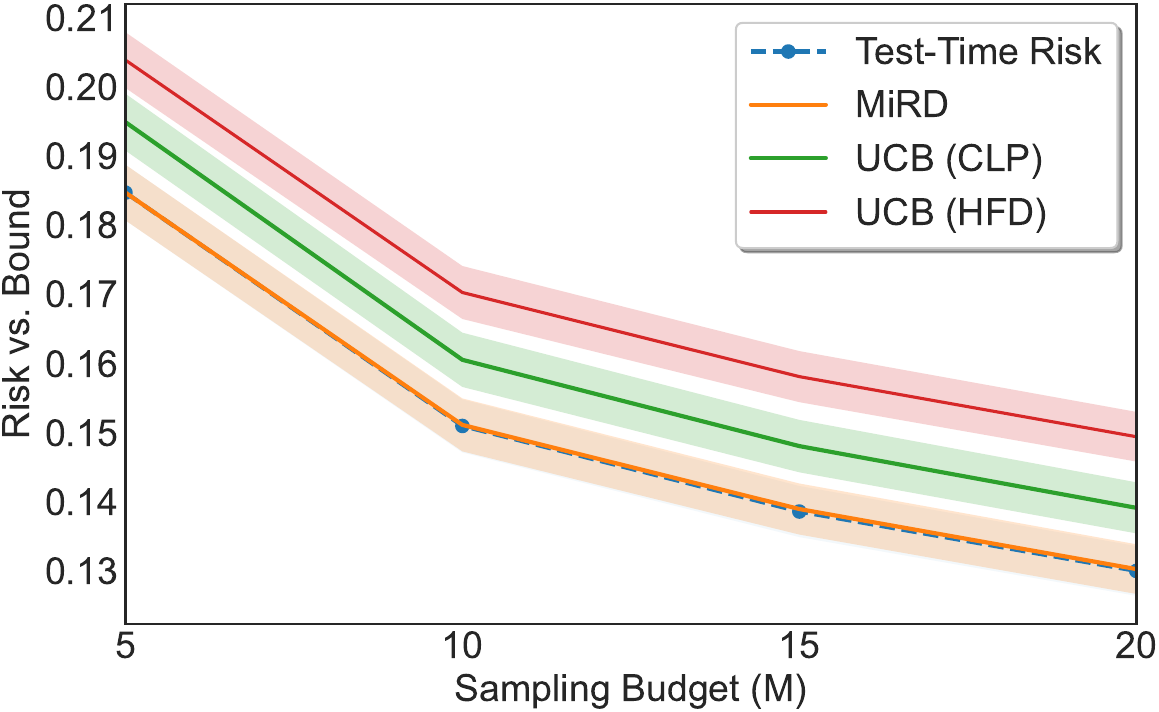}
    \caption{LLaMA-3.1-8B-Instruct.}
  \end{subfigure}
  \hfill
  \begin{subfigure}[b]{0.32\textwidth}
    \centering
    \includegraphics[width=\textwidth]{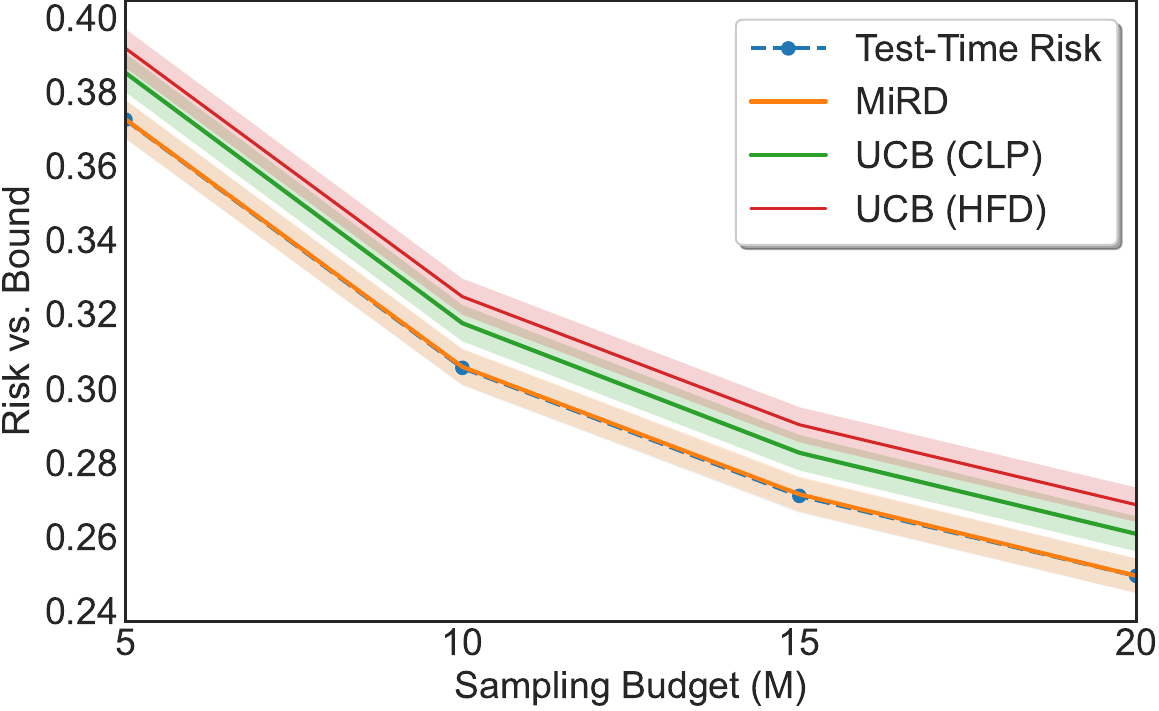}
    \caption{Qwen2.5-3B-Instruct.}
  \end{subfigure}
  \hfill
  \begin{subfigure}[b]{0.32\textwidth}
    \centering
    \includegraphics[width=\textwidth]{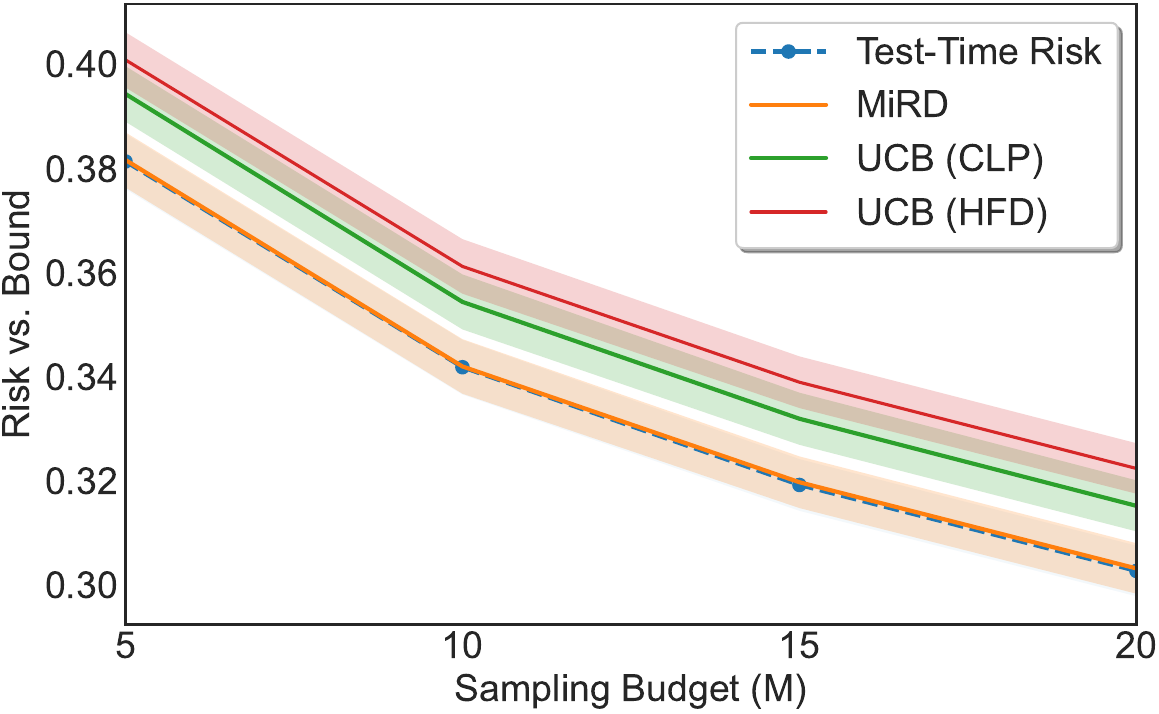}
    \caption{Qwen2.5-14B-Instruct.}
  \end{subfigure}

  \begin{subfigure}[b]{0.32\textwidth}
    \centering
    \includegraphics[width=\textwidth]{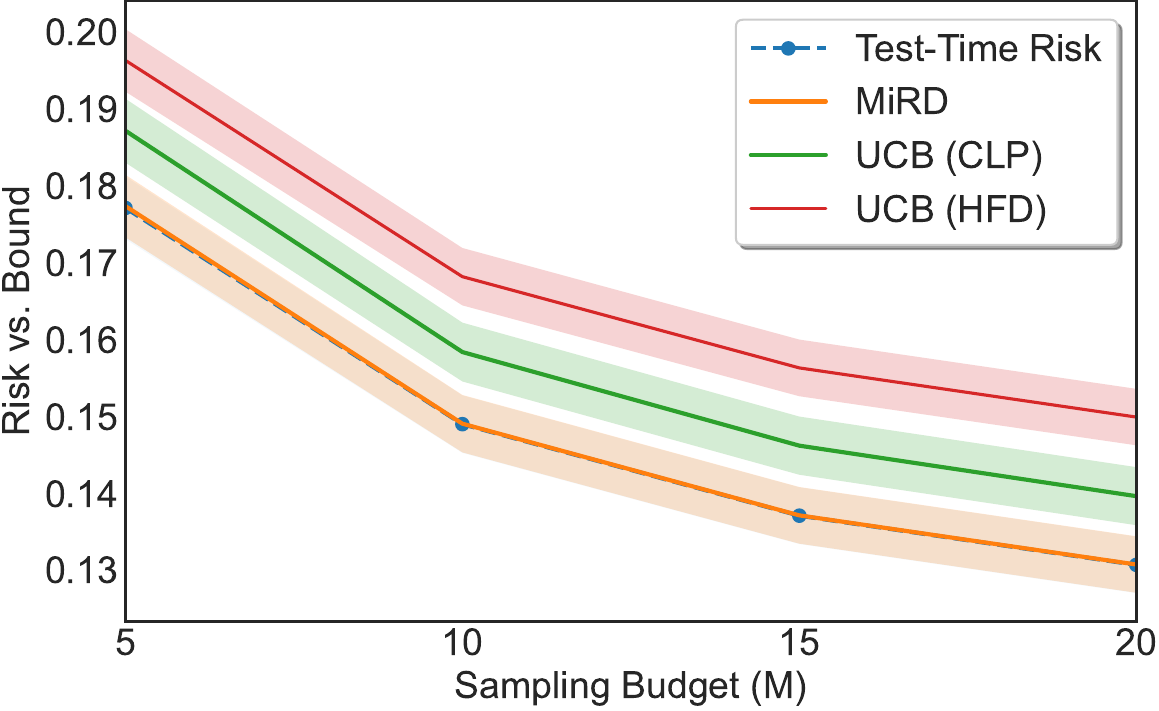}
    \caption{OpenChat-3.5 (7B).}
  \end{subfigure}
  \hfill
  \begin{subfigure}[b]{0.32\textwidth}
    \centering
    \includegraphics[width=\textwidth]{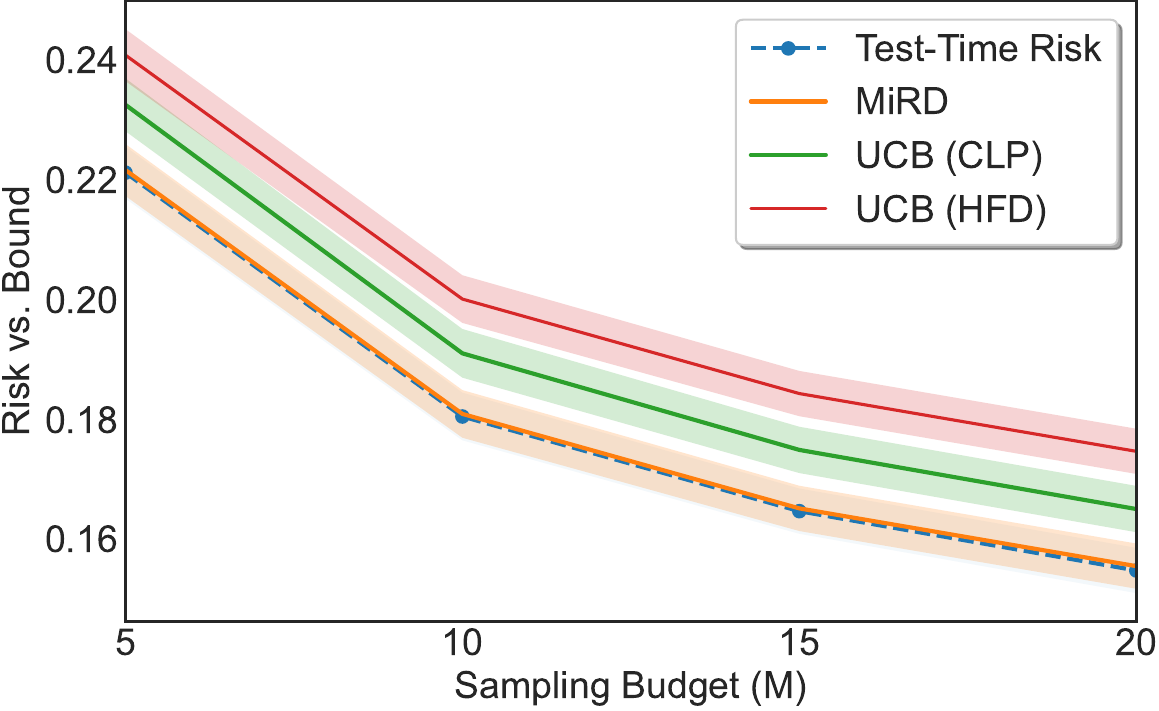}
    \caption{Vicuna-7B-V1.5.}
  \end{subfigure}
  \hfill
  \begin{subfigure}[b]{0.32\textwidth}
    \centering
    \includegraphics[width=\textwidth]{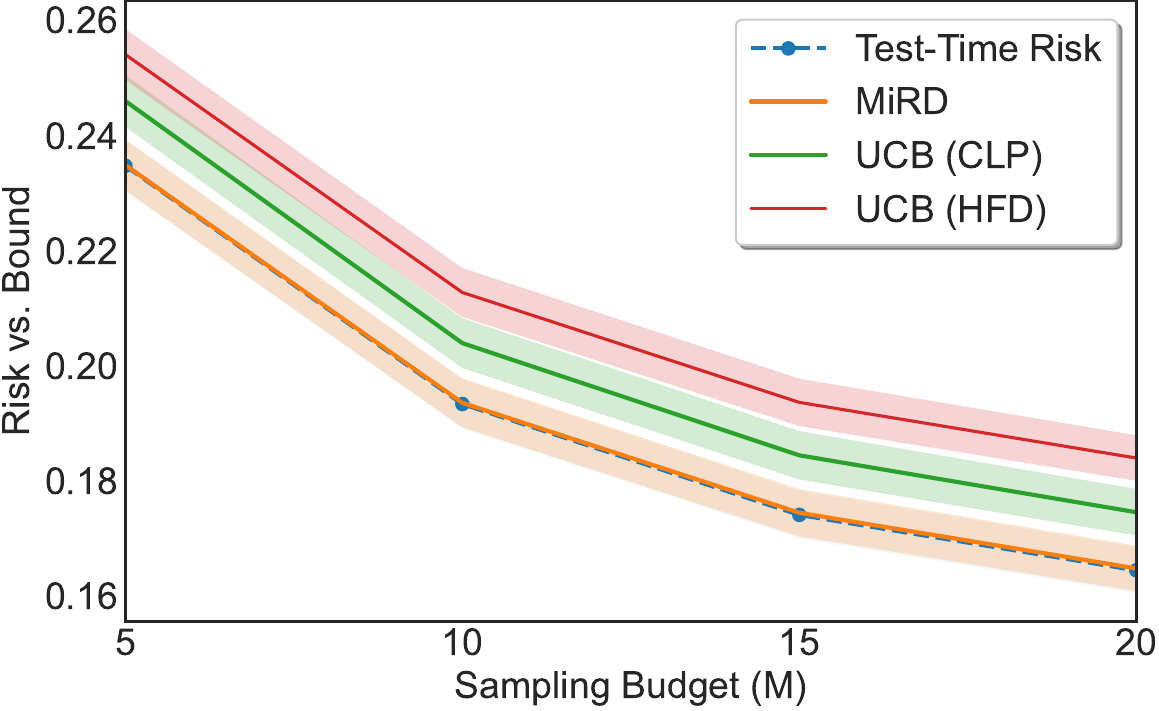}
    \caption{Vicuna-13B-V1.5.}
  \end{subfigure}

      \caption{Sampling risk vs. three upper bounds at various sampling budgets ($M$) on CoQA with six LLMs.}
  \label{fig: sampling risk control coqa (similairty-0.6).}
\end{figure*}

We further evaluate MiRD on CoQA, a context-supported open-ended QA benchmark. 
Compared with TriviaQA, CoQA provides supporting passages and therefore represents a complementary setting where answer generation is conditioned on explicit context rather than relying purely on parametric knowledge. 
The results show that the proposed miscoverage decomposition remains effective in this setting.

\begin{figure}[!t]
    \centering
    \includegraphics[width=\linewidth]{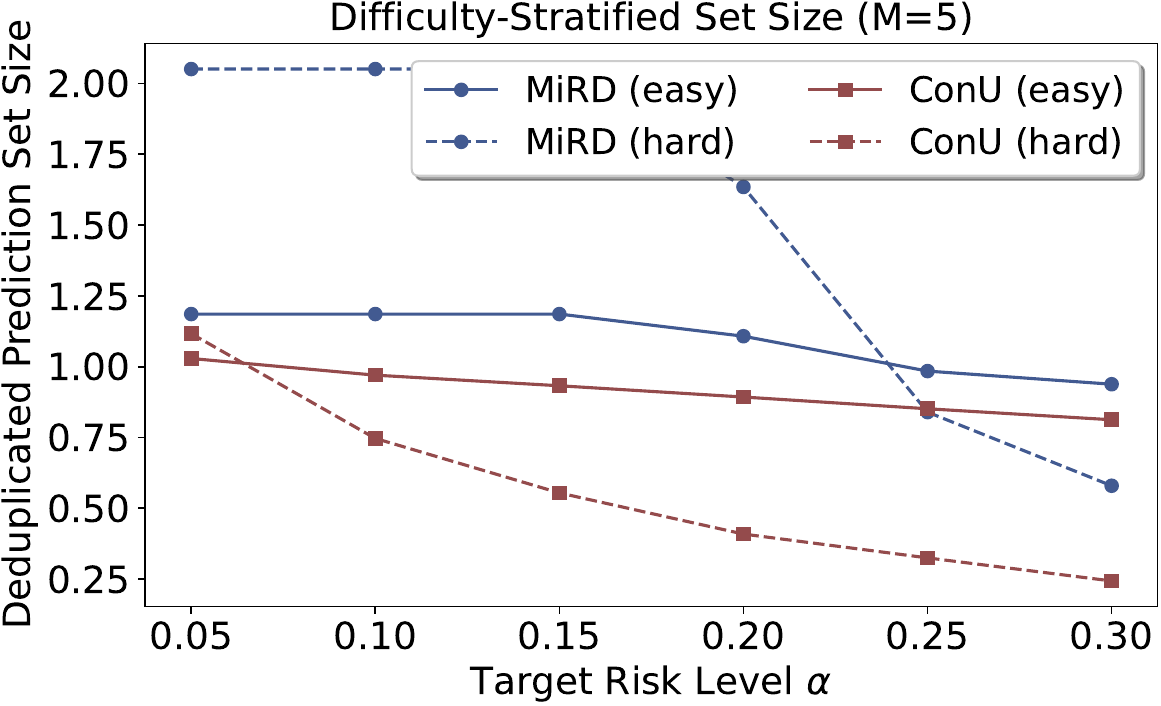}
    \caption{Difficulty-stratified deduplicated prediction set size on CoQA with LLaMA-3.1-8B-Instruct. Test samples are partitioned into easy and hard groups according to whether the top-1 candidate is admissible.}
    \label{fig:difficulty_stratified_set_size_5_coqa}
    \vspace{-4mm}
\end{figure}

\noindent \textbf{Sampling-failure risk control on CoQA.} 
Figure~\ref{fig: sampling risk control coqa (similairty-0.6).} illustrates the sampling-failure risk and its upper bounds on CoQA across six LLMs and multiple sampling budgets. 
Across all models and budgets considered, the MiRD bound consistently stays above the empirical test-time sampling-failure risk, validating the expectation-level marginal guarantee of Stage I. 
At the same time, the MiRD bound is uniformly tighter than both Clopper--Pearson and Hoeffding-style UCB bounds. 
This confirms that the advantage of the proposed expectation-level formulation is not specific to TriviaQA: even in a context-supported QA setting, MiRD provides a non-vacuous and empirically sharp characterization of finite-sampling failure.

\begin{figure}[!t]
    \centering
    \includegraphics[width=\linewidth]{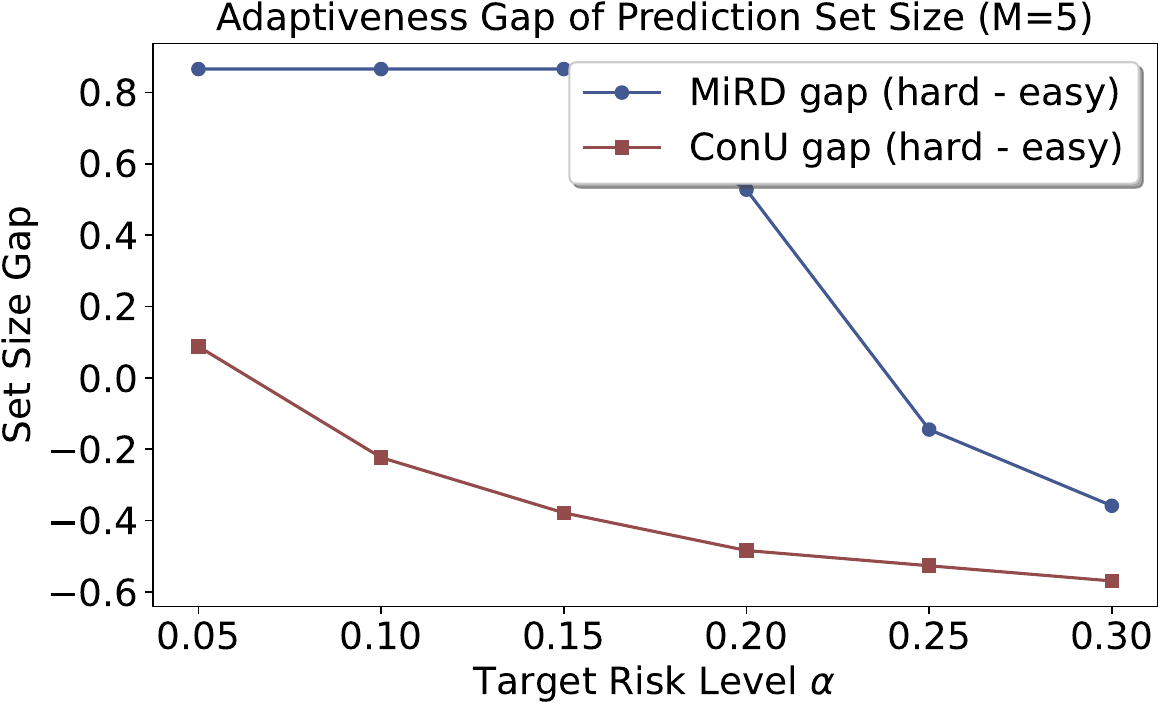}
    \caption{Adaptiveness gap of prediction set size on CoQA with LLaMA-3.1-8B-Instruct, defined as the difference between the average deduplicated prediction set sizes of hard and easy examples.}
    \label{fig:difficulty_stratified_gap_5_coqa}
    \vspace{-4mm}
\end{figure}

As expected, the sampling-failure risk decreases as the sampling budget $M$ increases. 
However, the risk remains non-negligible even at larger budgets for several backbones, showing that finite sampling cannot be ignored simply by increasing $M$. 
This supports the central motivation of MiRD: reliable open-ended set-valued prediction should explicitly account for the possibility that no admissible answer is sampled, rather than treating reliability only as a post-hoc filtering problem.

\noindent \textbf{Prediction-set adaptiveness on CoQA.} 
Figures~\ref{fig:difficulty_stratified_set_size_5_coqa} and~\ref{fig:difficulty_stratified_gap_5_coqa} analyze the deduplicated prediction set size on CoQA using LLaMA-3.1-8B-Instruct. 
We again divide sampling-successful test examples into easy and hard groups according to whether the top-1 candidate is admissible. 
MiRD produces larger prediction sets for hard examples than for easy examples over a broad range of risk levels, indicating that its set size reflects example difficulty. 
By contrast, ConU tends to shrink prediction sets more aggressively, especially on hard examples.

\begin{figure*}[!t]
  \centering
  \begin{subfigure}[b]{0.32\textwidth}
    \centering
    \includegraphics[width=\textwidth]{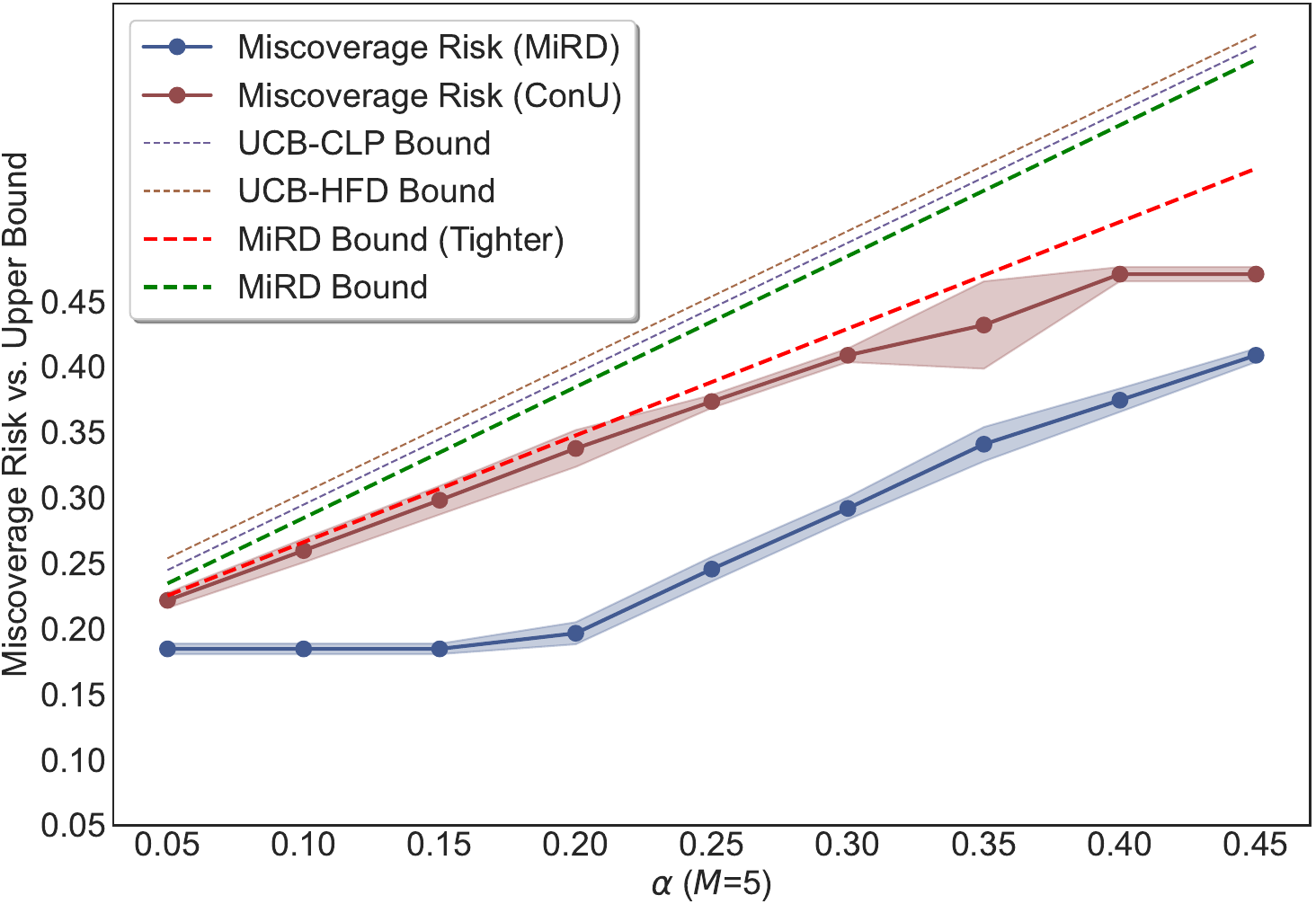}
    \caption{LLaMA-3.1-8B-Instruct.}
  \end{subfigure}
  \hfill
  \begin{subfigure}[b]{0.32\textwidth}
    \centering
    \includegraphics[width=\textwidth]{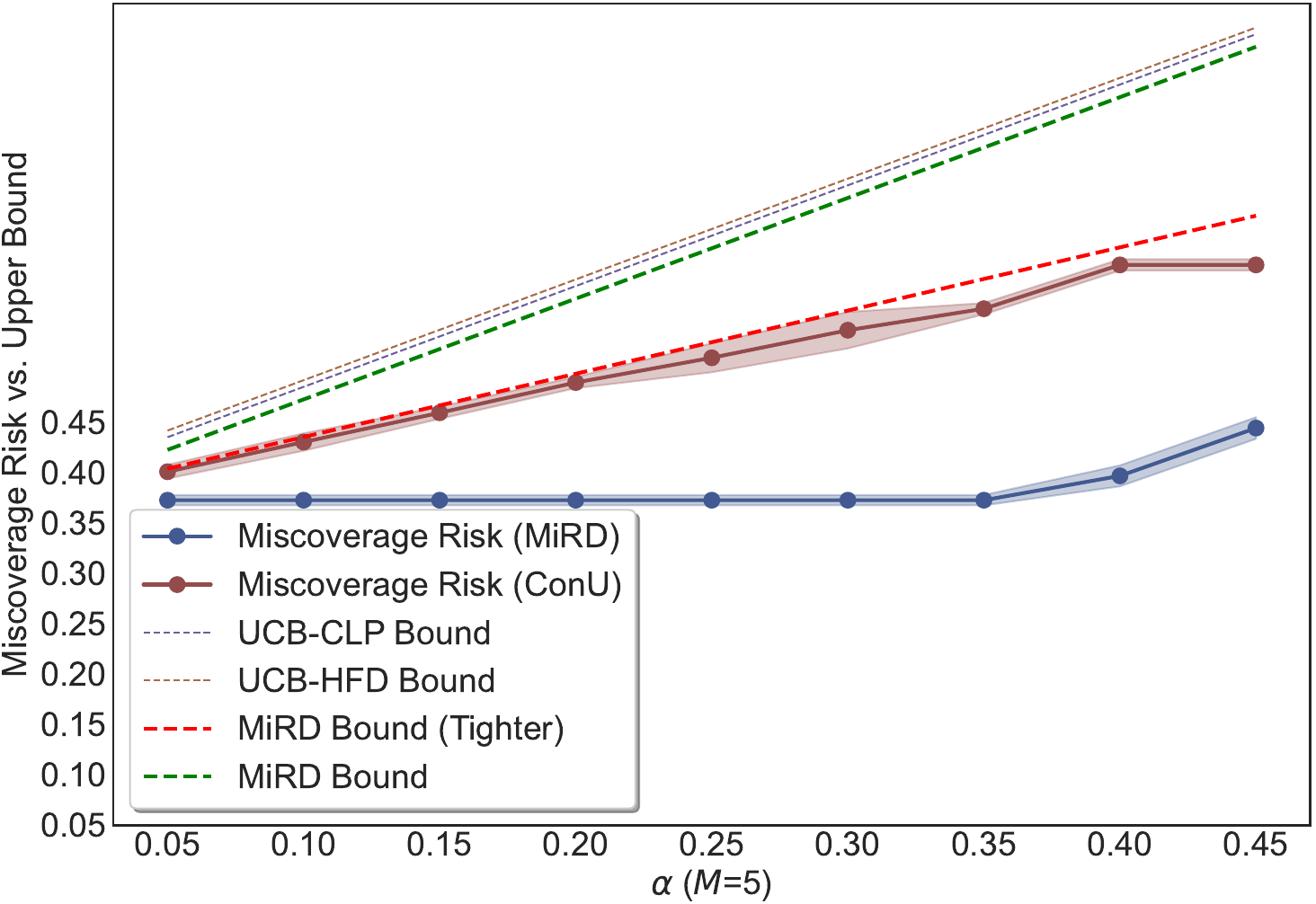}
    \caption{Qwen2.5-3B-Instruct.}
  \end{subfigure}
  \hfill
  \begin{subfigure}[b]{0.32\textwidth}
    \centering
    \includegraphics[width=\textwidth]{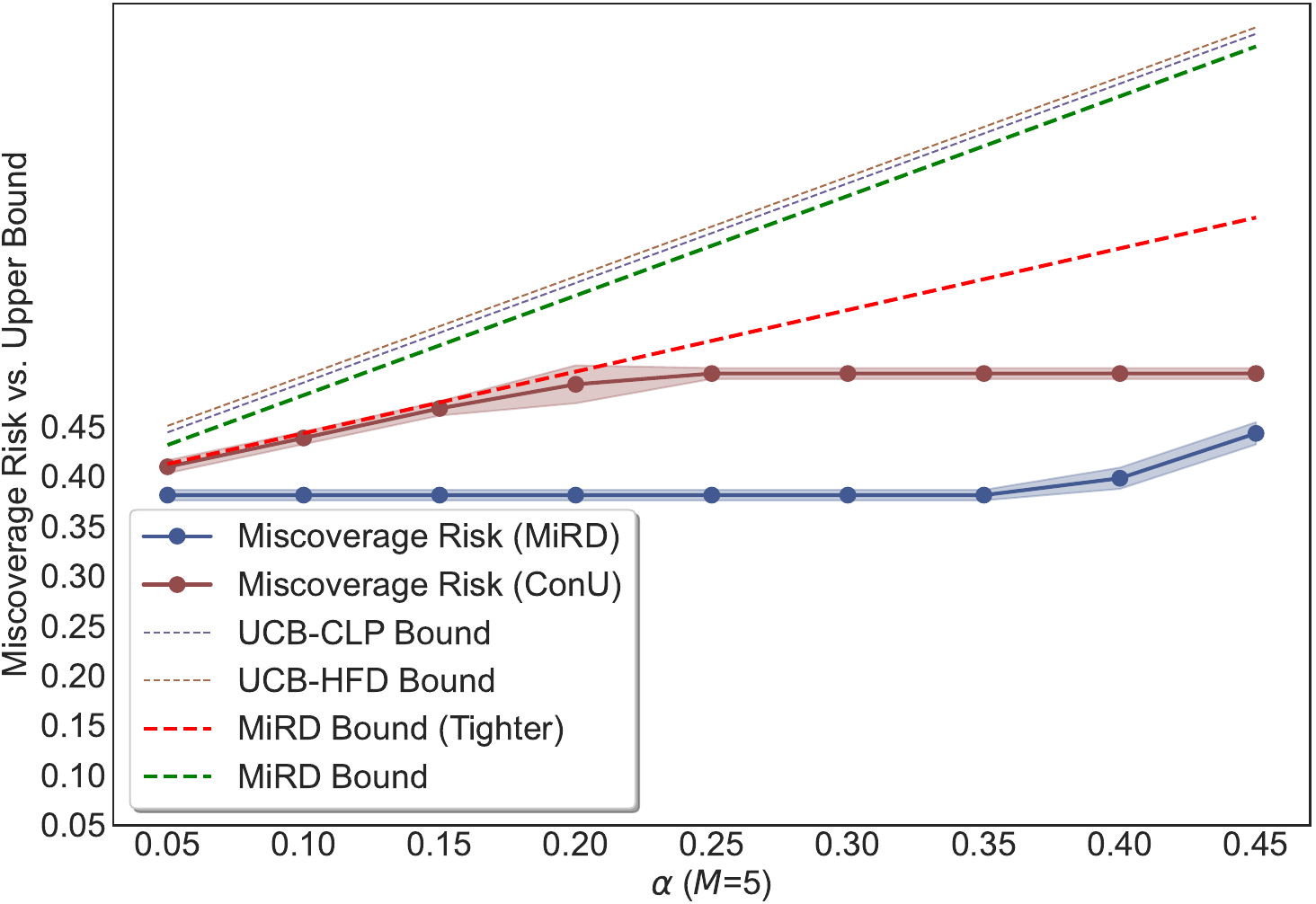}
    \caption{Qwen2.5-14B-Instruct.}
  \end{subfigure}

  \begin{subfigure}[b]{0.32\textwidth}
    \centering
    \includegraphics[width=\textwidth]{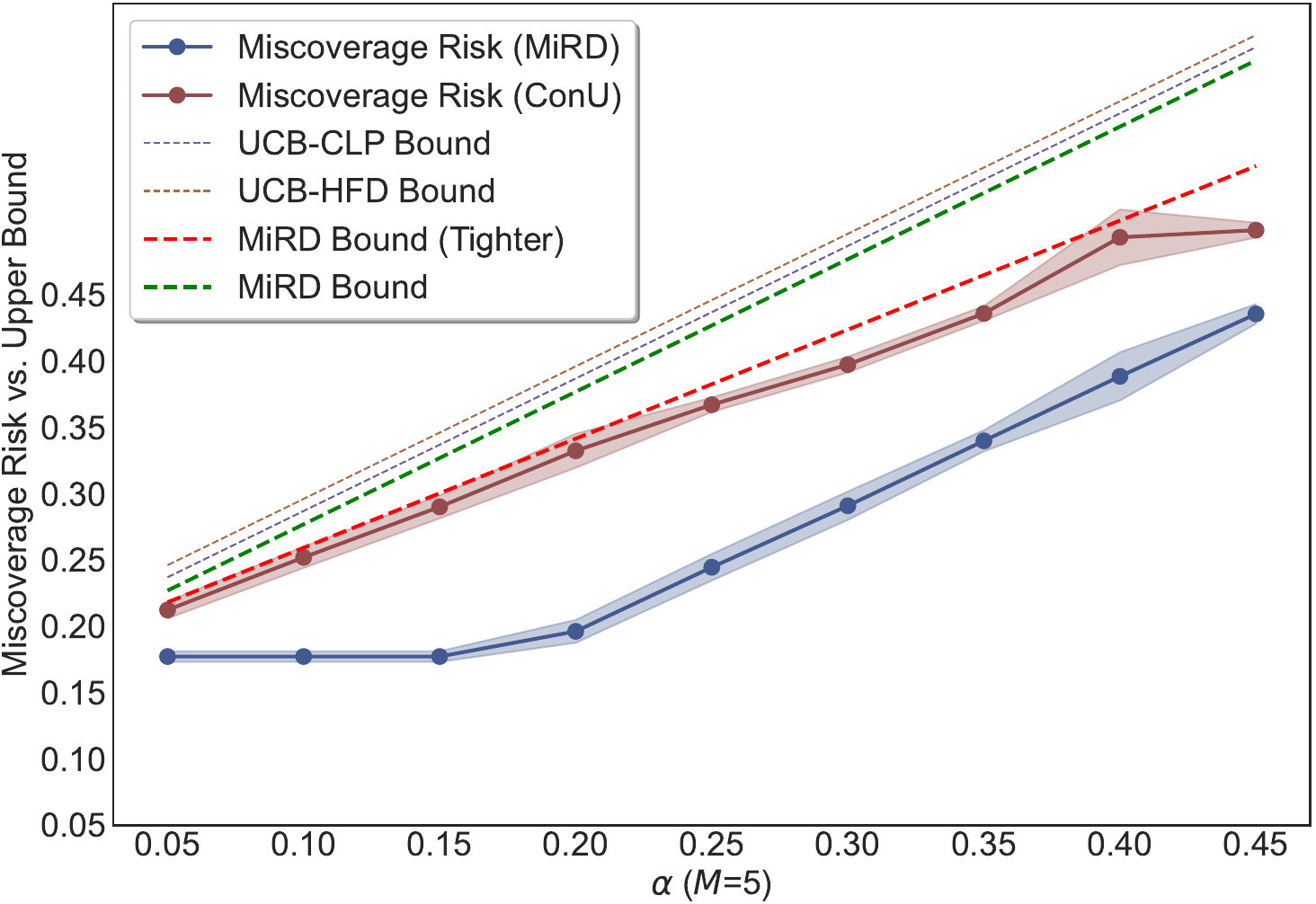}
    \caption{OpenChat-3.5 (7B).}
  \end{subfigure}
  \hfill
  \begin{subfigure}[b]{0.32\textwidth}
    \centering
    \includegraphics[width=\textwidth]{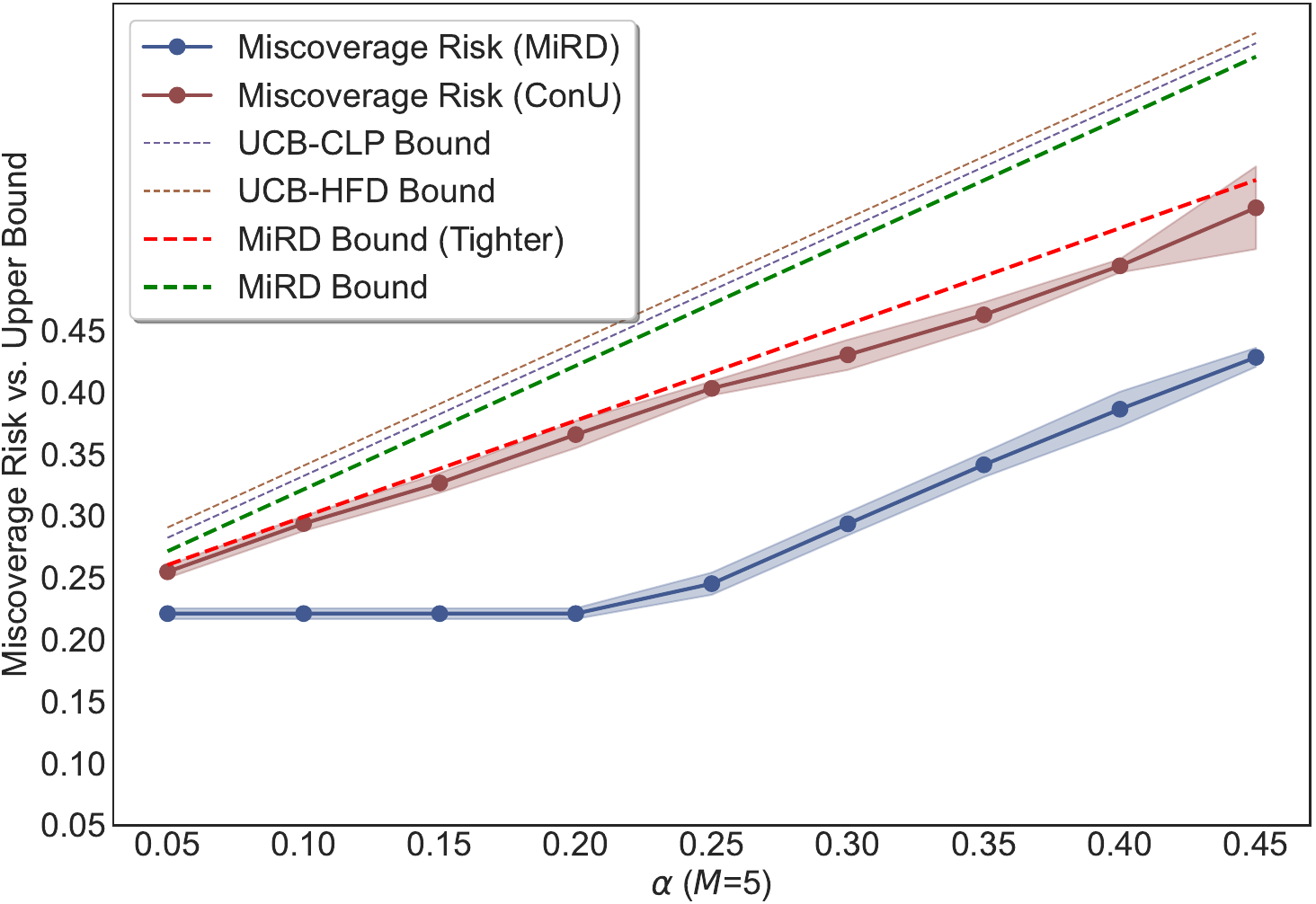}
    \caption{Vicuna-7B-V1.5.}
  \end{subfigure}
  \hfill
  \begin{subfigure}[b]{0.32\textwidth}
    \centering
    \includegraphics[width=\textwidth]{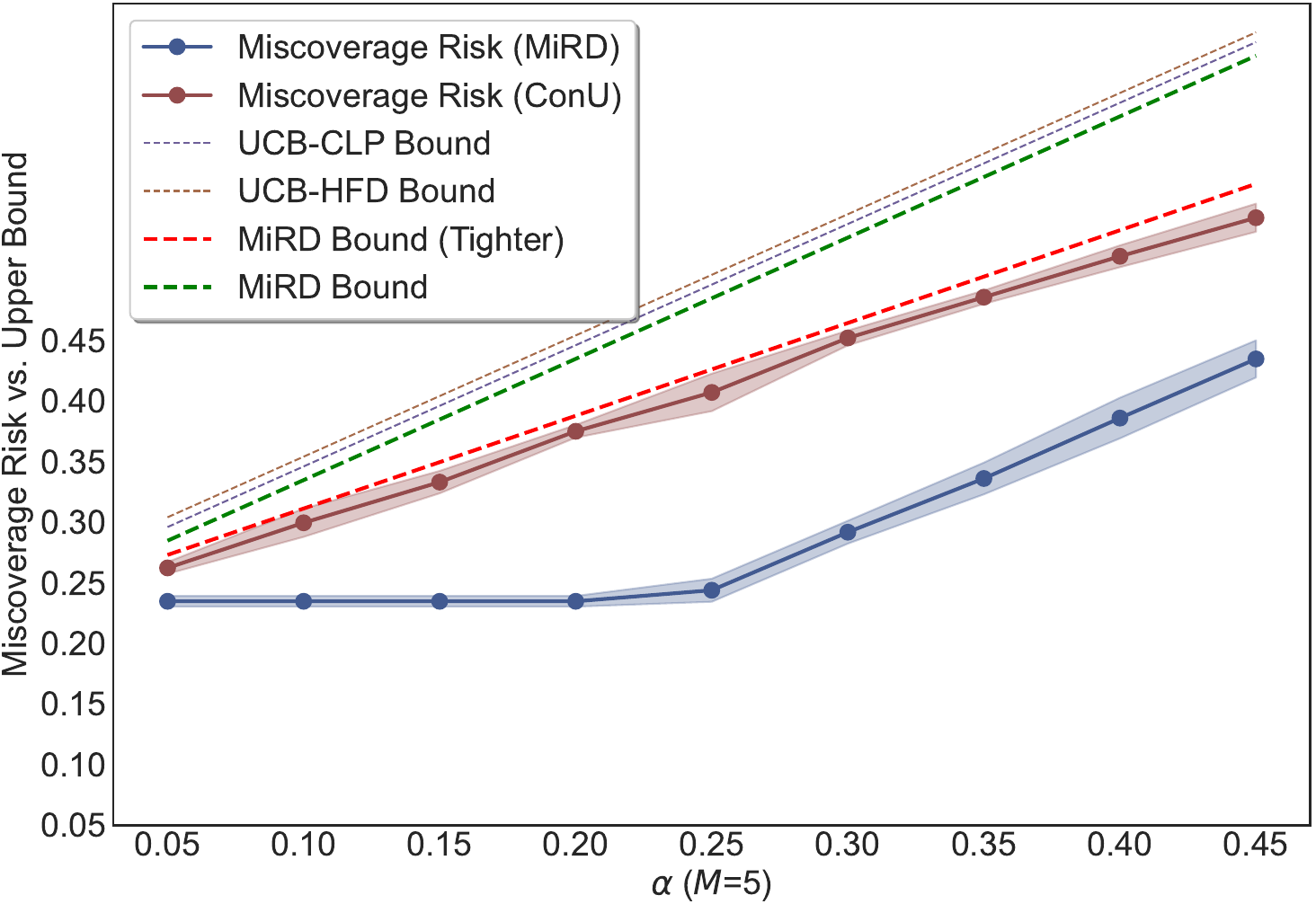}
    \caption{Vicuna-13B-V1.5.}
  \end{subfigure}

      \caption{Overall miscoverage risk vs. upper bound at various risk levels on CoQA with six LLMs ($M=5$).}
  \label{fig: miscoverage risk control coqa (similairty-0.6, M=5).}
  \vspace{-4mm}
\end{figure*}

The adaptiveness-gap plot further highlights this difference. 
MiRD maintains a positive hard--easy set-size gap at stricter risk levels, while ConU exhibits a much smaller or negative gap across most risk levels. 
This suggests that successful-only calibration may distort the relationship between set size and sample difficulty, whereas MiRD better preserves uncertainty-dependent variation by retaining the full calibration set. 
Thus, the additional prediction-set size introduced by MiRD is not merely uniform inflation; it is more aligned with the ambiguity of the test example.

\begin{figure*}[!t]
  \centering
  \begin{subfigure}[b]{0.48\textwidth}
    \centering
    \includegraphics[width=\textwidth]{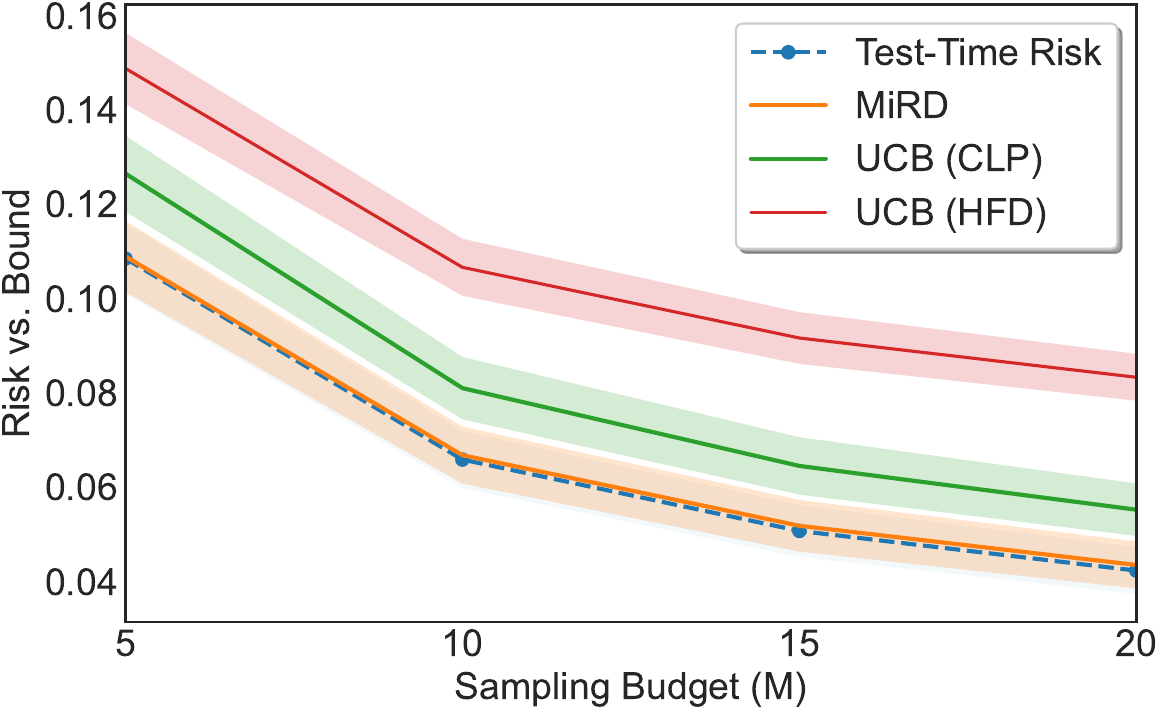}
    \caption{LLaMA-3.1-8B-Instruct (similarity).}
  \end{subfigure}
  \hfill
  \begin{subfigure}[b]{0.48\textwidth}
    \centering
    \includegraphics[width=\textwidth]{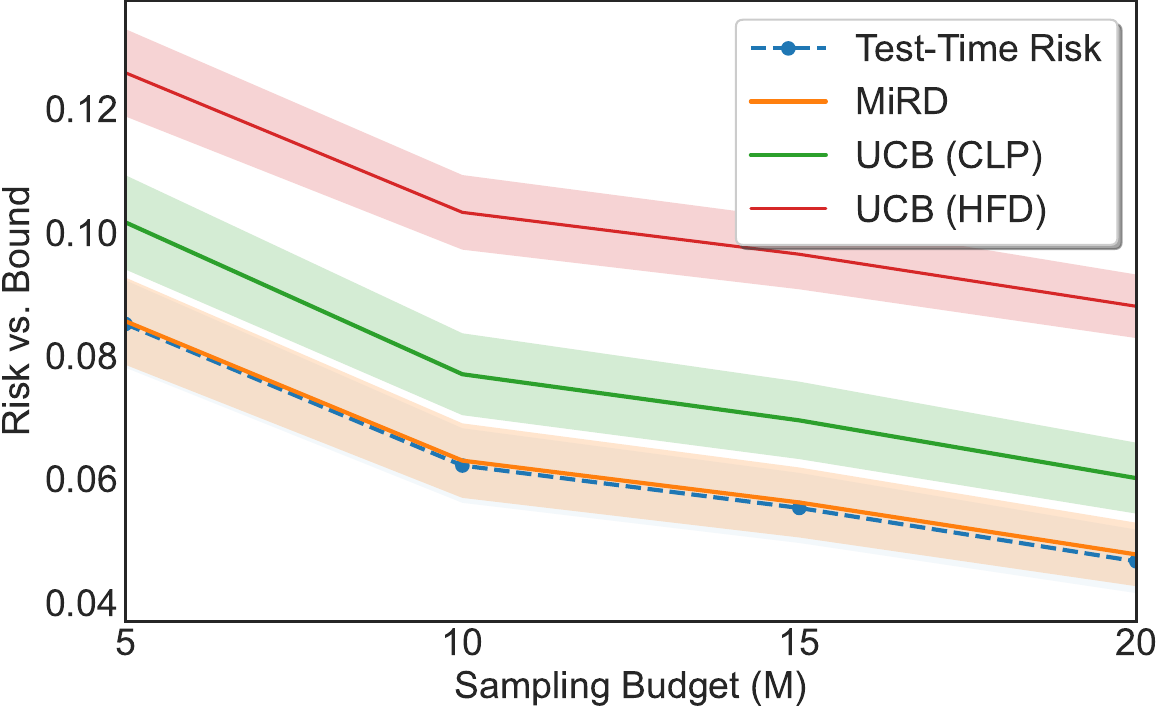}
    \caption{Qwen2.5-3B-Instruct (similarity).}
  \end{subfigure}

  \begin{subfigure}[b]{0.48\textwidth}
    \centering
    \includegraphics[width=\textwidth]{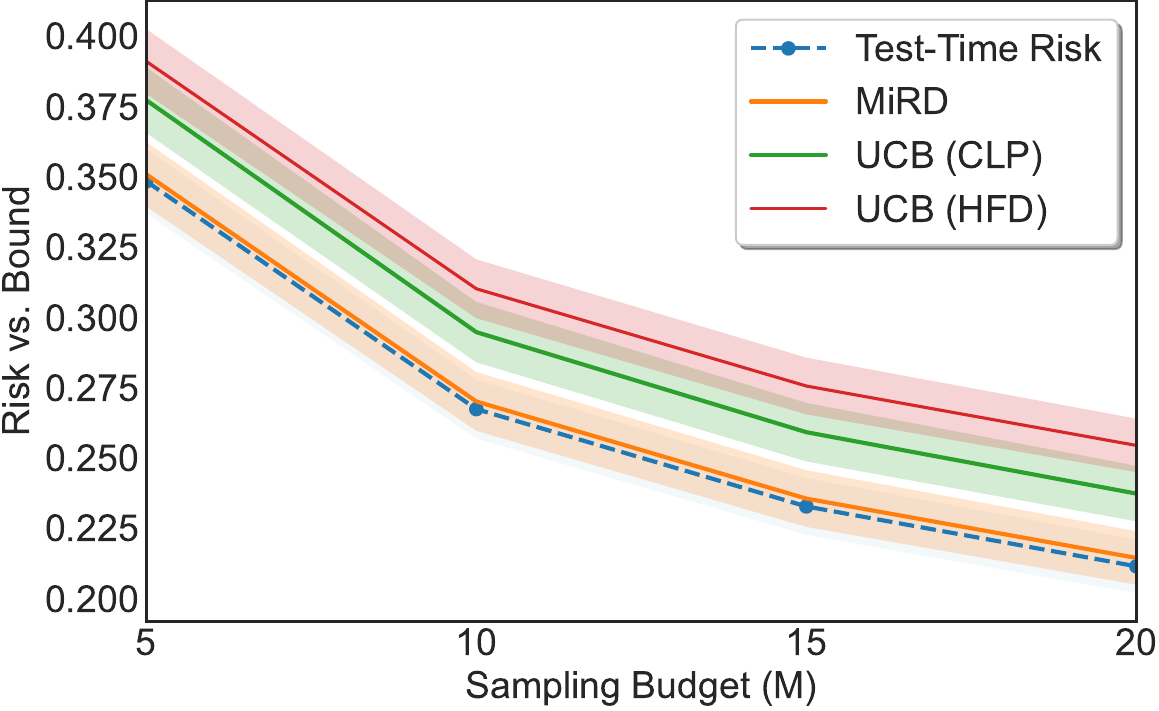}
    \caption{LLaMA-3.1-8B-Instruct (entailment).}
  \end{subfigure}
  \hfill
  \begin{subfigure}[b]{0.48\textwidth}
    \centering
    \includegraphics[width=\textwidth]{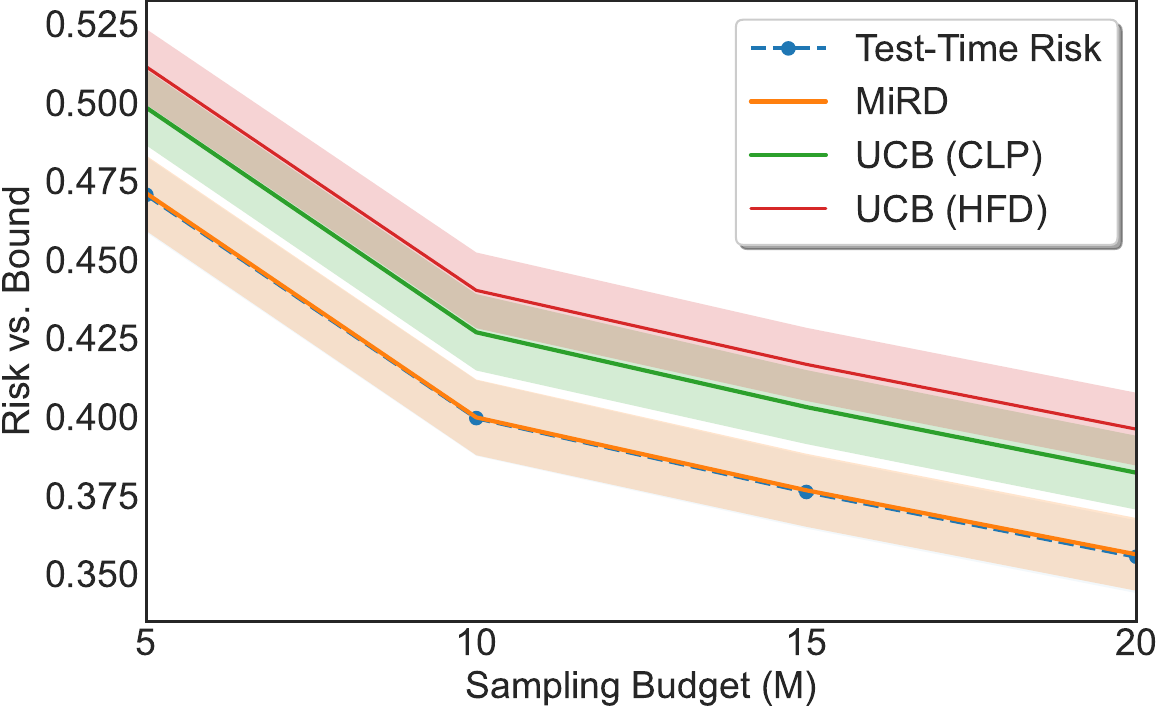}
    \caption{Qwen2.5-3B-Instruct (entailment).}
  \end{subfigure}
  
      \caption{Sampling risk vs. three upper bounds at various sampling budgets ($M$) on NQ with two LLMs.}
  \label{fig: sampling risk control nq (similairty-0.6 vs. entailment).}
\end{figure*}

\noindent \textbf{Overall miscoverage risk control on CoQA.} 
Figure~\ref{fig: miscoverage risk control coqa (similairty-0.6, M=5).} evaluates the end-to-end overall miscoverage risk on CoQA with $M=5$. 
Across all six LLMs, the empirical overall miscoverage of MiRD remains below the proposed upper bounds over the full range of target risk levels. 
Moreover, the tighter MiRD bound tracks the empirical risk more closely than the conservative bound, while the PAC-style UCB bounds remain substantially looser. 
This again confirms the practical value of the refined decomposition bound.

\begin{figure}[!t]
    \centering
    \includegraphics[width=\linewidth]{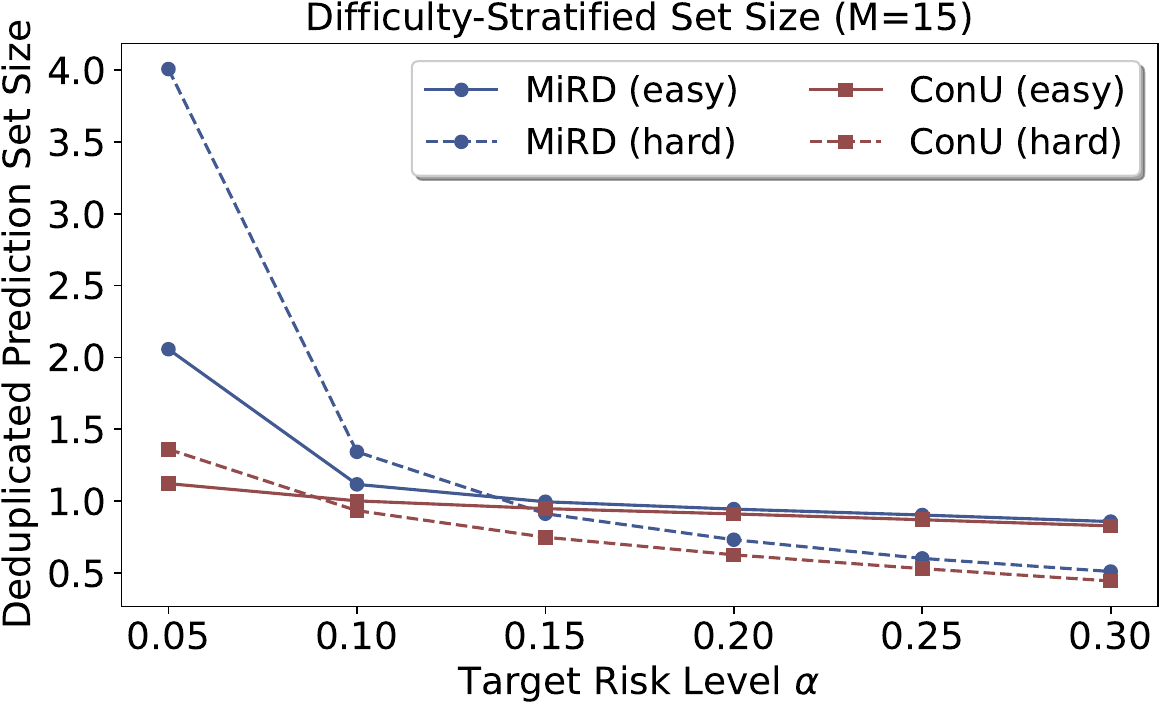}
    \caption{Difficulty-stratified deduplicated prediction set size on NQ with LLaMA-3.1-8B-Instruct. Test samples are partitioned into easy and hard groups according to whether the top-1 candidate is admissible.}
    \label{fig:difficulty_stratified_set_size_15_nq}
    \vspace{-4mm}
\end{figure}

\begin{figure}[!t]
    \centering
    \includegraphics[width=\linewidth]{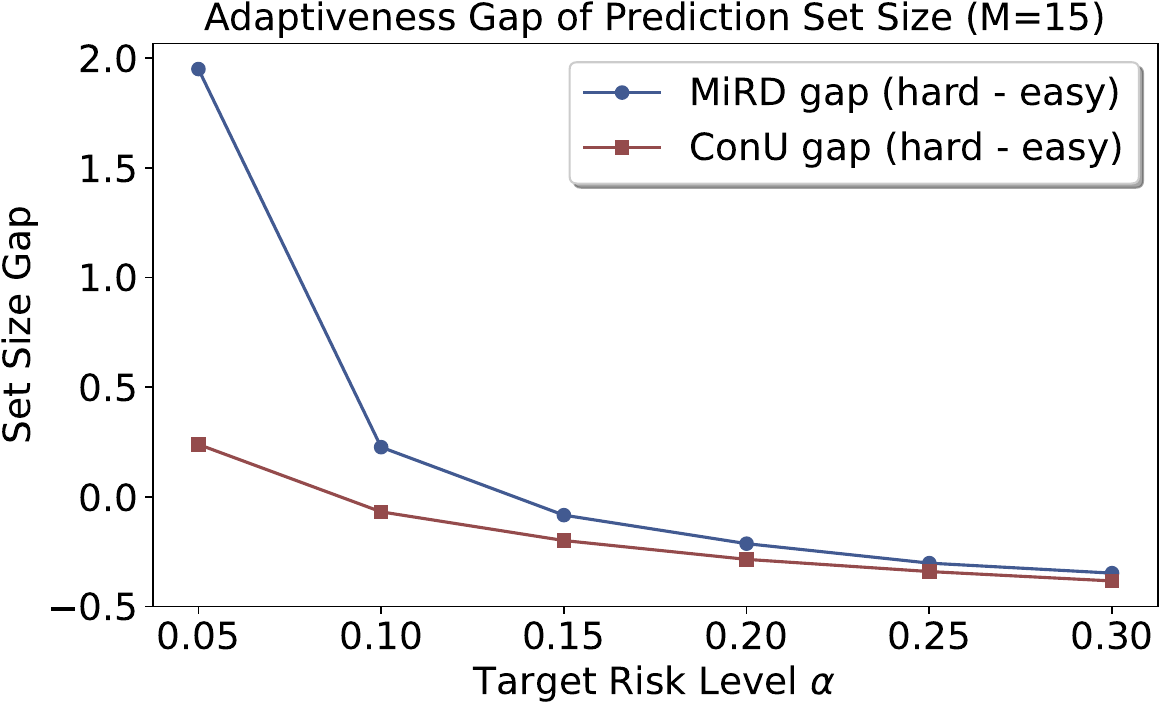}
    \caption{Adaptiveness gap of prediction set size on NQ with LLaMA-3.1-8B-Instruct, defined as the difference between the average deduplicated prediction set sizes of hard and easy examples.}
    \label{fig:difficulty_stratified_gap_15_nq}
    \vspace{-4mm}
\end{figure}

Compared with ConU, MiRD achieves competitive or lower overall miscoverage while explicitly accounting for finite-sampling failure. 
This result is particularly important on CoQA because the presence of supporting context does not eliminate sampling uncertainty: a model may still fail to generate an admissible answer within a finite budget, or may generate one but remove it during filtering. 
The CoQA results therefore reinforce the main conclusion that reliable open-ended QA requires separating finite-sampling failure from conditional post-selection failure, rather than collapsing both into a single filtering-stage error.

We finally evaluate MiRD on Natural Questions (NQ), a more challenging open-domain QA benchmark with real user questions. 
Compared with TriviaQA and CoQA, NQ poses a harder generation setting, making it a useful stress test for whether the proposed decomposition remains valid when finite sampling is more likely to fail. 
We report results under both sentence-similarity and bi-entailment correctness criteria.

\noindent \textbf{Sampling-failure risk control on NQ.} 
Figure~\ref{fig: sampling risk control nq (similairty-0.6 vs. entailment).} reports the sampling-failure risk and its upper bounds on NQ using two representative LLMs. 
Across both models, both correctness criteria, and all sampling budgets considered, the MiRD bound remains above the empirical test-time sampling-failure risk. 
This validates the expectation-level marginal guarantee of Stage I in a substantially more challenging QA setting. 
Moreover, MiRD continues to provide tighter bounds than both Clopper--Pearson and Hoeffding-style UCBs, showing that the advantage of the proposed marginal formulation is robust beyond the TriviaQA and CoQA settings.

The results also reveal that the choice of admissibility criterion substantially affects the magnitude of sampling failure. 
Under the stricter bi-entailment criterion, the empirical sampling-failure risk is noticeably higher than under sentence similarity, especially for weaker backbones. 
This observation further supports the need to treat finite-sampling failure as an explicit component of the reliability objective: when correctness is evaluated more conservatively, simply increasing the sampling budget may reduce but does not eliminate the risk that no admissible answer appears in the candidate set.

\noindent \textbf{Prediction-set adaptiveness on NQ.} 
Figures~\ref{fig:difficulty_stratified_set_size_15_nq} and~\ref{fig:difficulty_stratified_gap_15_nq} analyze the deduplicated prediction set size on NQ using LLaMA-3.1-8B-Instruct. 
As in the main experiments, we divide sampling-successful test examples into easy and hard groups according to whether the top-1 candidate is admissible. 
MiRD assigns larger prediction sets to hard examples than to easy examples across most risk levels, indicating that its prediction sets remain sensitive to example difficulty even on NQ.

\begin{figure*}[!t]
  \centering
  \begin{subfigure}[b]{0.48\textwidth}
    \centering
    \includegraphics[width=\textwidth]{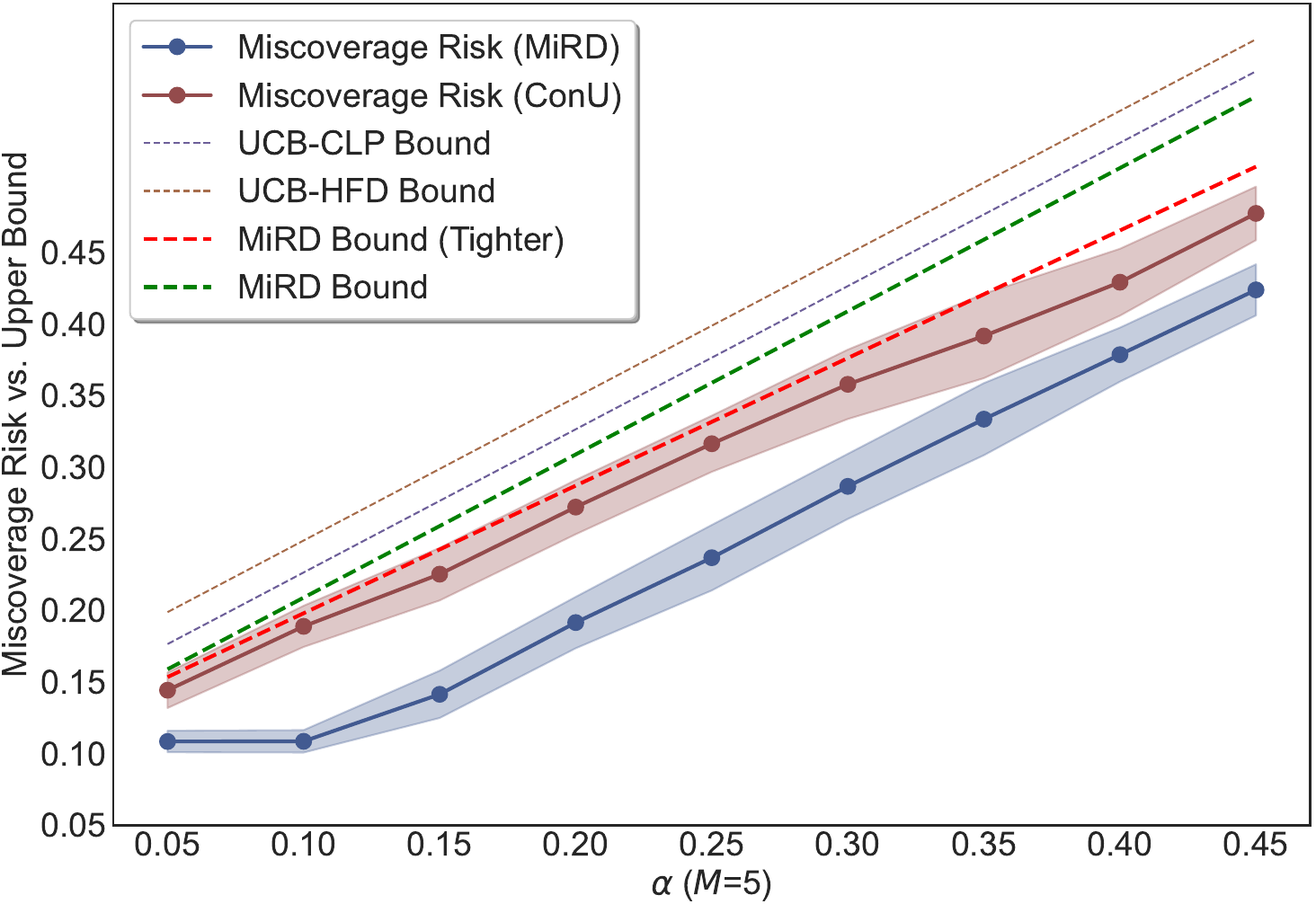}
    \caption{LLaMA-3.1-8B-Instruct (similarity).}
  \end{subfigure}
  \hfill
  \begin{subfigure}[b]{0.48\textwidth}
    \centering
    \includegraphics[width=\textwidth]{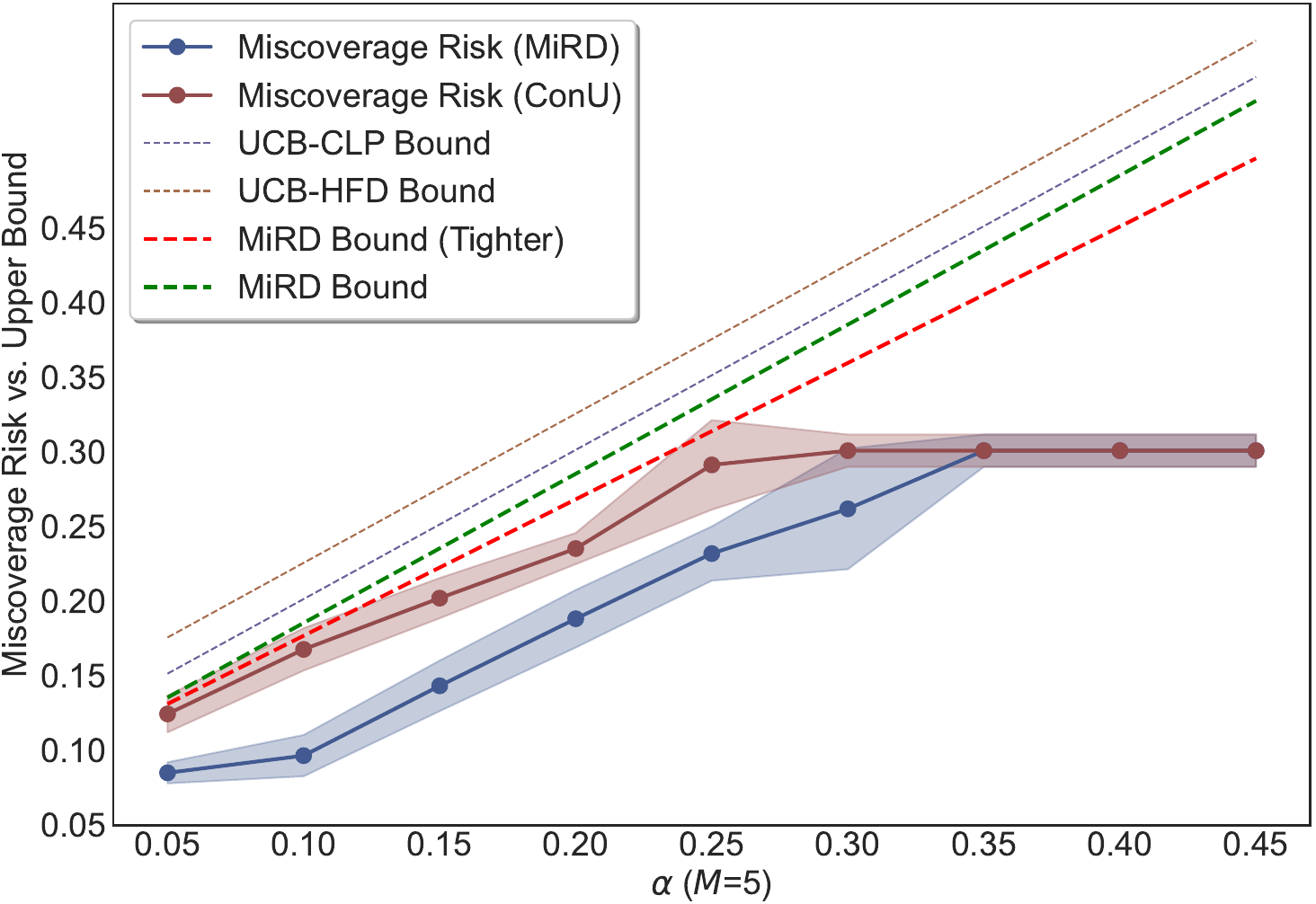}
    \caption{Qwen2.5-3B-Instruct (similarity).}
  \end{subfigure}

  \begin{subfigure}[b]{0.48\textwidth}
    \centering
    \includegraphics[width=\textwidth]{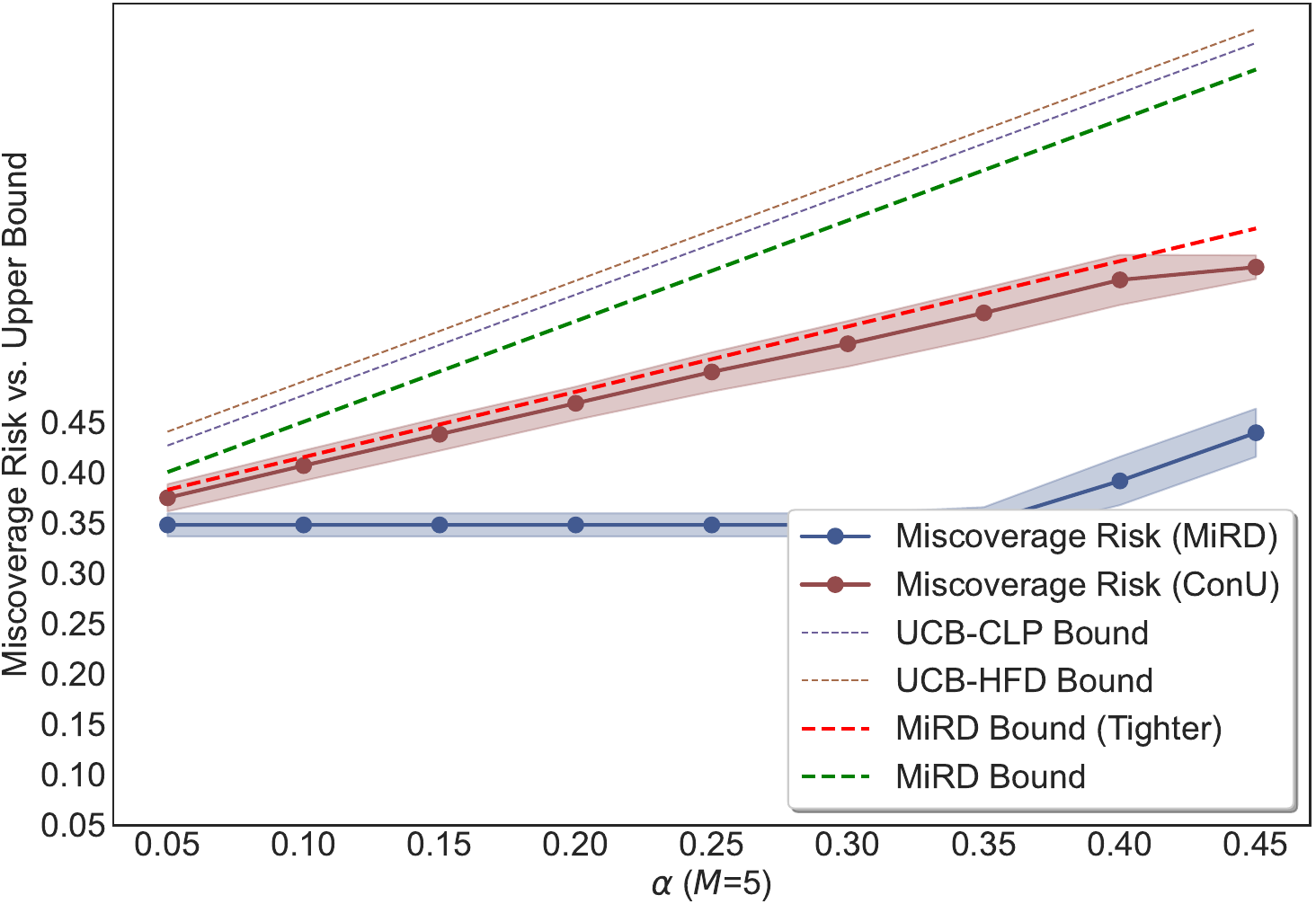}
    \caption{LLaMA-3.1-8B-Instruct (entailment).}
  \end{subfigure}
  \hfill
  \begin{subfigure}[b]{0.48\textwidth}
    \centering
    \includegraphics[width=\textwidth]{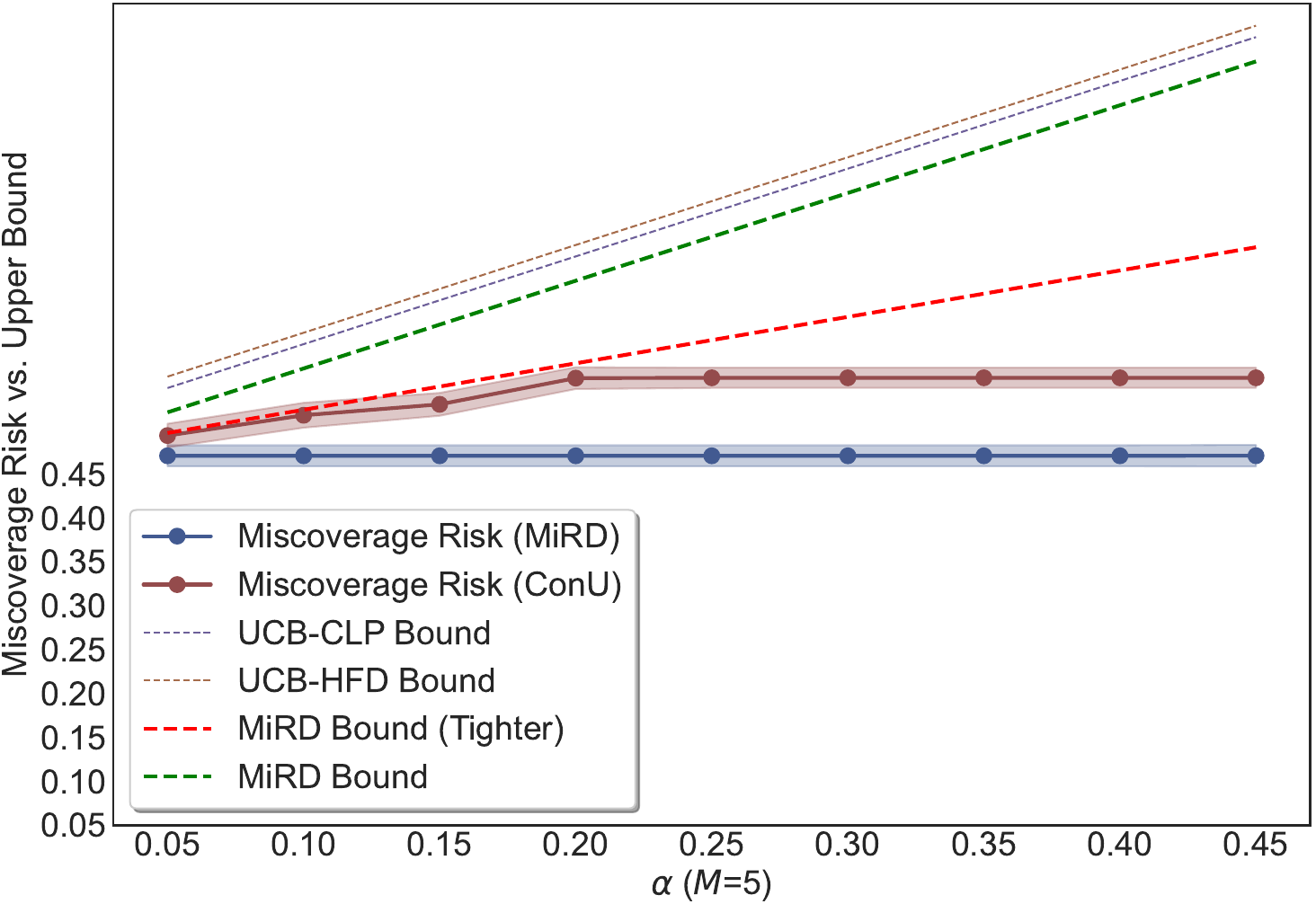}
    \caption{Qwen2.5-3B-Instruct (entailment).}
  \end{subfigure}

      \caption{Overall miscoverage risk vs. upper bound at various risk levels on NQ with two LLMs ($M=5$).}
  \label{fig: miscoverage risk control nq (similairty-0.6 vs. entail, M=5).}
  \vspace{-4mm}
\end{figure*}

The adaptiveness-gap plot further confirms this behavior. 
MiRD maintains a consistently larger hard--easy set-size gap than ConU, while ConU exhibits a much smaller gap and may even under-allocate set size to hard examples. 
This suggests that successful-only calibration can weaken the relationship between set size and sample difficulty, whereas MiRD better preserves uncertainty-dependent variation by calibrating over the full calibration distribution.

\noindent \textbf{Overall miscoverage risk control on NQ.} 
Figure~\ref{fig: miscoverage risk control nq (similairty-0.6 vs. entail, M=5).} evaluates the end-to-end overall miscoverage risk on NQ. 
Across both LLMs and both correctness criteria, the empirical overall miscoverage of MiRD remains below the proposed upper bounds over the range of target risk levels. 
The tighter MiRD bound again tracks the empirical miscoverage more closely than the conservative bound, while PAC-style UCB bounds remain looser. 
These results confirm that the refined decomposition bound remains useful even when the underlying sampling-failure risk is relatively high.

Compared with ConU, MiRD achieves competitive or lower overall miscoverage while explicitly accounting for finite-sampling failure. 
This is particularly important on NQ, where the harder question distribution and stricter admissibility criteria make sampling failure more pronounced. 
Together with the TriviaQA and CoQA results, the NQ experiments demonstrate that MiRD provides consistent risk control across knowledge-intensive, context-supported, and more challenging open-domain QA settings.


\end{document}